%% file: main.tex
\documentclass[10pt]{article}
\usepackage{enumerate}
\usepackage{geometry}
\usepackage[OT1]{fontenc}
\usepackage{amsmath,amssymb}
\usepackage{mathtools}
\usepackage{natbib}
\usepackage[usenames]{color}
\usepackage{dsfont}
\usepackage{pifont}
\usepackage{pgfplots}
\usepackage{smile}
\usepackage{multirow}
\usepackage{rotating}
\usepackage{enumerate}
\usepackage{esvect}
\usepackage{tikz}
\usetikzlibrary{patterns}
\usetikzlibrary{arrows}
\usepackage{color,xcolor}
\usepackage{colortbl}
\usepackage[colorlinks,
            pagebackref, 
            linktoc=page,
            linkcolor=red,
            anchorcolor=red,
            citecolor=blue
            ]{hyperref}
\usepackage{algorithm}
\usepackage{algorithmic}
\usepackage{cancel}
\usepackage{paralist}

\def\supp{\mathop{\text{supp}}}

\long\def\comment#1{}

\def\tr{\mathop{\text{Tr}}}

\def\cS{{\mathcal{S}}}

\newcommand{\bel}{\begin{eqnarray}\label}
\newcommand{\eel}{\end{eqnarray}}
\newcommand{\bes}{\begin{eqnarray*}}
\newcommand{\ees}{\end{eqnarray*}}

\let\hat\widehat
\let\tilde\widetilde
\def\mid{\,|\,}

\def\supp{\mathop{\text{supp}\kern.2ex}}
\def\argmin{\mathop{\text{\rm arg\,min}}}
\def\argmax{\mathop{\text{\rm arg\,max}}}

\def\tr{{\rm{Tr}}}

\def\supp{\mathop{\text{supp}}}

\def\tr{\mathrm{Tr}}

\def\##1\#{\begin{align}#1\end{align}}
\def\$#1\${\begin{align*}#1\end{align*}}

\usepackage{undertilde}

\definecolor{green2}{HTML}{c2fdfe}
\definecolor{green1}{HTML}{bce672}
\definecolor{yellow1}{HTML}{f5dd6f}
\definecolor{blue1}{HTML}{3eede7}
\definecolor{orange1}{HTML}{e77c4b}
\definecolor{red1}{HTML}{f47983}

\newcommand{\xmark}{\ding{55}}
\def\shownotes{1}  \ifnum\shownotes=1
\newcommand{\authnote}[2]{{[#1: #2]}}
\else
\newcommand{\authnote}[2]{}
\fi

\title{Double Pessimism is Provably Efficient for Distributionally Robust Offline Reinforcement Learning: Generic Algorithm and Robust Partial Coverage}

\author{
    Jose Blanchet\thanks{Alphabetical order. Email to \texttt{miaolu@stanford.edu}} \thanks{Department of Management Science and Engineering, Stanford University.}\qquad 
    Miao Lu$^*$\footnotemark[2] \qquad 
    Tong Zhang$^*$\thanks{Department of Mathematics, The Hong Kong University of Science and Technology.}\qquad 
    Han Zhong$^*$\thanks{Center for Data Science, Peking University.} 
}

\date{May 17, 2023; \quad Revised: \today}

\begin{document}

\maketitle

\begin{abstract}
    In this paper, we study distributionally robust offline reinforcement learning (robust offline RL), which seeks to find an optimal policy purely from an offline dataset that can perform well in perturbed environments. 
    In specific, we propose a generic algorithm framework called \underline{D}oubly \underline{P}essimistic \underline{M}odel-based \underline{P}olicy \underline{O}ptimization ($\texttt{P}^2\texttt{MPO}$), which features a novel combination of a flexible model estimation subroutine and a doubly pessimistic policy optimization step. 
    Notably, the \emph{double pessimism} principle is crucial to overcome the distributional shifts incurred by (i) the mismatch between the behavior policy and the family of target policies; and (ii) the perturbation of the nominal model. Under certain accuracy conditions on the model estimation subroutine, we prove that $\texttt{P}^2\texttt{MPO}$ is sample-efficient with \emph{robust partial coverage data}, which only requires the offline data to have good coverage of the distributions induced by the optimal robust policy and the perturbed models around the nominal model. Our assumption on data is relatively mild compared with previous full-coverage-style assumptions which need a uniformly lower bounded data distribution.

    Our algorithm and theory can be applied to a vast body of robust Markov decision processes (RMDPs) in the regime of large state spaces. By tailoring specific model estimation subroutines for concrete examples of RMDPs, including tabular RMDPs, factored RMDPs, kernel and neural RMDPs, we prove that for all these examples $\texttt{P}^2\texttt{MPO}$ enjoys a $\tilde{\mathcal{O}}(n^{-1/2})$ convergence rate, where $n$ is the number of trajectories in data.
    We highlight that all these RMDP examples, except tabular RMDPs, are first identified and proven tractable by this work. Furthermore, as an extension to multi-agent decision-making, we continue our study of robust offline RL in the multi-player robust Markov games (RMGs). 
    By extending the double pessimism principle identified for single-agent RMDPs, we propose another doubly-pessimistic-type algorithm framework that can efficiently find the \emph{robust Nash equilibria} among players using only robust unilateral (partial) coverage data.
    To our best knowledge, this work proposes the first general learning principle --- double pessimism --- for robust offline RL and shows that it is provably efficient in the context of general function approximation. 
\end{abstract}


\tableofcontents

\input{tex/introduction.tex}

\input{tex/preliminaries.tex}

\input{tex/algorithm.tex}

\input{tex/implementation.tex}

\input{tex/game.tex}

\input{tex/discussion.tex}

\bibliographystyle{ims}
\bibliography{ref}

\newpage
\appendix

\input{tex/appendix/robust_bellman.tex}

\input{tex/appendix/sketch.tex}

\input{tex/appendix/offline_rmg.tex}

\input{tex/appendix/sarmdp.tex}

\input{tex/appendix/safrmdp.tex}

\input{tex/appendix/drlmdp.tex}

\input{tex/appendix/mle.tex}

\input{tex/appendix/tech.tex}

\end{document}

%% file: tex/introduction.tex
\section{Introduction}

Reinforcement learning (RL) \citep{sutton2018reinforcement} aims to find an optimal policy that can maximize the expected cumulative rewards obtained from an unknown environment. 
Typically, modern deep RL algorithms learn such a policy in an online trial-and-error fashion, collecting millions to billions of data.
However, online data collection could be costly and risky in many practical applications, such as healthcare \citep{wang2018supervised} and autonomous driving \citep{pan2017agile}, prohibiting the use of RL in these critical domains.
To tackle this challenge, offline RL \citep{levine2020offline} (also known as batch RL \citep{lange2012batch}) proposes to learn a near-optimal policy purely from a dataset collected a priori without further interactions with the environment. 

Recent years have witnessed great progresses in offline RL for both practice and theory \citep{yu2020mopo,kumar2020conservative,jin2021pessimism,uehara2021pessimistic,xie2021bellman,cheng2022adversarially}.
Nevertheless, these works implicitly require that the offline data are generated by the real-world environment, which may fail in practice. 
Taking robotics \citep{kober2013reinforcement,openai2018learning} as an example, the experimenter trains the agents in a simulated physical environment and then deploys them in real-world environments. 
Since the experimenter does not have access to the true physical environments, there would be a mismatch between the simulated environment to generate the offline dataset and the real-world environments to deploy the trained agents. 
Such a mismatch is commonly referred to as the \emph{sim-to-real gap} \citep{peng2018sim,zhao2020sim}. 
Since in RL the optimal policy is sensitive to the model \citep{mannor2004bias,el2005robust}, the potential sim-to-real gap may lead to the poor performance of RL algorithms.

A promising solution to remedy this issue is robust RL \citep{iyengar2005robust,el2005robust,morimoto2005robust} --- learning a robust policy that can perform well in a bad or even adversarial environment. 
A line of works on deep robust RL \citep{pinto2017supervision,pinto2017robust,pattanaik2017robust, mandlekar2017adversarially, tessler2019action, zhang2020robust, kuang2022learning} demonstrates the superiority of the trained robust policy in the real world environments. 
Furthermore, the recent work of \citet{hu2022provable} theoretically proves that the ideal robust policy does attain near optimality for problems with the sim-to-real gap. 
However, this work does not suggest how to learn an optimal robust policy efficiently from a theoretical perspective. 

To understand robust RL from theoretical sides, robust Markov decision process (RMDP) \citep{iyengar2005robust,el2005robust} has been proposed and extensively studied, and many recent works \citep{zhou2021finite,yang2021towards,shi2022distributionally,ma2022distributionally} design sample-efficient algorithms for offline RL in RMDPs. 
But these works mainly focus on the tabular case, i.e., finite state space, and thus are not capable of tackling large or even infinite state spaces which usually appear in modern RL applications. 
Meanwhile, in the non-robust offline RL setting, a line of works \citep{jin2021pessimism,uehara2021pessimistic,xie2021bellman,zanette2021provable,rashidinejad2021bridging} has shown that ``\emph{pessimism}'' is the general learning principle for designing sample-efficient algorithms that can overcome the key difficulty in offline RL, that is, the distributional shift problem caused by finite fixed data.
In particular, in the context of function approximation, \cite{xie2021bellman} and \cite{uehara2021pessimistic} leverage the pessimism principle and propose generic algorithms in the model-free and model-based fashion, respectively. 
Hence, it is natural to ask the following questions:
\begin{center}
    \textbf{Q1:} \emph{What is the general learning principle for robust offline RL?} \\
    \textbf{Q2:} \emph{Based on this learning principle, can we design a generic algorithm for robust offline RL in the context of function approximation?}
\end{center}
To answer these two questions, we need to handle two intertwined challenges --- \emph{distributional shifts} 
%\han{distribution shift or distributional shift} 
and \emph{large state space}.
In general, the distributional shift is caused by the mismatch between the offline data distribution and the distributions induced by the target policies \emph{and} the target environments.
Here in robust offline RL, the distributional shifts have two sources: (i) the mismatch between the behavior policy and the target policies to be learned; (ii) the mismatch between the nominal environment and the perturbed environment.
The latter is a unique challenge that is not presented in non-robust offline RL.
Besides, regarding the state space, existing works mainly focus on the tabular case, and it still remains elusive how to add reasonable structural conditions to make RMDPs with large state spaces tractable. 
Despite all these challenges, in this paper, we answer the aforementioned two questions affirmatively. 
Our contributions are summarized below.

\subsection{Our Contributions}

Our work contributes to the theoretical understanding of robust offline (multi-agent) RL in large state spaces.
More concretely, our contributions are three-fold.

\vspace{3mm}
\begin{itemize}
    \item \textbf{General learning principle and algorithmic design.} 
    We first study robust offline single-agent RL within a general framework, which not only includes existing known tractable $\mathcal{S} \times \mathcal{A}$-rectangular tabular RMDPs, but also subsumes several newly proposed models: $\mathcal{S} \times \mathcal{A}$-rectangular factored RMDPs, $\mathcal{S} \times \mathcal{A}$-rectangular kernel RMDPs, and $\mathcal{S} \times \mathcal{A}$-rectangular neural RMDPs.
    Under this framework, we propose a generic model-based algorithm, dubbed as \underline{D}oubly \underline{P}essimistic \underline{M}odel-based \underline{P}olicy \underline{O}ptimization ($\texttt{P}^2\texttt{MPO}$), which consists of a model estimation subroutine and a policy optimization step based on \emph{doubly pessimistic} value estimators. 
    The algorithm is based on a \emph{double pessimism} principle, which requires being pessimism in the face of \emph{model estimation uncertainty} and \emph{environment uncertainty} simultaneously.
    This plays a key role in overcoming the distributional shift problem in robust offline RL. 
    Notably, the model estimation subroutine can be flexibly chosen according to the structural conditions of specific RMDP examples. 
    \vspace{3mm}
    \item \textbf{Theoretical guarantees based on a robust partial coverage assumption.} 
    From the theoretical perspective, we characterize the optimality of $\texttt{P}^2\texttt{MPO}$ via the notion of \emph{robust partial coverage coefficient} and \emph{robust model estimation error}.
    The robust partial coverage assumption only requires that the offline dataset has good coverage of distributions induced by the optimal robust policy and the perturbed models around the nominal model.
    In specific, we prove that the suboptimality of $\texttt{P}^2\texttt{MPO}$ is bounded by the \emph{robust model estimation error} (Condition~\ref{cond: model estimation}) and the \emph{robust partial coverage coefficient} (Assumption~\ref{ass: partial coverage}). 
    For concrete examples of RMDPs, by customizing specific model estimation mechanisms and plugging them into $\texttt{P}^2\texttt{MPO}$, we show that $\texttt{P}^2\texttt{MPO}$ enjoys a $\tilde{\mathcal{O}}(n^{-1/2})$ convergence rate with robust partial coverage data, where $n$ is the number of trajectories in the offline dataset. 
    \vspace{3mm}
    \item \textbf{Extension to robust offline multi-agent RL.} 
    As a natural extension of single-agent RMDPs, we also make the first attempt to study offline RL in robust Markov games (RMGs) \citep{kardes2005robust}, wherein the goal is to learn a \emph{robust Nash equilibrium} (RNE). 
    For this multi-agent setting, we extend the double pessimism principle identified for single-agent RMDPs, based on which we propose the \underline{D}oubly \underline{P}essimistic \underline{M}odel-based \underline{M}ulti-agent \underline{P}olicy \underline{O}ptimization ($\texttt{P}^2\texttt{M}^2\texttt{PO}$) algorithm. 
    Similar to $\texttt{P}^2\texttt{M}\texttt{PO}$ for the single-agent setting, $\texttt{P}^2\texttt{M}^2\texttt{PO}$ comprises a model estimation step and a surrogate objective minimization step, where the latter adopts a generalization of the double pessimism principle.
    We further demonstrate that the suboptimality of $\texttt{P}^2\texttt{M}^2\texttt{PO}$ is controlled by the \emph{robust unilateral (partial) coverage coefficient} (Assumption~\ref{assumption:unilateral}) and the \emph{robust model estimation error} (Condition~\ref{cond: model estimation rmg}). 
    Here the newly proposed robust unilateral coverage condition can be regarded as the robust counterpart of the unilateral coverage condition for offline non-robust Markov games (MGs) \citep{zhong2022pessimistic,cui2022offline}. 
    Finally, as in the single-agent setting, by specifying the robust model estimation error for concrete RMG examples, we can generally prove that $\texttt{P}^2\texttt{M}^2\texttt{PO}$ converges to a robust Nash equilibrium at a rate of $\tilde{\mathcal{O}}(n^{-1/2})$ with robust unilateral coverage data, where $n$ is the number of trajectories in the offline dataset. 
\end{itemize}

\vspace{3mm}
\noindent
In summary, our work identifies the first general learning principle, which we call \emph{double pessimism}, for robust offline RL. 
Based on this general principle, we can perform sample-efficient robust offline RL with robust partial coverage data in the context of general function approximation.

\subsection{Related Works}

Our work is related to a line of previous theoretical works on robust RL in RMDPs, offline RL with pessimism principle, and robust Markov games, which we compare respectively in the following.
Also, please see Table~\ref{table:comp} for a summary of our results and a comparison with mostly related works on robust offline RL.

\paragraph*{Robust reinforcement learning in robust Markov decision processes.} 
Robust RL is usually modeled as a robust MDP (RMDP) \citep{iyengar2005robust,el2005robust}, and its planning has been well studied \citep{iyengar2005robust,el2005robust,xu2010distributionally,wang2022policy,wang2022convergence}. 
Recently, robust RL in RMDPs has attracted considerable attention, and a growing body of works studies this problem in the generative model \citep{yang2021towards,panaganti2022sample,si2023distributionally,wang2023finite,yang2023avoiding, xu2023improved, clavier2023towards}, online setting \citep{wang2021online,badrinath2021robust,dong2022online}, and offline setting \citep{zhou2021finite, panaganti2022robust, shi2022distributionally,ma2022distributionally}. Our work focuses on robust offline RL, and we provide a more in-depth comparison with \citet{zhou2021finite,shi2022distributionally,ma2022distributionally} as follows.
Under the full coverage condition (a uniformly lower bounded data distribution), \citet{zhou2021finite} provide the first sample-efficient algorithm for $\mathcal{S} \times \mathcal{A}$-rectangular tabular RMDPs. 
After, \citet{shi2022distributionally} leverage the pessimism principle and design a sample-efficient offline algorithm that only requires robust partial coverage data for $\mathcal{S} \times \mathcal{A}$-rectangular tabular RMDPs. \citet{ma2022distributionally} propose a new $d$-rectangular RMDP and develop a pessimistic style algorithm that can find a near-optimal robust policy with partial coverage data. 
In comparison, we provide a generic algorithm that can not only solve the models in \citet{zhou2021finite,shi2022distributionally,ma2022distributionally}, but can also tackle various newly proposed RMDP models such as  $\mathcal{S} \times \mathcal{A}$-rectangular factored RMDP, $\mathcal{S} \times \mathcal{A}$-rectangular kernel RMDP, and $\mathcal{S} \times \mathcal{A}$-rectangular neural RMDP. 
See Table~\ref{table:comp} for a summary. 
Moreover, we propose a new pessimistic type learning principle ``double pessimism'' for robust offline RL. 
Although \citet{shi2022pessimistic} and \citet{ma2022distributionally} adopt the similar algorithmic idea in tabular or linear settings, neither of them have identified a general learning principle for robust offline RL in the regime of large state spaces.

\begin{table*}[t]
    \centering\resizebox{\columnwidth}{!}{
    \begin{tabular}{ | c | c | c | c | c |  }
    \hline
    & \citet{zhou2021finite}
    & \citet{shi2022distributionally} & \citet{ma2022distributionally}
     & This Work  \\ 
    \hline
    $\mathcal{S} \times \mathcal{A}$-rectangular tabular RMDP & \checkmark!  & \checkmark  & \xmark &  \checkmark  \\
    \hline
    $d$-rectangular linear RMDP   & \xmark  & \xmark  & \checkmark &  \checkmark  \\
    \hline
    \rowcolor{green2}
    $\mathcal{S} \times \mathcal{A}$-rectangular factored RMDP  &  \xmark &  \xmark& \xmark &  \checkmark  \\
    \hline
    \rowcolor{green2}
    $\mathcal{S} \times \mathcal{A}$-rectangular kernel RMDP  &  \xmark &  \xmark& \xmark &  \checkmark  \\
    \hline
    \rowcolor{green2}
    $\mathcal{S} \times \mathcal{A}$-rectangular neural RMDP   &  \xmark &  \xmark& \xmark &  \checkmark  \\
    \hline
    \hline
    \rowcolor{green2}
    $\mathcal{S} \times \mathcal{A}$-rectangular general RMG   &  \texttt{NA} &   \texttt{NA} &  \texttt{NA} &  \checkmark  \\
    \hline
    \end{tabular}
    }
    \caption{ A comparison with the most related works on robust offline RL. 
    \checkmark means that the work can tackle this model with robust partial coverage data, $\checkmark!$ means that the work requires full coverage data to solve the model, and \xmark ~means that the work cannot tackle the model. 
    The light green color denotes the models that are first proposed or proved tractable in this work. 
    }
    \label{table:comp}
    \end{table*}

\paragraph*{Non-robust offline RL and pessimism principle.} 
The line of works on offline RL aims to design efficient learning algorithms that find an optimal policy given an offline dataset collected a priori. 
Prior works \citep{munos2005error,antos2008learning,chen2019information} typically require a dataset of full coverage, which assumes that the offline data have good coverage of all state-action pairs. 
In order to avoid such a strong coverage condition on data, the \emph{pessimism} principle -- being conservative in policy or value estimation of those state-action pairs that are not sufficiently covered by data  -- has been proposed. 
Based on this principle, a long line of works \citep[see e.g.,][]{jin2021pessimism,uehara2021pessimistic,xie2021bellman,xie2021policy,rashidinejad2021bridging,zanette2021provable,yin2021towards,xiong2022nearly,shi2022pessimistic,li2022settling,zhan2022offline,lu2022pessimism,rashidinejad2022optimal} propose algorithms that can learn the optimal policy only with \emph{partial coverage data}.
The partial coverage data only need to cover the state-action pairs visited by the optimal policy.
Among these works, our work is mostly related to \citet{uehara2021pessimistic}, which proposes a generic model-based algorithm for non-robust offline RL. 
Our algorithm for robust offline RL is also in a model-based fashion, and our study covers some models such as $\mathcal{S} \times \mathcal{A}$-rectangular kernel and neural RMDPs whose non-robust counterparts are not studied by \citet{uehara2021pessimistic}. 
More importantly, our algorithm is based on a newly proposed \emph{double pessimism} principle, which is tailored for robust offline RL and is in parallel with the pessimism principle used in non-robust offline RL. 
Also, we show that the performance of our proposed algorithm depends on the notion of \emph{robust partial coverage coefficient}, which is also different from the notions of partial coverage coefficient in previous non-robust offline RL works \citep{jin2021pessimism,xie2021bellman,uehara2021pessimistic}.

As an extension of offline RL in single-agent MDPs, another line of works \citep{zhong2022pessimistic,cui2022provably,cui2022offline,xiong2022nearly,yan2022model,zhang2023offline} studies offline RL in Markov games (MGs) and demonstrates that the \emph{unilateral coverage} condition is the necessary and sufficient coverage condition for sample-efficient non-robust offline RL in MGs. 
In comparison, our work focuses on robust offline RL in RMGs and designs a generic algorithm framework that can efficiently learn a robust Nash equilibrium with only robust unilateral coverage data, which can be regarded as a robust counterpart of unilateral coverage data.

\paragraph*{Robust reinforcement learning in robust Markov games.} 
Robust RL has also been previously considered in a multi-agent setting, and the decision process is modeled as a robust Markov game (RMG) \citep{kardes2005robust, kardecs2011discounted}, wherein the goal is to learn a \emph{robust Nash equilibirum} (RNE) among players.
The concept of RNE takes transition (and reward) uncertainty into consideration, thus rendering it more stable in the face of perturbed environments.
Recently, \cite{zhang2020robust} study policy gradient and actor-critic style algorithms for solving the RNE.
But the sample complexity and the convergence property of their algorithm are unknown.
\cite{ma2023decentralized} study online RL in RMGs and design a decentralized-style algorithm to learn the RNE.
However, robust offline RL in RMGs and its sample-efficiency still remain open (to our best knowledge), which is the focus of our work.

% \cite{kardes2005robust, kardecs2011discounted, zhangk2020robust, ma2023decentralized}.

\subsection{Notations}

For any set $A$, we use $2^A$ to denote the collection of all the subsets of $A$. 
We use $A^{ \complement}$ to denote the complementary set of $A$.
For any measurable space $\mathcal{X}$, we use $\Delta(\mathcal{X})$ to denote the collection of all the probability measures over $\mathcal{X}$.
For any integer $n$, we use $[n]$ to denote the set $\{1,\cdots,n\}$.
Throughout the paper, we use $D(\cdot\|\cdot)$ to denote a (pseudo-)distance between two probability measures (or densities).
In specific, we define the KL-divergence $D_{\mathrm{KL}}(p\|q)$ between two probability densities $p$ and $q$ over $\cX$ as 
\begin{align*}
    D_{\mathrm{KL}}(p\|q) = \int_{\mathcal{X}}p(x)\log\left(\frac{p(x)}{q(x)}\right)\mathrm{d}x,
\end{align*}
and we define the TV-distance $D_{\mathrm{TV}}(p\|q)$ between two probability densities $p$ and $q$ over $\cX$ as 
\begin{align*}
    D_{\mathrm{TV}}(p\|q) = \frac{1}{2}\int_{\mathcal{X}}\left|q(x) - p(x)\right|\mathrm{d}x.
\end{align*}
Given a function class $\mathcal{F}$ equipped with some norm $\|\cdot\|_{\mathcal{F}}$, we denote by $\mathcal{N}_{[]}(\epsilon,\mathcal{F},\|\cdot\|_{\mathcal{F}})$ the $\epsilon$-bracket number of $\mathcal{F}$, and $\mathcal{N}(\epsilon,\mathcal{F},\|\cdot\|_{\mathcal{F}})$ the $\epsilon$-covering number of $\mathcal{F}$.
We denote $\mathcal{P}=\{P(\cdot|\cdot,\cdot): \mathcal{S}\times\mathcal{A}\mapsto\Delta(\mathcal{S})\}$ as the space of transition kernels.
We use $\Delta(\mathcal{A}|\mathcal{S},H)$ to denote the collection $\{\pi=\{\pi_h\}_{h=1}^H|\pi_h(\cdot|\cdot):\mathcal{S}\mapsto\Delta(\mathcal{A})\}$.

%% file: tex/preliminaries.tex
\section{Preliminaries on Robust Markov Decision Processes}\label{sec: preliminaries}

In this section, we introduce robust Markov decision processes (RMDPs) and formulate the offline RL problem.
In Section~\ref{subsec: rmdp framework}, we introduce a unified framework for studying RMDPs in the episodic setting.
In Section~\ref{subsec: offline RL in rmdp} we formulate the problem of offline RL in the proposed framework of RMDPs.

\subsection{A Unified Framework of Robust Markov Decision Processes}\label{subsec: rmdp framework}

We first introduce a unified framework of episodic RMDPs, denote by a tuple $(\mathcal{S}, \mathcal{A}, H, P^{\star}, R, \mathcal{P}_\mathrm{M}, \mathbf{\Phi})$.
The set $\mathcal{S}$ is the state space with possibly \emph{infinite} cardinality. 
The set $\mathcal{A}$ is the action space with finite cardinality.
The integer $H$ is the length of each episode.
The set $P^{\star}=\{P_h^{\star}\}_{h=1}^H$ is the collection of \emph{nominal} transition kernels where $P_h^{\star}:\mathcal{S}\times\mathcal{A}\mapsto\Delta(\mathcal{S})$.
The set $R=\{R_h\}_{h=1}^H$ is the collection of reward functions where $R_h:\mathcal{S}\times\mathcal{A}\mapsto[0,1]$. 
The space $\mathcal{P}_{\text{M}} \subseteq \mathcal{P}$ is a realizable model space which contains the nominal transition kernel $P^{\star}$, i.e., $P_h^{\star}\in\mathcal{P}_{\mathrm{M}}$ for each step $h\in[H]$.

Most importantly and different from standard MDPs, the RMDP is equipped with a mapping $\mathbf{\Phi}:\mathcal{P}_{\text{M}}\mapsto 2^{\mathcal{P}}$ that characterizes the \emph{robust set} of any transition kernel in $\mathcal{P}_{\mathrm{M}}$.
Formally, for any transition kernel $P\in\mathcal{P}_\mathrm{M}$, we call $\mathbf{\Phi}(P)$ the \emph{robust set} of $P$.
One can interpret the nominal transition kernel $P^{\star}_h$ as the transition of the training environment, while $\boldsymbol{\Phi}(P^{\star}_h)$ contains all possible transitions of the test environment.

\begin{remark}
    The mapping $\boldsymbol{\Phi}$ is defined on the realizable model space $\mathcal{P}_{\mathrm{M}}$, while for generality we allow the image of $\boldsymbol{\Phi}$ to be outside of $\mathcal{P}_{\mathrm{M}}$.
    That is, a $\tilde{P}\in\boldsymbol{\Phi}(P)$ for some $P\in\mathcal{P}_{\mathrm{M}}$ might be in $\mathcal{P}_{\mathrm{M}}^{\complement}$.
\end{remark}

\paragraph*{Policy and robust value function.}
Given an RMDP $(\mathcal{S}, \mathcal{A}, H, P^{\star}, R, \mathcal{P}_{\mathrm{M}}, \mathbf{\Phi})$, we consider using a Markovian policy to make decisions.
A Markovian policy $\pi$ is defined as $\pi=\{\pi_h\}_{h=1}^H$ with $\pi_h:\mathcal{S}\mapsto\Delta(\mathcal{A})$ for each step $h\in[H]$.
For simplicity, we use \emph{policy} to refer to a Markovian policy in the sequel.

Given any policy $\pi$, we define the \emph{robust value function} of $\pi$ with respect to any set of transition kernels $P = \{P_h\}_{h=1}^H\subseteq\mathcal{P}_{\mathrm{M}}$ as the following, for each step $h\in[H]$,
\begin{align}
    V_{h, P, \mathbf{\Phi}}^{\pi}(s)&:= \inf_{\tilde{P}_h\in\mathbf{\Phi}(P_h), 1\leq h\leq H} V_h^{\pi}(s; \{\tilde{P}_h\}_{h=1}^H),\quad \forall s\in\mathcal{S},\label{eq: robust V}\\
    Q_{h, P, \mathbf{\Phi}}^{\pi}(s, a)&:= \inf_{\tilde{P}_h\in\mathbf{\Phi}(P_h), 1\leq h\leq H} Q_h^{\pi}(s, a; \{\tilde{P}_h\}_{h=1}^H),\quad \forall (s,a)\in\mathcal{S}\times\mathcal{A}.\label{eq: robust Q}
\end{align}
Here $V_h^{\pi}(\cdot; \{\tilde{P}_h\}_{h=1}^H)$ and $Q_h^{\pi}(\cdot; \{\tilde{P}_h\}_{h=1}^H)$ are the \emph{state-value function} and the \emph{action-value function} \citep{sutton2018reinforcement} of policy $\pi$ in the standard episodic MDP $(\mathcal{S}, \mathcal{A}, H, \{\tilde{P}_h\}_{h=1}^H, R)$, defined as
\begin{align}
    V_h^{\pi}(s; \{\tilde{P}_h\}_{h=1}^H) &:= \mathbb{E}_{\{\tilde{P}_h\}_{h=1}^H,\pi}\left[\sum_{i=h}^HR_{i}(s_i,a_i)\, \middle|\, s_h=s\right],\quad \forall s\in\mathcal{S},\label{eq: V}\\
    Q_h^{\pi}(s, a; \{\tilde{P}_h\}_{h=1}^H) &:= \mathbb{E}_{\{\tilde{P}_h\}_{h=1}^H,\pi}\left[\sum_{i=h}^HR_{i}(s_i,a_i)\, \middle|\, s_h=s,a_h=a\right],\quad \forall (s,a)\in\mathcal{S}\times\mathcal{A},\label{eq: Q}
\end{align}
where the expectation $\mathbb{E}_{\{\tilde{P}_h\}_{h=1}^H,\pi}[\cdot]$ is taken with respect to the trajectories induced by the transition kernels $\{\tilde{P}_h\}_{h=1}^H$ and the policy $\pi$.
Intuitively, the robust value functions \eqref{eq: robust V} and \eqref{eq: robust Q} of a policy $\pi$ given transition kernel $P$ are defined as the least expected cumulative reward achieved by $\pi$ when the transition kernel varies in the robust set of $P$.
This is how an RMDP takes the perturbed models into consideration.

\paragraph*{$\mathcal{S}\times\mathcal{A}$-rectangular robust set and robust Bellman equation.}
Ideally, we would like to consider robust value functions that have recursive expressions, just like the Bellman equations satisfied by \eqref{eq: V} and \eqref{eq: Q} in a standard episodic MDP \citep{sutton2018reinforcement}.
To achieve this, we impose a generally adopted \emph{rectangular} assumption on the robust sets, which is called the \emph{$\mathcal{S}\times\mathcal{A}$-rectangular} assumption \citep{iyengar2005robust}.

\begin{assumption}[$\mathcal{S}\times\mathcal{A}$-rectangular robust set]\label{ass: sarmdp}
    We assume that the mapping $\boldsymbol{\Phi}$ induces $\mathcal{S}\times\mathcal{A}$-rectangular robust sets. 
    More specifically, the mapping $\mathbf{\Phi}$ satisfies, for any $P\in\mathcal{P}_{\mathrm{M}}$,
    \begin{align*}
        \mathbf{\Phi}(P) = \bigotimes_{(s,a)\in\mathcal{S}\times\mathcal{A}} \mathcal{P}(s,a; P),\quad \text{where}\quad  \mathcal{P}(s,a;P)\subseteq\Delta(\mathcal{S}).  %\qquad \mathcal{P}_{\rho}(s,a; P) = \left\{\tilde{P}(\cdot)\in\Delta(\mathcal{S}):D(\tilde{P}(\cdot)\|P(\cdot|s,a))\leq \rho\right\},
    \end{align*}
\end{assumption}

Intuitively, the $\mathcal{S}\times\mathcal{A}$-rectangular assumption requires that the mapping $\boldsymbol{\Phi}(P)$ gives decoupled robust sets for any $P(\cdot|s,a)$ across different $(s,a)$-pairs.
We give specific forms of $\mathcal{P}(\cdot,\cdot; P)$ in Section~\ref{subsec: example RMDP}.
Commonly, one chooses $\mathcal{P}(s,a; P)$ as the set of distributions centered at $P(\cdot|s,a)$.

Now thanks to the $\mathcal{S}\times\mathcal{A}$-rectangular assumption, the robust value functions \eqref{eq: robust V} and \eqref{eq: robust Q} of any policy $\pi$ satisfy a recursive expression, called robust Bellman equations \citep{iyengar2005robust, nilim2005robust}.

\begin{prop}[Robust Bellman equation]\label{prop: robust Bellman}
    Under Assumption \ref{ass: sarmdp}, for any $P=\{P_h\}_{h=1}^H$ where $P_h\in\mathcal{P}_{\mathrm{M}}$ and any $\pi=\{\pi_h\}_{h=1}^H$ with $\pi_h:\mathcal{S}\mapsto\Delta(\mathcal{A})$, the following robust Bellman equations hold, 
    \begin{align}
        V_{h, P, \mathbf{\Phi}}^{\pi}(s) &= \mathbb{E}_{a\sim \pi_h(\cdot|s)}[Q_{h, P, \mathbf{\Phi}}^{\pi}(s, a)],\quad \forall s\in\mathcal{S},\label{eq: robust bellman V}\\
        Q_{h, P, \mathbf{\Phi}}^{\pi}(s, a) &= R_h(s,a) + \inf_{\tilde{P}_h\in\mathbf{\Phi}(P_h)}\mathbb{E}_{s'\sim \tilde{P}_h(\cdot|s,a)}[V_{h+1, P, \mathbf{\Phi}}^{\pi}(s')],\quad \forall (s,a)\in\mathcal{S}\times\mathcal{A}.\label{eq: robust bellman Q}
    \end{align}
\end{prop}

\begin{proof}[Proof of Proposition \ref{prop: robust Bellman}]
    \cite{iyengar2005robust} first showed that $\mathcal{S}\times\mathcal{A}$-rectangular-style robust sets allow for recursive expressions of robust value functions. 
    To be self-contained, in Appendix~\ref{subsec: robust bellman sarmdp} we provide a detailed proof of the robust Bellman equation in our framework of RMDPs under Assumption~\ref{ass: sarmdp}.
\end{proof}

Equations \eqref{eq: robust bellman V} and \eqref{eq: robust bellman Q} actually says that the infimum over all the transition kernels (recall the definition of $V_{h, P, \mathbf{\Phi}}^{\pi}$ in \eqref{eq: robust V}) can be decomposed into a ``one-step" infimum over the transition kernels at step $h$, i.e., $\inf_{\tilde{P}_h\in\mathbf{\Phi}(P_h)}$, and an infimum over the transition kernels at steps larger than $h$, i.e., $V_{h+1, P, \mathbf{\Phi}}^{\pi}$.
Such a property is crucial to the algorithmic design and theoretical analysis for solving RMDPs.

We note that besides the $\cS\times\cA$-rectangular assumption, there are other types of rectangular assumptions considered by robust RL literatures, including $\cS$-rectangular \citep{wiesemann2013robust} and $d$-rectangular \citep{ma2022distributionally} assumptions.
The above framework can also represent RMDPs with these kinds of robust set.
We refer to Section~\ref{sec: discussion} for more discussions about solving RMDPs with these two types of robust sets.

\subsection{Robust Offline RL in Robust Markov Decision Processes}\label{subsec: offline RL in rmdp}

In this subsection, we define the offline RL protocol in an RMDP $(\mathcal{S}, \mathcal{A}, H, P^{\star}, R, \mathcal{P}_{\mathrm{M}}, \mathbf{\Phi})$.
The learner is given the realizable model space $\mathcal{P}_{\mathrm{M}}$ and the robust mapping $\boldsymbol{\Phi}$, but the learner doesn't know the nominal transition kernel $P^{\star}$.
For simplicity, we assume that the learner knows the reward function $R$\footnote{This is a reasonable assumption since learning the reward function is easier than learning the transition kernel.}.

\paragraph*{Offline dataset.}
We assume that the learner is given an offline dataset $\mathbb{D}$ that consists of $n$ i.i.d. trajectories generated from the standard episodic MDP $(\mathcal{S},\mathcal{A},H,P^{\star},R)$ using some behavior policy $\pi^{\mathrm{b}}$.
For each $\tau\in[n]$, the trajectory has the form of $\{(s_h^\tau,a_h^\tau,r_h^{\tau})\}_{h=1}^H$, satisfying that $a_h^{\tau}\sim\pi_h^\mathrm{b}(\cdot|s_h^{\tau})$, $r_h^{\tau}=R_h(s_h^{\tau},a_h^{\tau})$, and $s_{h+1}^{\tau}\sim P_h^{\star}(\cdot|s_h^{\tau},a_h^{\tau})$ for each step $h\in[H]$, starting from some $s_1^{\tau}$.

Given transition kernels $P = \{P_h\}_{h=1}^H$ and a policy $\pi$, we use $d_{P,h}^{\pi}(\cdot,\cdot)$ to denote the state-action visitation distribution at step $h$ when following policy $\pi$ and transition kernel $P$.
With this notation, the distribution of $(s_h^{\tau},a_h^{\tau})$ can be written as $d_{P^{\star},h}^{\pi^{\mathrm{b}}}$ or simply $d_{P^{\star},h}^{\mathrm{b}}$, for each $\tau\in[n]$ and $h\in[H]$.
We also use $d_{P^{\star},h}^{\pi^{\mathrm{b}}}(\cdot)$ to denote the marginal distribution of states at step $h$ when there is no confusion.

\paragraph{Learning objective.} 
In robust offline RL, the goal is to learn the policy $\pi^{\star}$ from the offline dataset $\mathbb{D}$ which maximizes the robust value function $V_{1,P^{\star},\mathbf{\Phi}}^{\pi}$, that is, 
\begin{align}\label{eq: pi star}
    \pi^{\star} := \argsup_{\pi\in\Pi} V_{1,P^{\star},\mathbf{\Phi}}^{\pi}(s_1),\quad s_1\in\mathcal{S},
\end{align}
where the set $\Pi=\{\pi = \{\pi_h\}_{h=1}^H \,\,|\,\,\pi_h:\mathcal{S}\mapsto\Delta(\mathcal{A})\}$ denotes the collection of all Markovian policies. 
In view of \eqref{eq: pi star}, we call $\pi^{\star}$ the \emph{optimal robust policy}.
Equivalently, we want to learn a policy $\hat{\pi}\in\Pi$ which minimizes the suboptimality gap between $\hat{\pi}$ and $\pi^{\star}$.
Formally, the suboptimality gap between $\hat{\pi}$ and $\pi^{\star}$ is defined as\footnote{Without loss of generality, we assume that the initial state is fixed to some $s_1\in\mathcal{S}$. Our algorithm and theory can be directly extended to the case when $s_1\sim \rho\in\Delta(\mathcal{S})$.}
\begin{align}\label{eq: suboptimality}
    \mathrm{SubOpt}_\mathbf{\Phi}(\hat{\pi};s_1):=V_{1,P^{\star},\mathbf{\Phi}}^{\pi^{\star}}(s_1)-V_{1,P^{\star},\mathbf{\Phi}}^{\hat{\pi}}(s_1),\quad \forall s_1\in\mathcal{S}.
\end{align}
In conclusion, the problem of offline RL in RMDPs is to learn a robust policy from an offline dataset generated in some training environment (the nominal transition), which we hope can perform well across all the perturbed test environments (transitions in the robust set of the nominal transition).

%% file: tex/algorithm.tex
\section{Offline RL in RMDPs: Generic Algorithm Framework and Unified Theory}\label{sec: algorithm and theory}

In this section, we propose the \underline{D}oubly \underline{P}essimistic \underline{M}odel-based \underline{P}olicy \underline{O}ptimization ($\texttt{P}^2\texttt{MPO}$) algorithm framework to solve robust offline RL.
Theoretically, we establish a unified suboptimality guarantee for $\texttt{P}^2\texttt{MPO}$.
Our proposed algorithm and theory show that \emph{double pessimism} is a general principle for designing sample-efficient algorithms for robust offline RL.
We highlight that the proposed algorithm features three key points: 
i) learning the optimal robust policy $\pi^{\star}$ approximately; 
ii) requiring only a partial coverage property of the offline dataset $\mathbb{D}$; 
iii) being able to handle infinite state space via powerful function approximators.

We first introduce the algorithm framework $\texttt{P}^2\texttt{MPO}$ in Section~\ref{subsec: algorithm framework}. 
We establish a unified theoretical analysis for $\texttt{P}^2\texttt{MPO}$ in Section~\ref{subsec: theoretical analysis}.
We apply the generic algorithm and theory to concrete RMDP examples in Section~\ref{sec: implementation}.
% Our algorithm framework can be specified to solve all the concrete examples of RMDP we introduce in Section \ref{subsec: example RMDP},
% which we show in Section \ref{sec: implementation}.

\subsection{Algorithm Framework: P$^2$MPO}\label{subsec: algorithm framework}

The $\texttt{P}^2\texttt{MPO}$ algorithm framework (Algorithm \ref{alg: p2mpo}) consists of a \textcolor{blue}{\emph{model estimation step}} and a \textcolor{blue}{\emph{doubly pessimistic policy optimization step}}, which we introduce in the following respectively.

\paragraph{Model estimation step (Line \ref{lin: model estimation}).}
$\texttt{P}^2\texttt{MPO}$ first constructs an estimation of the nominal transition kernels $P^{\star}=\{P^{\star}_h\}_{h=1}^H$ from the offline dataset $\mathbb{D}$, i.e., estimating the dynamics of the training environment.
In specific, $\texttt{P}^2\texttt{MPO}$ implements a sub-algorithm $\texttt{ModelEst}(\mathbb{D},\mathcal{P}_{\mathrm{M}})$ that returns a confidence region $\hat{\mathcal{P}}$ for $P^{\star}$, in the form of
\begin{align*}
    \hat{\mathcal{P}}=\{\hat{\mathcal{P}}_h\}_{h=1}^H,\quad  \text{with} \quad \hat{\mathcal{P}}_h\subseteq\mathcal{P}_{\mathrm{M}}\quad \text{for each step $h\in[H]$}.
\end{align*}
We note that the sub-algorithm \texttt{ModelEst} can be tailored to various concrete RMDP examples.
See Section~\ref{sec: implementation} for detailed implementations of \texttt{ModelEst} for different examples of RMDPs.

Ideally, to ensure sample-efficient robust offline RL, we need the confidence region to satisfy: 
i) the nominal transition kernel $P^{\star}_h$ is contained in $\hat{\mathcal{P}}_h$ for each step $h\in[H]$;
ii) each transition kernel $P_h\in\hat{\mathcal{P}}_h$ enjoys a small ``robust estimation error" which is derived from the robust Bellman equation \eqref{eq: robust bellman V}.
We later characterize these two conditions on $\hat{\mathcal{P}}$ in detail in Section \ref{subsec: theoretical analysis}.

\begin{algorithm}[t]
	\caption{\underline{D}oubly \underline{P}essimistic \underline{M}odel-based \underline{P}olicy \underline{O}ptimization ($\texttt{P}^2\texttt{MPO}$)}
	\label{alg: p2mpo}
	\begin{algorithmic}[1]
	\STATE \textbf{Input}: model space $\mathcal{P}_{\mathrm{M}}$, mapping $\mathbf{\Phi}$, dataset $\mathbb{D}$, policy class $\Pi$, algorithm \texttt{ModelEst}.
    \STATE \textcolor{blue}{\texttt{Model estimation step:}} 
    \STATE Obtain a confidence region $\hat{\mathcal{P}} = \texttt{ModelEst}(\mathbb{D}, \mathcal{P}_{\mathrm{M}})$.\label{lin: model estimation}
    \STATE \textcolor{blue}{\texttt{Doubly pessimistic policy optimization step:}} 
    \STATE Set policy $\hat{\pi}$ as $\argsup_{\pi\in\Pi}J_{\texttt{Pess}^2}(\pi)$, where $J_{\texttt{Pess}^2}(\pi)$ is defined in \eqref{eq: J pess}.\label{lin: doubly pessimistic policy optimization}
    \STATE \textbf{Output}: $\hat\pi=\{\hat\pi_h\}_{h=1}^H$.
	\end{algorithmic}
\end{algorithm}

\paragraph{Doubly pessimistic policy optimization step (Line \ref{lin: doubly pessimistic policy optimization}).}
After \textbf{model estimation step}, $\texttt{P}^2\texttt{MPO}$ performs policy optimization to learn the optimal robust policy $\pi^{\star}$.
In the face of \emph{uncertainties}, $\texttt{P}^2\texttt{MPO}$ adopts a \emph{double pessimism} principle.
To explain, this general principle has two sources of pessimism: 
i) \emph{pessimism} in the face of data uncertainty which originates from statistical estimation of the nominal transition kernels; 
ii) \emph{pessimism} in the face of test environment transition uncertainty which comes from the target of finding a robust policy against the environment perturbation.

We combine these two sources of pessimism in an organic way through a \emph{doubly pessimistic value estimator}.
In specific, for each policy $\pi$, we define its value estimator $J_{\texttt{Pess}^2}$ via an iterative infimum: 
i) an infimum over the confidence region constructed in the \textbf{model estimation step}, i.e., $P_h\in\hat{\cP}_h$; ii) an infimum over the robust set of $P_h$, i.e., $\tilde{P}_h\in\boldsymbol{\Phi}(P_h)$.
Putting together, we define the doubly pessimistic value estimator as
\begin{align}\label{eq: J pess}
    J_{\texttt{Pess}^2}(\pi) := \inf_{P_h\in \hat{\mathcal{P}}_h,1\leq h\leq H}\,\, \inf_{\tilde{P}_h\in \mathbf{\Phi}(P_h), 1\leq h\leq H} \,\, V_1^{\pi}(s_1; \{\tilde{P}_h\}_{h=1}^H),
\end{align}
where $V_1^{\pi}$ is the standard state-value function of policy $\pi$ defined in \eqref{eq: V}.
Then $\texttt{P}^2\texttt{MPO}$ outputs the policy $\hat{\pi}$ that maximizes the doubly pessimistic value estimator $J_{\texttt{Pess}^2}(\pi)$ \eqref{eq: J pess}, i.e., 
\begin{align}\label{eq: hat pi}
    \hat{\pi} := \argsup_{\pi\in\Pi}J_{\texttt{Pess}^2}(\pi).
\end{align}

By performing pessimism from two sources (in the face of data uncertainty and test environment transition uncertainty) in a neat and iterative way, the double pessimism value estimator $J_{\texttt{Pess}^2}$ contrasts with all existing offline RL value estimators.
Compared with existing works on standard offline RL in MDPs \citep{xie2021bellman, uehara2021pessimistic} and robust offline RL in RMDPs \citep{zhou2021finite, yang2021towards, panaganti2022robust}, they only contain one source of pessimism in value estimation and algorithm design. 

Besides, we note that a recent work of \cite{shi2022distributionally} also studies robust offline RL in tabular RMDPs using the principle of pessimism in the face of data uncertainty.
Compared with the double pessimism principle, their algorithm performs pessimism in the face of data uncertainty i) depending on the tabular structure of the model since a point-wise pessimism penalty term based on state-count is needed and ii) depending on the specific form of the robust set $\boldsymbol{\Phi}(P)$.
This makes their algorithm and analysis difficult to adapt to the infinite state space case coped with general types of robust set $\boldsymbol{\Phi}(P)$ and general function approximations.
In contrast, our double-pessimism-based algorithm is capable of handling general RMDPs.
In the coming subsection, we show that with a proper model estimation subroutine implemented, the $\texttt{P}^2\texttt{MPO}$ algorithm can learn the optimal robust policy $\pi^{\star}$ in a statistically efficient manner.

\subsection{Unified Theoretical Analysis}\label{subsec: theoretical analysis}

In this subsection, we establish a unified theoretical analysis for the $\texttt{P}^2\texttt{MPO}$ algorithm framework (Algorithm~\ref{alg: p2mpo}).
We first specify two conditions that the \textbf{model estimation step} of $\texttt{P}^2\texttt{MPO}$ should satisfy in order for sample-efficient learning.
Then we establish an upper bound of the suboptimality \eqref{eq: suboptimality} of the policy obtained by $\texttt{P}^2\texttt{MPO}$ given that these two conditions are satisfied.
% In Section \ref{sec: implementation}, we show that the specific implementations of the sub-algorithm \texttt{ModelEst} for the RMDPs examples in Section \ref{subsec: example RMDP} satisfy these two conditions, which results in tailored suboptimality bounds for these  examples.

\subsubsection{Conditions on Model Estimation Subroutine}\label{subsubsec: condition rmdp}
To achieve sample-efficient robust offline RL, we require the following two conditions on the \textbf{model estimation step} of $\texttt{P}^2\texttt{MPO}$.
The first condition requires the confidence region to contain the nominal transition kernel.

\begin{cond}[$\delta$-accuracy]\label{cond: accuracy}
    With probability at least $1-\delta$, it holds that $P^{\star}_h\in\hat{\mathcal{P}}_h$ for any $h\in[H]$.
\end{cond}

The second condition requires that each transition kernel in the confidence region has some small ``robust estimation error''.
To be specific, we for each transition kernel $P_h\in\cP$ and function $V:\cS\mapsto[0,H]$, we define 
\begin{align}
    \mathcal{E}^{\boldsymbol{\Phi}}_h(s,a;P_h,V) := \inf_{\tilde{P}_h\in\mathbf{\Phi}(P_h)}\mathbb{E}_{s'\sim \tilde{P}_h(\cdot|s,a)}[V(s')] - \inf_{\tilde{P}_h\in\mathbf{\Phi}(P_h^{\star})}\mathbb{E}_{s'\sim \tilde{P}_h(\cdot|s,a)}[V(s')],
\end{align}
where $P_h^{\star}$ is the nominal transition kernel.
Intuitively, $\mathcal{E}^{\boldsymbol{\Phi}}_h$ characterizes the difference in distributionally robust prediction between $P_h$ and $P^{\star}_h$.
The second condition goes as follows.

\begin{cond}[$\delta$-model estimation error]\label{cond: model estimation}
    For some function of the sample size $n$ and failure probability $\delta$ denoted by $\mathrm{Err}_{h}^{\mathbf{\Phi}}(n,\delta)<+\infty$, with probability at least $1-\delta$, it holds that
    \begin{align}\label{eq: model estimation condition}
       \mathbb{E}_{(s,a)\sim d^{\pi^{\mathrm{b}}}_{P^{\star},h}}\left[\left(\mathcal{E}^{\boldsymbol{\Phi}}_h(s,a;P_h,V_{h+1,P,\boldsymbol{\Phi}}^{\pi^{\star}})\right)^2\right] \leq \mathrm{Err}_{h}^{\mathbf{\Phi}}(n,\delta),
    \end{align}
    for any $P = \{P_h\}_{h=1}^H$ with $P_h\in\hat{\mathcal{P}}_h$ for each step $h\in[H]$. 
\end{cond}

A seemingly more natural but stronger version of Condition~\ref{cond: model estimation} is to ensure \eqref{eq: model estimation condition} holds for any function $V$, rather than only for $V_{h+1,P,\boldsymbol{\Phi}}^{\pi^{\star}}$.
But for a valid theoretical analysis it turns out that we only need \eqref{cond: model estimation} to hold for $V_{h+1,P,\boldsymbol{\Phi}}^{\pi^\star}$.
We remark that Condition~\ref{cond: accuracy} is relatively standard for the transition kernel estimation since it does not involve robust sets and thus is a normal statistical estimation property.
Condition~\ref{cond: model estimation} turns out to be more problem-specific since one needs to ensure that the model estimation has an accurate robust prediction.

In Section~\ref{sec: implementation}, we give concrete model estimation subroutine implementations for specific RMDP examples and thus specify Conditions \ref{cond: accuracy} and \ref{cond: model estimation}.
We further prove that all the implementations result in an $\mathrm{Err}_h^{\mathbf{\Phi}}(n,\delta)$ scaling with $\tilde{\mathcal{O}}(n^{-1})$.

\subsubsection{Suboptimality Analysis under Robust Partial Coverage}
Now we establish a suboptimality upper bound for the $\texttt{P}^2\texttt{MPO}$ algorithm framework, given that Conditions~\ref{cond: accuracy} and \ref{cond: model estimation} hold.
To make sample-efficient offline RL possible, it is crucial to make certain \emph{coverage} assumptions on the offline dataset \citep{chen2019information, jin2021pessimism}.
Such assumptions generally require the offline data to cover the trajectories induced by certain polic\emph{ies}.
Thanks to the double pessimism principle of $\texttt{P}^2\texttt{MPO}$, we can prove a suboptimality bound while only making a \emph{robust partial coverage assumption} on the dataset.

\begin{assumption}[Robust partial coverage]\label{ass: partial coverage}
    We assume that the offline dataset satisfies that
    \begin{align}
        C^{\star}_{P^{\star},\mathbf{\Phi}} := \sup_{1\leq h \leq H}\,\,\sup_{P = \{P_h\}_{h=1}^H, P_h\in\mathbf{\Phi}(P_h^{\star})}\,\,\mathbb{E}_{(s,a)\sim d^{\pi^{\mathrm{b}}}_{P^{\star}, h}}\left[\left(\frac{d^{\pi^{\star}}_{P, h}(s,a)}{d^{\pi^{\mathrm{b}}}_{P^{\star}, h}(s,a)}\right)^2\right] < +\infty,
    \end{align}
    and we call $C^{\star}_{P^{\star},\mathbf{\Phi}}$ the robust partial coverage coefficient.
\end{assumption}

To interpret, Assumption \ref{ass: partial coverage} only requires that the dataset covers the visitation distribution of the optimal robust policy $\pi^{\star}$, but in a robust fashion since $C^{\star}_{P^{\star},\mathbf{\Phi}}$ considers all possible transition kernels in the robust set $\boldsymbol{\Phi}(P^{\star})$.
The robust consideration in $C^{\star}_{P^{\star},\mathbf{\Phi}}$ is because in RMDPs the policies are evaluated in a robust way.
To connect to the literature in offline RL in standard MDPs, Assumption \ref{ass: partial coverage} corresponds to the \emph{partial coverage} or \emph{single-policy concentrability} assumption \citep{jin2021pessimism,uehara2021pessimistic,xie2021bellman,xie2021policy,yin2021towards,rashidinejad2021bridging,zanette2021provable,xiong2022nearly,shi2022pessimistic,li2022settling,zhan2022offline,lu2022pessimism,rashidinejad2022optimal}, which requires the offline dataset to cover the trajectories of the optimal policy.
When the robust set mapping $\boldsymbol{\Phi}(P) = \{P\}$, Assumption \ref{ass: partial coverage} reduces to the partial coverage assumption for standard MDPs.

This partial-coverage-style assumption is weaker and more practical than the full-coverage-style assumptions for robust offline RL \citep{yang2021towards, zhou2021finite,panaganti2022robust}, for which they require either a uniformly lower bounded dataset distribution or covering the visitation distribution of any $\pi\in\Pi$.

For $\mathcal{S}\times\mathcal{A}$-rectangular robust tabular MDPs (Example \ref{exp: tabular rmdp}), the robust partial coverage coefficient $C^{\star}_{P^{\star},\mathbf{\Phi}}$ is similar to the \emph{robust single-policy clipped coefficient} $C_{\texttt{rob}}^{\star}$ recently proposed by \cite{shi2022distributionally} who solve offline tabular RMDPs under partial-coverage-style assumptions.
We highlight that beyond $\mathcal{S}\times\mathcal{A}$-rectangular robust tabular MDPs, our robust partial coverage assumption together with the double pessimism algorithm can handle general RMDPs (including examples of RMDPs introduced in Section \ref{subsec: example RMDP}) under our unified theory.
Besides, in the tabular setting, the robust partial coverage coefficient $C_{P^{\star},\boldsymbol{\Phi}}^{\star}$ can be related to the robust single-policy clipped coefficient $C_{\texttt{rob}}^{\star}$ via the inequality $C_{P^{\star},\boldsymbol{\Phi}}^{\star} \leq (C_{\texttt{rob}}^{\star})^2$.

Our main result is the following theorem, which upper bounds the suboptimality of $\texttt{P}^2\texttt{MPO}$.

\begin{theorem}[Suboptimality of $\texttt{P}^2\texttt{MPO}$]\label{thm: subopt general}
    Under Assumptions \ref{ass: sarmdp} and \ref{ass: partial coverage}, suppose that Algorithm \ref{alg: p2mpo} implements a sub-algorithm that satisfies Conditions \ref{cond: accuracy} and \ref{cond: model estimation}, then with probability at least $1-2\delta$,
    \begin{align*}
        \mathrm{SubOpt}(\hat{\pi}; s_1)\leq \sqrt{C^{\star}_{P^{\star},\mathbf{\Phi}}}\cdot\sum_{h=1}^H\sqrt{\mathrm{Err}_{h}^{\mathbf{\Phi}}(n,\delta)}.
    \end{align*}
\end{theorem}

\begin{proof}[Proof of Theorem \ref{thm: subopt general}]
    See Appendix \ref{sec: sketch} for a detailed proof.
\end{proof}

Theorem \ref{thm: subopt general} shows that the suboptimality of $\texttt{P}^2\texttt{MPO}$ is characterized by the robust partial coverage coefficient $C_{P^{\star},\boldsymbol{\Phi}}^{\star}$ (Assumption \ref{ass: partial coverage}) and the sum of model estimation error $\mathrm{Err}_h^{\boldsymbol{\Phi}}$ (Condition \ref{cond: model estimation}).

When $\mathrm{Err}_{h}^{\mathbf{\Phi}}(n,\delta)$ achieves a rate of $\tilde{\mathcal{O}}(n^{-1})$, $\texttt{P}^2\texttt{MPO}$ enjoys a $\tilde{\mathcal{O}}(n^{-1/2})$-suboptimality.
In Section~\ref{sec: implementation}, we give implementations of the model estimation step of $\texttt{P}^2\texttt{MPO}$ for concrete examples of RMDPs in Section~\ref{subsec: example RMDP}.
The implementations will make Conditions~\ref{cond: accuracy} and \ref{cond: model estimation} satisfied and thus specify the general result Theorem~\ref{thm: subopt general}.

%% file: tex/implementation.tex
\section{Implementations of P$^2$MPO for Examples of RMDPs}\label{sec: implementation}

In this section, we provide concrete examples for the unified RMDP framework, based on which we specify the implementation of the \texttt{ModelEst} algorithm in $\texttt{P}^2\texttt{MPO}$ (Algorithm~\ref{alg: p2mpo}) for different examples.
Upon specifying the \texttt{ModelEst} algorithm for a concrete RMDP example, we can then specify the general theory (Theorem~\ref{thm: subopt general}) to this specific case.
Examples of RMDPs are introduced in Section \ref{subsec: example RMDP}. 
Implementations and analysis are in the following subsections.

\subsection{Examples of Robust Markov Decision Processes}\label{subsec: example RMDP}

In this subsection, we give concrete examples for the unified RMDP framework introduced in Section~\ref{subsec: rmdp framework} via specifying the realizable model space $\mathcal{P}_{\mathrm{M}}$ and the robust set mapping $\boldsymbol{\Phi}$.
Most existing works on RMDPs hinge on the finiteness assumption on the state space $\cS$, which fails to deal with prohibitively large or even infinite state spaces. 
In our framework, RMDPs can be studied in the paradigm of infinite state spaces, for which we use function approximation tools for the realizable model space $\mathcal{P}_{\mathrm{M}}$.
For the robust set mapping $\boldsymbol{\Phi}$, we mainly consider $\boldsymbol{\Phi}(P)$ as a distribution ball centered at $P$ as adopted in most existing works.

% Also, we introduce a new type of RMDP named robust factored MDP, which is a robust extension of standard factored MDPs \citep{kearns1999efficient}.

% \begin{remark}
%     Besides $\mathcal{S}\times\mathcal{A}$-rectangular-type robust sets (Assumption \ref{ass: sarmdp}), our unified framework of RMDP can also cover other types of robust sets considered in some previous works as special cases, including $\mathcal{S}$-rectangular robust set \citep{wiesemann2013robust} and $d$-rectangular robust set for linear MDPs \citep{ma2022distributionally}.
%     See Section \ref{sec: discussion} for a discussion about these two types of robust sets.
% \end{remark}

% In the sequel, we introduce concrete examples of our framework of RMDP.

\begin{example}[$\mathcal{S}\times\mathcal{A}$-rectangular robust tabular MDP]\label{exp: tabular rmdp}
    When the state space $\mathcal{S}$ is a finite set, we call the corresponding model an $\mathcal{S}\times\mathcal{A}$-rectangular robust tabular MDP.
    Recently, there is a line of works on the $\mathcal{S}\times\mathcal{A}$-rectangular robust tabular MDP \citep{zhou2021finite, yang2021towards, panaganti2022sample, liu2022distributionally, shi2022distributionally, panaganti2022robust, dong2022online, ho2022robust,neufeld2022robust,wang2023finite, yang2023avoiding, xu2023improved, clavier2023towards}.
    For $\mathcal{S}\times\mathcal{A}$-rectangular robust tabular MDPs, we choose the realizable model space $\mathcal{P}_{\mathrm{M}} = \mathcal{P}$ which contains all possible transition kernels.
    We also choose the robust set mapping $\boldsymbol{\Phi}$ as 
    \begin{align}\label{eq: distribution ball}
        \mathbf{\Phi}(P) = \bigotimes_{(s,a)\in\mathcal{S}\times\mathcal{A}} \mathcal{P}(s,a; P),\quad \text{where}\quad \mathcal{P}_{\rho}(s,a; P) = \left\{\tilde{P}(\cdot)\in\Delta(\mathcal{S}):D(\tilde{P}(\cdot)\|P(\cdot|s,a))\leq \rho\right\},
    \end{align}
    for some (pseudo-)distance $D(\cdot\|\cdot)$ on $\Delta(\mathcal{S})$ and some $\rho\in\mathbb{R}_+$.
    The (pseudo-)distance $D(\cdot\|\cdot)$ can be chosen as a general $\phi$-divergence \citep{yang2021towards} or a $p$-Wasserstein-distance \citep{neufeld2022robust}.
\end{example}

\begin{remark}
    We highlight that our unified framework of RMDPs covers substantially more model than $\mathcal{S}\times\mathcal{A}$-rectangular robust tabular MDPs since our state space $\mathcal{S}$ can be infinite.
    The model space $\mathcal{P}_{\mathrm{M}}$ can be adapted to function approximation methods to handle the infinite state space.
    Thus any efficient algorithm developed for our framework of RMDPs \textbf{can not} be covered by algorithms for $\mathcal{S}\times\mathcal{A}$-rectangular robust tabular MDPs.
    Example \ref{exp: sarmdp kernel} and \ref{exp: sarmdp neural} are infinite state space $\mathcal{S}\times\mathcal{A}$-rectangular robust MDPs with function approximations.
\end{remark}

\begin{example}[$\mathcal{S}\times\mathcal{A}$-rectangular robust MDP with kernel function approximations]\label{exp: sarmdp kernel}
    We consider an infinite state space $\mathcal{S}\times\mathcal{A}$-rectangular robust MDP whose realizable model space $\mathcal{P}_{\mathrm{M}}$ is in a reproduced kernel Hilbert space (RKHS).
    Let $\mathcal{H}$ be an RKHS associated with a positive definite kernel $\mathcal{K}:(\mathcal{S}\times\mathcal{A}\times\mathcal{S})\times(\mathcal{S}\times\mathcal{A}\times\mathcal{S})\mapsto\mathbb{R}_+$ (See Appendix \ref{subsubsec: rkhs} for a review of the basics of RKHS).
    We denote the feature mapping of $\mathcal{H}$ by $\boldsymbol{\psi}:\mathcal{S}\times\mathcal{A}\times\mathcal{S}\mapsto\mathcal{H}$.
    With $\mathcal{H}$, an $\mathcal{S}\times\mathcal{A}$-rectangular robust MDP with kernel function approximation is defined as  an $\mathcal{S}\times\mathcal{A}$-rectangular robust MDP with $\mathcal{P}_{\mathrm{M}}$ given by
    \begin{align}\label{eq: model space kernel}
        \mathcal{P}_{\mathrm{M}} = \Big\{P(s'|s,a) = \langle\boldsymbol{\psi}(s,a,s'), \boldsymbol{f}\rangle_{\mathcal{H}}:\boldsymbol{f}\in\mathcal{H},\|f\|_{\mathcal{H}}\leq B_{\mathrm{K}} \Big\},
    \end{align}
    for some $B_{\mathrm{K}}>0$. 
    In \eqref{eq: model space kernel}, we have implicitly identified $P(\cdot|\cdot,\cdot)$ as the density of the corresponding distribution with respect to a proper base measure on $\mathcal{S}\times\mathcal{A}\times\mathcal{S}$.
    Regarding the robust set mapping $\boldsymbol{\Phi}$, we apply the same choice as \eqref{eq: distribution ball} in Example \ref{exp: tabular rmdp}.    
    %We do the same thing in Example \ref{exp: sarmdp neural}.
\end{example}

\begin{example}[$\mathcal{S}\times\mathcal{A}$-rectangular robust MDP with neural function approximations]\label{exp: sarmdp neural}
    We consider an infinite state space $\mathcal{S}\times\mathcal{A}$-rectangular robust MDP whose realizable model space $\mathcal{P}_{\mathrm{M}}$ is parameterized by an overparameterized neural network.
    We define a two-layer fully-connected neural network on some $\mathcal{X}\subseteq\mathbb{R}^{d_{\mathcal{X}}}$ as 
    \begin{align}\label{eq: nn}
        \mathrm{NN}(\mathbf{x};\mathbf{W},\mathbf{a}) = \frac{1}{\sqrt{m}}\sum_{j=1}^ma_j\sigma(\mathbf{x}^\top\mathbf{w}_j),\quad \forall \mathbf{x}\in\mathcal{X},
    \end{align}
    where $m\in\mathbb{N}_+$ is the number of hidden units, $(\mathbf{W},\mathbf{a})$ is the parameters given by $\mathbf{W} = (\mathbf{w}_1,\cdots,\mathbf{w}_m)\in\mathbb{R}^{d\times m}$, $\mathbf{a} = (a_1,\cdots,a_m)^\top\in\mathbb{R}^m$, $\sigma(\cdot)$ is the activation function.
    Now we assume that the state space $\mathcal{S}\subseteq\mathbb{R}^{d_{\mathcal{S}}}$ for some $d_{\mathcal{S}}\in\mathbb{N}_+$.
    Also, we identify actions via one-hot vectors in $\mathbb{R}^{|\mathcal{A}|}$, i.e., we represent $a\in\mathcal{A}$ by $(0,\cdots,0,1,0,\cdots,0)$ with $1$ in the $a$-th coordinate.
    Let $\mathcal{X} = \mathcal{S}\times\mathcal{A}\times\mathcal{S}$ with $d_{\mathcal{X}} = 2d_{\mathcal{S}} + |\mathcal{A}|$.
    Then an $\mathcal{S}\times\mathcal{A}$-rectangular robust MDP with neural function approximation is defined as an $\mathcal{S}\times\mathcal{A}$-rectangular robust MDP with $\mathcal{P}_{\mathrm{M}}$ given by
    \begin{align}\label{eq: model space neural}
        \mathcal{P}_{\mathrm{M}} = \Big\{P(s'|s,a) = \mathrm{NN}((s,a,s');\mathbf{W},\mathbf{a}^{0}): \|\mathbf{W} - \mathbf{W}^{\mathrm{0}}\|_2\leq B_{\mathrm{N}} \Big\},
    \end{align}
    for some $B_{\mathrm{N}}>0$ and some fixed $(\mathbf{W}^{0},\mathbf{a}^{0})$ which can be interpreted as the initialization. 
    See Appendix~\ref{subsubec: ntk and linearization} for more details about neural function approximations and analysis techniques.  
    Regarding the robust set mapping $\boldsymbol{\Phi}$, we apply the same choice as \eqref{eq: distribution ball} in Example \ref{exp: tabular rmdp}.    
\end{example}

\begin{example}[$\mathcal{S}\times\mathcal{A}$-rectangular robust factored MDP]\label{exp: safrmdp}
    We consider a factored MDP equipped with $\mathcal{S}\times\mathcal{A}$-rectangular factored robust set.
    A standard factored MDP \citep{kearns1999efficient} is defined as follows.
    Let $d\in\mathbb{N}_+$ and $\mathcal{O}$ be a finite set.
    The state space $\mathcal{S}$ is factored as $\mathcal{S} = \mathcal{O}^d$.
    For each $i\in[d]$, $s[i]$ is the $i$-coordinate of $s$ and it is only influenced by $s[\mathrm{pa}_i]$, where $\mathrm{pa}_i\subseteq[d]$.
    In other words, the transition of a factored MDP can be factorized as 
    \begin{align*}
        P_h^{\star}(s'|s,a) = \prod_{i=1}^dP^{\star}_{h,i}(s'[i]|s[\mathrm{pa}_i],a).
    \end{align*}
    Here we let the realizable model space $\mathcal{P}_{\mathrm{M}}$ consist of all the factored transition kernels, i.e., 
    \begin{align}
        \mathcal{P}_{\mathrm{M}} = \left\{P(s'|s,a) = \prod_{i=1}^dP_i(s'[i]|s[\mathrm{pa}_i],a)\,:\,P_i:\mathcal{S}[\mathrm{pa}_i]\times\mathcal{A}\mapsto\Delta(\mathcal{O}),\forall i\in[d]\right\}.
    \end{align}
    For an $\mathcal{S}\times\mathcal{A}$-rectangular robust factored MDP, we define the robust set mapping $\boldsymbol{\Phi}$ as, for each transition kernel $P(s'|s,a) = \prod_{i=1}^dP_i(s'[i]|s[\mathrm{pa}_i],a)\in\mathcal{P}_{\mathrm{M}}$,
    \begin{align*}
        \mathbf{\Phi}(P) &= \bigotimes_{(s,a)\in\mathcal{S}\times\mathcal{A}} \mathcal{P}_{\mathrm{Fac},\rho}(s,a; P),\quad \text{with}\\
        \mathcal{P}_{\mathrm{Fac},\rho}(s,a; P) &= \left\{\prod_{i=1}^d\tilde{P}_i(\cdot):\tilde{P}_i(\cdot)\in\Delta(\mathcal{O}),D(\tilde{P}_i(\cdot)\|P_i(\cdot|s[\mathrm{pa}_i],a))\leq \rho_i,\forall i\in[d]\right\}.
    \end{align*}
    for some (pseudo-)distance $D(\cdot\|\cdot)$ on $\Delta(\mathcal{O})$ and $d$ positive real numbers $\{\rho_i\}_{i=1}^d$.
\end{example}

\begin{remark}
    The $\mathcal{S}\times\mathcal{A}$-rectangular robust factored MDP (Example \ref{exp: safrmdp}) can also be considered in an infinite state space paradigm, but for ease of presentation, we only consider factored MDPs with finite states here. 
\end{remark}

\begin{remark}
    Besides the $\cS\times\cA$-rectangular robust set (Assumption \ref{ass: sarmdp}),
    we refer to Section~\ref{sec: discussion} for discussions about RMDPs with $\cS$-rectangular \citep{wiesemann2013robust} and $d$-rectangular \citep{ma2022distributionally} robust sets.
\end{remark}

\subsection{Model Estimation for General RMDPs with $\mathcal{S}\times\mathcal{A}$-rectangular Robust Sets}\label{subsec: model estimation sarmdp}

In this subsection, we implement the \texttt{ModelEst} algorithm and specify Theorem \ref{thm: subopt general} for general $\cS\times\cA$-rectangular RMDPs with robust sets given by \eqref{eq: distribution ball}.
The proposed implementation and analysis apply to all of the RMDP examples in Section \ref{subsec: example RMDP} (Examples \ref{exp: tabular rmdp}, \ref{exp: sarmdp kernel}, \ref{exp: sarmdp neural}, and \ref{exp: safrmdp}).
In the next subsection, we further give a customized implementation for robust factorized MDPs (Example \ref{exp: safrmdp}) which utilizes the factorization property and results in a refined analysis.

\paragraph{Implementations of the \texttt{ModelEst} algorithm.}
Using the offline data $\mathbb{D}$, we first construct the \emph{maximum likelihood estimator} (MLE) of the transition kernel $P^{\star}$.
Specifically, for each step $h\in[H]$, we define
\begin{align}\label{eq: mle sarmdp}
    \hat{P}_h = \argmax_{P \in \mathcal{P}_{\mathrm{M}}}\frac{1}{n}\sum_{\tau=1}^n\log P(s_{h+1}^\tau|s_h^\tau,a_h^\tau).
\end{align}
After, we construct a confidence region $\widehat{\mathcal{P}}$ for the MLE estimator \eqref{eq: mle sarmdp}, which contains all the transition kernels having a small total variance distance from $\hat{P}$.
In specific, for each step $h\in[H]$, we define 
\begin{align}\label{eq: confidence sarmdp}
    \hat{\mathcal{P}}_h = \bigg\{P\in\mathcal{P}_{\mathrm{M}}: \frac{1}{n}\sum_{\tau=1}^n \|\hat{P}_h(\cdot|s_h^\tau,a_h^\tau) - P(\cdot|s_h^\tau,a_h^\tau)\|_1^2 \leq \xi\bigg\},
\end{align}
where $\xi>0$ is a tuning parameter controlling the size of $\hat{\mathcal{P}}_h$.
Finally, we define that $\texttt{ModelEst}(\mathbb{D},\mathcal{P}_{\mathrm{M}}) = \hat{\mathcal{P}} = \{\hat{\mathcal{P}}_h\}_{h=1}^H$ with $\hat{\mathcal{P}}_h$ given by \eqref{eq: confidence sarmdp}.

\paragraph{Analysis for general $\cS\times\cA$-rectangular RMDPs with robust sets \eqref{eq: distribution ball}.}
In the following, we consider a general RMDP with $\cS\times\cA$-rectangular robust sets satisfying \eqref{eq: distribution ball}. 
We choose the distance $D(\cdot\|\cdot)$ defining the robust set as \emph{KL-divergence} and \emph{TV-distance}. 
The following proposition shows that the above implementation \eqref{eq: confidence sarmdp} of \texttt{ModelEst} in $\texttt{P}^2\texttt{MPO}$ satisfies Conditions \ref{cond: accuracy} and \ref{cond: model estimation} for both KL-divergence and TV-distance.

\begin{prop}[Guarantees for model estimation]\label{prop: sarmdp estimation}
    Under Assumption \ref{ass: sarmdp} and \eqref{eq: distribution ball}, choosing the (pseudo) distance $D(\cdot\|\cdot)$ as KL-divergence or TV-distance, setting the tuning parameter $\xi$ as
    \begin{align*}
        \xi = \frac{C_1\log(C_2H\mathcal{N}_{[]}(1/n^2,\mathcal{P}_{\mathrm{M}},\|\cdot\|_{1,\infty})/\delta)}{n},
    \end{align*}
    for some constants $C_1,C_2>0$, then Conditions \ref{cond: accuracy} and \ref{cond: model estimation} are satisfied respectively by,
    \begin{itemize}
        \item [$\spadesuit$] when $D(\cdot\|\cdot)$ is KL-divergence and Assumption \ref{ass: kl regularity} (See Appendix \ref{subsec: proof prop sarmdp estimation}) holds with parameter $\underline{\lambda}$, $\mathrm{Err}_h^{\mathbf{\Phi}}(n,\delta)$ is given by
        \begin{align*}
            \sqrt{\mathrm{Err}_{h,\mathrm{KL}}^{\mathbf{\Phi}}(n,\delta)} = \frac{H\exp(H/\underline{\lambda})}{\rho}\cdot\sqrt{\frac{C_1'\log(C_2'H\mathcal{N}_{[]}(1/n^2,\mathcal{P}_{\mathrm{M}},\|\cdot\|_{1,\infty})/\delta)}{n}}.
        \end{align*}
        %\han{make sure we have defined $c$}
        \item [$\clubsuit$] when $D(\cdot\|\cdot)$ is TV-distance, $\mathrm{Err}_h^{\mathbf{\Phi}}(n,\delta)$ is given by
        \begin{align*}
            \sqrt{\mathrm{Err}_{h,\mathrm{TV}}^{\mathbf{\Phi}}(n,\delta)} = H\cdot\sqrt{\frac{C_1'\log(C_2'H\mathcal{N}_{[]}(1/n^2,\mathcal{P}_{\mathrm{M}},\|\cdot\|_{1,\infty})/\delta)}{n}}.
        \end{align*}
    \end{itemize}
    Here $c$, $C_1'$, $C_2'>0$ stand for three universal constants.
\end{prop}

\begin{proof}[Proof of Proposition \ref{prop: sarmdp estimation}]
    See Appendix \ref{subsec: proof prop sarmdp estimation} for a detailed proof.
\end{proof}

In Proposition \ref{prop: sarmdp estimation} for KL-divergence-based robust sets, we make a technical assumption (Assumption \ref{ass: kl regularity} given in Appendix~\ref{subsec: proof prop sarmdp estimation}).
This assumption is actually a mild regularity condition for analyzing KL-divergence-based distributional robust optimization problems.
Similar assumptions also appear in \cite{ma2022distributionally}.

Proposition \ref{prop: sarmdp estimation} deals with the general realizable model space $\mathcal{P}_{\mathrm{M}}$.
For the special case of finite model space, the corresponding results basically replace bracket number $\mathcal{N}_{[]}$ of model space $\mathcal{P}_{\mathrm{M}}$ by its cardinality $|\mathcal{P}_{\mathrm{M}}|$.
Now plugging the results of Proposition \ref{prop: sarmdp estimation} in Theorem \ref{thm: subopt general}, we arrive at the following corollary.

\begin{corollary}[Suboptimality of $\texttt{P}^2\texttt{MPO}$: $\mathcal{S}\times\mathcal{A}$-rectangular RMDP]\label{cor: suboptimality sarmdp}
    Under the same assumptions and parameter choice as Theorem \ref{thm: subopt general} and Proposition \ref{prop: sarmdp estimation}, 
    $\texttt{P}^2\texttt{MPO}$ with model estimation step \eqref{eq: confidence sarmdp} satisfies that 
    \begin{itemize}
        \item [$\spadesuit$] when $D(\cdot\|\cdot)$ is KL-divergence and Assumption \ref{ass: kl regularity} holds with parameter $\underline{\lambda}$, then with probability at least $1-2\delta$, 
        \begin{align*}
            \mathrm{SubOpt}(\hat{\pi}; s_1)\leq \frac{\sqrt{C^{\star}_{P^{\star},\mathbf{\Phi}}} H^2\exp(H/\underline{\lambda})}{\rho}\cdot\sqrt{\frac{C_1'\log(C_2'H\mathcal{N}_{[]}(1/n^2,\mathcal{P}_{\mathrm{M}},\|\cdot\|_{1,\infty})/\delta)}{n}}.
        \end{align*}
        \item [$\clubsuit$] when $D(\cdot\|\cdot)$ is TV-divergence, then with probability at least $1-2\delta$, 
        \begin{align*}
            \mathrm{SubOpt}(\hat{\pi}; s_1)\leq \sqrt{C^{\star}_{P^{\star},\mathbf{\Phi}}}H^2\cdot\sqrt{\frac{C_1'\log(C_2'H\mathcal{N}_{[]}(1/n^2,\mathcal{P}_{\mathrm{M}},\|\cdot\|_{1,\infty})/\delta)}{n}}.
        \end{align*}
    \end{itemize}
\end{corollary}

\begin{proof}[Proof of Corollary \ref{cor: suboptimality sarmdp}]
    This is a direct corollary of Theorem \ref{thm: subopt general} and Proposition \ref{prop: sarmdp estimation}.    
\end{proof}

\paragraph*{Analysis for $\mathcal{S}\times\mathcal{A}$-rectangular robust tabular MDP (Example \ref{exp: tabular rmdp}).}
When the state space $\mathcal{S}$ is finite as in Example~\ref{exp: tabular rmdp}, the MLE estimator \eqref{eq: mle sarmdp} coincides the empirical estimator 
\begin{align}
    \hat{P}_h(s'|s,a) = \frac{\sum_{\tau=1}^n\mathbf{1}\{s_h^\tau = s, a_h^{\tau} = a, s_{h+1}^{\tau} = s'\}}{1\vee \sum_{\tau=1}^n\mathbf{1}\{s_h^\tau = s, a_h^{\tau} = a\}},
\end{align}
which is adopted by \cite{zhou2021finite, yang2021towards, panaganti2022sample, shi2022distributionally, panaganti2022robust}. 
Furthermore, in Example \ref{exp: tabular rmdp}, the realizable model space $\mathcal{P}_{\mathrm{M}} = \mathcal{P} = \{P:\mathcal{S}\times\mathcal{A}\mapsto\Delta(\mathcal{S})\}$.
Since both $\mathcal{S}$ and $\mathcal{A}$ are finite, we can bound the bracket number of $\mathcal{P}_{\mathrm{M}}$ as 
\begin{align}\label{eq: bracket number tabular}
    \mathcal{N}_{[]}(1/n^2,\mathcal{P}_{\mathrm{M}},\|\cdot\|_{1,\infty}) \leq n^{2|\mathcal{S}|^2|\mathcal{A}|}.
\end{align}
Combining \eqref{eq: bracket number tabular} and Corollary \ref{cor: suboptimality sarmdp}, we can then conclude that: i) under TV-distance the suboptimality of $\texttt{P}^2\texttt{MPO}$ for $\mathcal{S}\times\mathcal{A}$-rectangular robust tabular MDP is given by $\mathcal{O}(H^2\sqrt{C^{\star}_{P^{\star},\mathbf{\Phi}}|\mathcal{S}|^2|\mathcal{A}|\log(nH/\delta)/n})$, and
ii) under KL-divergence the suboptimality of $\texttt{P}^2\texttt{MPO}$ for $\mathcal{S}\times\mathcal{A}$-rectangular robust tabular MDP is given by $\mathcal{O}(H^2\exp(H/\underline{\lambda})/\rho\cdot\sqrt{C^{\star}_{P^{\star},\mathbf{\Phi}}|\mathcal{S}|^2|\mathcal{A}|\log(nH/\delta)/n})$.
See Appendix \ref{subsec: proof sartmdp} for a proof of \eqref{eq: bracket number tabular}.

% \begin{remark}
%     Compared with the sample complexities for $\mathcal{S}\times\mathcal{A}$-rectangular robust tabular MDPs with TV-distance robust sets \citep{yang2021towards, panaganti2022robust, dong2022online}, our result does not scale with $\rho$. In contrast, their results scale with $\mathcal{O}(\rho^{-1})$ \citep{yang2021towards,panaganti2022robust}.
%     For small $\rho$, our work provides a tighter sample complexity compared with  \cite{yang2021towards, panaganti2022robust}.
% \end{remark}

\begin{remark}
    We note that for KL-divergence-based robust sets, the dependence on $\exp(H)$ is due to the usage of general function approximations, which also appears in a recent work \citep{ma2022distributionally} for RMDPs with linear function approximations.
    For the special case of robust tabular MDPs with KL-divergence-based robust sets, existing work \citep{shi2022distributionally} has derived sample complexities without $\exp(H)$, but with an additional dependence on $1/d_{\min}^{\mathrm{b}}$ and $1/P_{\min}^\star$\footnote{Here $d_{\min}^{\mathrm{b}} = \min_{(s, a, h): d^{\pi^{\mathrm{b}}}_{P^{\star}, h}(s,a) > 0} d^{\pi^{\mathrm{b}}}_{P^{\star}, h}(s,a)$ and $P^{\star}_{\min} = \min_{(s,s',h):P_h(s'|s,\pi_h^\star(s)) > 0} P_h^\star(s'|s,\pi_h^\star(s))$.}.
    We remark that our analysis for $\mathtt{P}^2\mathtt{MPO}$ algorithm can be tailored to the tabular case and become $\exp(H)$-free using their techniques, with the cost of an additional dependence on $1/d_{\min}^{\mathrm{b}}$ and $1/P_{\min}^\star$. 
    %Since $d^{\pi}_{P, h}(s,a)/d^{\pi^{\mathrm{b}}}_{P^{\star}, h}(s,a) \le 1/d_{\min}^{\mathrm{b}}$ for all transition $P$ and policy $\pi$, the results in \citet{shi2022distributionally} implicitly require the full coverage data. 
    But we note that in the infinite state space case, both the $d_{\min}^{\mathrm{b}}$-dependence and the $1/P_{\min}$-dependence  become problematic.
    So, it serves as an interesting future work to answer whether one can derive both $\exp(H)$-free and $(1/d_{\min}^{\mathrm{b}}, 1/P_{\min}^\star)$-free results for (general) function approximations under KL-divergence.
\end{remark}

\paragraph*{Analysis for $\mathcal{S}\times\mathcal{A}$-rectangular robust MDP with kernel function approximations (Example \ref{exp: sarmdp kernel}).}
For kernel function approximations, our theoretical results rely on the following regularity assumptions on the RKHS involved in Example \ref{exp: sarmdp kernel}, which is commonly adopted by the literature on kernel function approximation \citep{yang2020provably,cai2020optimistic,li2022learning}.
Specifically, the kernel $\mathcal{K}$ can be decomposed as $\mathcal{K}(x,y) = \sum_{j=1}^{+\infty}\lambda_j\psi_j(x)\psi_j(y)$ for some $\{\lambda_j\}_{j=1}^{+\infty}\subseteq\mathbb{R}$ and $\{\psi_j:\mathcal{X}\mapsto\mathbb{R}\}_{j=1}^{+\infty}$ with $\mathcal{X} = \mathcal{S}\times\mathcal{A}\times\mathcal{S}$ (See Appendix~\ref{subsec: proof sarmdp kernel} for details).
Our assumption on $\mathcal{K}$ is summarized in the following.

\begin{assumption}[Regularity of RKHS]\label{ass: rkhs}
    We assume that the kernel $\mathcal{K}$ of the RKHS satisfies that:
    \begin{enumerate}
        \item (Boundedness) It holds that $|\mathcal{K}(x,y)|\leq 1$, $|\psi_j(x)|\leq 1$, and $|\lambda_j|\leq 1$ for any $j\in\mathbb{N}_+$, $x,y\in\mathcal{X}$.
        \item (Eigenvalue decay) There exists some $\gamma\in(0,1/2)$, $C_1,C_2>0$ such that $|\lambda_j| \leq C_1\exp(-C_2 j^{\gamma})$ for any $j\in\mathbb{N}_+$.
    \end{enumerate}
\end{assumption}

Under Assumption \ref{ass: rkhs}, we can then upper bound the bracket number $\mathcal{N}_{[]}$ for the realizable model space $\mathcal{P}_{\mathrm{M}}$ defined in \eqref{eq: model space kernel} as (see Appendix \ref{subsubsec: proof bracket kernel} for a proof),
\begin{align}\label{eq: bracket number kernel}
    \log(\mathcal{N}_{[]}(1/n^2,\mathcal{P}_{\mathrm{M}},\|\cdot\|_{1,\infty})) \leq C_{\mathrm{K}}\cdot 1/\gamma\cdot\log^2(1/\gamma)\cdot\log^{1+1/\gamma}(n\mathtt{Vol}(\mathcal{S})B_{\mathrm{K}}),
\end{align}
where $C_{\mathrm{K}}>0$ is an absolute constant, $\mathtt{Vol}(\mathcal{S})$ is the measure of the state space $\mathcal{S}$, and $B_{\mathrm{K}}$ is defined in Example \ref{exp: sarmdp kernel}.
Combining \eqref{eq: bracket number kernel} and Corollary \ref{cor: suboptimality sarmdp}, we can conclude that: i) under TV-distance the suboptimality of $\texttt{P}^2\texttt{MPO}$ for $\mathcal{S}\times\mathcal{A}$-rectangular robust MDP with kernel function approximations is,
\begin{align}
    \mathrm{SubOpt}(\hat{\pi};s_1) \le \mathcal{O}\left(H^2\log(1/\gamma)\cdot\sqrt{C^{\star}_{P^{\star},\mathbf{\Phi}}/\gamma\cdot\log^{1+1/\gamma}(nH\mathtt{Vol}(\mathcal{S})/\delta)/n}\right),
\end{align}
and ii) under KL-divergence the suboptimality of $\texttt{P}^2\texttt{MPO}$ for $\mathcal{S}\times\mathcal{A}$-rectangular robust MDP with kernel function approximations is,
\begin{align}
    \mathrm{SubOpt}(\hat{\pi};s_1) \le \mathcal{O}\left(H^2\exp(H/\underline{\lambda})\log(1/\gamma)/\rho\cdot\sqrt{C^{\star}_{P^{\star},\mathbf{\Phi}}/\gamma\cdot\log^{1+1/\gamma}(nH\mathtt{Vol}(\mathcal{S})/\delta)/n}\right).
\end{align}

\paragraph*{Analysis for $\mathcal{S}\times\mathcal{A}$-rectangular robust MDP with neural function approximations (Example \ref{exp: sarmdp neural}).} For neural function approximation analysis, we use the tool of neural tangent kernel (NTK \citep{jacot2018neural}), which relates overparameterized neural networks  \eqref{eq: nn} to kernel function approximations.
To this end, given the neural network \eqref{eq: nn}, we define its NTK $\mathcal{K}_{\mathrm{NTK}}:\mathcal{X}\times\mathcal{X}\mapsto\mathbb{R}$ as
\begin{align}\label{eq: ntk}
    \mathcal{K}_{\mathrm{NTK}}(x,y) := \nabla_{\mathbf{W}}\mathrm{NN}(x,\mathbf{W}^0,\mathbf{a}^0)^\top\nabla_{\mathbf{W}}\mathrm{NN}(y,\mathbf{W}^0,\mathbf{a}^0),\quad \forall x,y\in\mathcal{X}.
\end{align}

\begin{assumption}[Regularity of Neural Tangent Kernel]\label{ass: ntk regularity}
    We assume that the neural tangent kernel $\mathcal{K}_{\mathrm{NTK}}$ defined in \eqref{eq: ntk} satisfies Assumption \ref{ass: rkhs} with constant $\gamma_{\mathrm{N}}\in(0,1/2)$.
\end{assumption}

This assumption on the spectral perspective of NTK is justified by \cite{yang2019fine}. 
As we prove in Appendix \ref{subsubec: ntk and linearization}, when the number of hidden units is large enough, i.e., overparameterized, the neural network is well approximated by its linear expansion at initialization (Lemma \ref{lem: implicit linearization}), for which we can apply the tool of NTK.
Under Assumption~\ref{ass: ntk regularity}, for the number of hidden units $m\geq d_{\mathcal{X}}n^4B_{\mathrm{N}}^4$, the bracket number $\mathcal{N}_{[]}$ of $\mathcal{P}_{\mathrm{M}}$ defined in \eqref{eq: model space neural} is bounded by (see Appendix~\ref{subsubsec: proof bracket neural} for a proof), 
\begin{align}\label{eq: bracket number neural}
    \log(\mathcal{N}_{[]}(1/n^2,\mathcal{P}_{\mathrm{M}},\|\cdot\|_{1,\infty})) \leq C_{\mathrm{N}}\cdot 1/\gamma_{\mathrm{N}}\cdot\log^2(1/\gamma_{\mathrm{N}})\cdot\log^{1+1/\gamma_{\mathrm{N}}}(n\mathtt{Vol}(\mathcal{S})B_{\mathrm{N}}),
\end{align}
where $C_{\mathrm{N}}>0$ denotes an absolute constant, $\gamma_{\mathrm{N}}\in(0,1/2)$ is specified in Assumption \ref{ass: ntk regularity}, and $B_{\mathrm{N}}$ is defined in Example \ref{exp: sarmdp neural}.
Combining \eqref{eq: bracket number neural} and Corollary \ref{cor: suboptimality sarmdp}, we can conclude that, in the overparameterized paradigm, i.e., $m\geq d_{\mathcal{X}}n^4B_{\mathrm{N}}^4$: i) under TV-distance the suboptimality of $\texttt{P}^2\texttt{MPO}$ for $\mathcal{S}\times\mathcal{A}$-rectangular robust MDP with neural function approximations is,
\begin{align}
    \mathrm{SubOpt}(\hat{\pi};s_1) \le \mathcal{O}\left(H^2\log(1/\gamma_{\mathrm{N}})\cdot\sqrt{C^{\star}_{P^{\star},\mathbf{\Phi}}/\gamma_{\mathrm{N}}\cdot\log^{1+1/\gamma_{\mathrm{N}}}(nH\mathtt{Vol}(\mathcal{S})/\delta)/n}\right),
\end{align}
and ii) under KL-divergence the suboptimality of $\texttt{P}^2\texttt{MPO}$ for $\mathcal{S}\times\mathcal{A}$-rectangular robust MDP with neural function approximations is,
\begin{align}
    \mathrm{SubOpt}(\hat{\pi};s_1) \le \mathcal{O}\left(H^2\exp(H/\underline{\lambda})\log(1/\gamma_{\mathrm{N}})/\rho\cdot\sqrt{C^{\star}_{P^{\star},\mathbf{\Phi}}/\gamma_{\mathrm{N}}\cdot\log^{1+1/\gamma_{\mathrm{N}}}(nH\mathtt{Vol}(\mathcal{S})/\delta)/n}\right).
\end{align}

\subsection{Model Estimation for $\mathcal{S}\times\mathcal{A}$-rectangular Robust Factored MDP (Example~\ref{exp: safrmdp})}\label{subsec: model estimation safrmdp}

In this subsection, we propose an customized implementation of the \texttt{ModelEst} algorithm for $\cS\times\cA$-rectangular robust factored MDPs (Example \ref{exp: safrmdp}), resulting in a refined theoretical analysis for this specific RMDP example.

We first construct MLE estimator for each factor $P^{\star}_{h,i}$ of the transition $P^{\star}_{h} = \prod_{i=1}^dP^{\star}_{h,i}$, that is, 
\begin{align}\label{eq: mle safrmdp}
    \hat{P}_{h,i} = \argmax_{P_i:\mathcal{S}[\mathrm{pa}_i]\times\mathcal{A}\mapsto\Delta(\mathcal{O})}\frac{1}{n}\sum_{k=1}^n\log P(s_{h+1}^{\tau}[i]|s_h^{\tau}[\mathrm{pa}_i],a_h^{\tau}).
\end{align}
Then given $\{\hat{P}_{h,i}\}_{i=1}^d$ we construct a confidence region $\hat{\mathcal{P}}$ that is factored across $i\in[d]$.
Specifically, we define $\hat{\mathcal{P}}_h$ for each step $h\in[H]$ as 
\begin{align}\label{eq: confidence safrmdp}
    \hat{\mathcal{P}}_h = \left\{P(s'|s,a) = \prod_{i=1}^dP_i(s'[i]|s[\mathrm{pa}_i],a): \frac{1}{n}\sum_{i=1}^n\|(P_i-\hat{P}_{h,i})(\cdot|s_h^{\tau}[\mathrm{pa}_i],a_h^{\tau})\|_1^2\leq \xi_i,\forall i\in[d]\right\}.
\end{align}
Finally, we set $\mathtt{ModelEst}(\mathbb{D},\mathcal{P}_{\mathrm{M}})$ = $\hat{\mathcal{P}} = \{\hat{\mathcal{P}}\}_{h=1}^H$ with $\hat{\mathcal{P}}_h$ given by \eqref{eq: confidence safrmdp}.

\begin{prop}[Guarantees for model estimation]\label{prop: safrmdp estimation}
    Suppose the RMDP is the $\mathcal{S}\times\mathcal{A}$-rectangular robust factored MDP in Example \ref{exp: safrmdp} with $D(\cdot\|\cdot)$ being KL-divergence or TV-distance. 
    By choosing the tuning parameter $\xi_i$ defined in \eqref{eq: confidence safrmdp} as
    \begin{align*}
        \xi_i = \frac{C_1|\mathcal{O}|^{1+|\mathrm{pa}_i|}|\mathcal{A}|\log(C_2ndH/\delta)}{n}
    \end{align*}
    for constants $C_1,C_2>0$ and each $i\in[d]$, then Conditions \ref{cond: accuracy} and \ref{cond: model estimation} are satisfied respectively by,
    \begin{itemize}
        \item [$\spadesuit$] when $D(\cdot\|\cdot)$ is KL-divergence and Assumption \ref{ass: kl regularity frmdp} (given in Appendix \ref{subsec: proof prop safrmdp estimation}) holds with parameter $\underline{\lambda}$, then $\mathrm{Err}_h^{\mathbf{\Phi}}(n,\delta)$ is given by
        \begin{align*}
            \sqrt{\mathrm{Err}_{h,\mathrm{KL}}^{\mathbf{\Phi}}(n,\delta)} = \frac{H\exp(H/\underline{\lambda})}{\rho_{\min}}\cdot\sqrt{\frac{dC_1'\sum_{i=1}^d|\mathcal{O}|^{1+|\mathrm{pa}_i|}|\mathcal{A}|\log(C_2'nd/\delta)}{n}},
        \end{align*}
        where $\rho_{\min} = \min_{i\in[d]}\rho_{i}$.
        \item [$\clubsuit$] when $D(\cdot\|\cdot)$ is TV-distance, then $\mathrm{Err}_h^{\mathbf{\Phi}}(n,\delta)$ is given by
        \begin{align*}
            \sqrt{\mathrm{Err}_{h,\mathrm{KL}}^{\mathbf{\Phi}}(n,\delta)} = H\cdot\sqrt{\frac{dC_1'\sum_{i=1}^d|\mathcal{O}|^{1+|\mathrm{pa}_i|}|\mathcal{A}|\log(C_2'nd/\delta)}{n}}.
        \end{align*}
    \end{itemize}
    Here $c$, $C_1'$, $C_2'>0$ stand for three universal constants.
\end{prop}

\begin{proof}[Proof of Proposition \ref{prop: safrmdp estimation}]
    See Appendix \ref{subsec: proof prop safrmdp estimation} for a detailed proof.
\end{proof}

\begin{corollary}[Suboptimality of $\texttt{P}^2\texttt{MPO}$: $\mathcal{S}\times\mathcal{A}$-rectangular robust factored MDP]\label{cor: suboptimality safrmdp}
    Supposing the RMDP is an $\mathcal{S}\times\mathcal{A}$-rectangular robust factored MDP, under the same Assumptions and parameter choice in Theorem~\ref{thm: subopt general} and Proposition~\ref{prop: safrmdp estimation}, 
    $\texttt{P}^2\texttt{MPO}$ with model estimation step given by \eqref{eq: confidence safrmdp} satisfies
    \begin{itemize}
        \item [$\spadesuit$] when $D(\cdot\|\cdot)$ is KL-divergence and Assumption \ref{ass: kl regularity} holds with parameter $\underline{\lambda}$, then with probability at least $1-2\delta$, 
        \begin{align*}
            \mathrm{SubOpt}(\hat{\pi}; s_1)\leq \frac{\sqrt{C^{\star}_{P^{\star},\mathbf{\Phi}}} H^2\exp(H/\underline{\lambda})}{\rho_{\min}}\cdot\sqrt{\frac{dC_1'\sum_{i=1}^d|\mathcal{O}|^{1+|\mathrm{pa}_i|}|\mathcal{A}|\log(C_2'nd/\delta)}{n}}.
        \end{align*}
        \item [$\clubsuit$] when $D(\cdot\|\cdot)$ is TV-divergence, then with probability at least $1-2\delta$, 
        \begin{align*}
            \mathrm{SubOpt}(\hat{\pi}; s_1)\leq \sqrt{C^{\star}_{P^{\star},\mathbf{\Phi}}}H^2\cdot\sqrt{\frac{dC_1'\sum_{i=1}^d|\mathcal{O}|^{1+|\mathrm{pa}_i|}|\mathcal{A}|\log(C_2'nd/\delta)}{n}}.
        \end{align*}
    \end{itemize}
\end{corollary}

\begin{proof}[Proof of Corollary \ref{cor: suboptimality safrmdp}]
    This is a direct corollary of Theorem \ref{thm: subopt general} and Proposition \ref{prop: safrmdp estimation}.    
\end{proof}

Compared with the suboptimality bounds for $\mathcal{S}\times\mathcal{A}$-rectangular robust MDPs in Section \ref{subsec: model estimation sarmdp}, the suboptimality of $\mathcal{S}\times\mathcal{A}$-rectangular robust factored MDPs with \texttt{ModelEst} given in \eqref{eq: confidence safrmdp} only scales with $\sum_{i=1}^d|\mathcal{O}|^{1+|\mathrm{pa}_i|}$ instead of scaling with $|\mathcal{S}| = \prod_{i=1}^d|\mathcal{O}|$ which is of order $\exp(d)$.
This justifies the statistical benefits of considering $\mathcal{S}\times\mathcal{A}$-rectangular robust factored MDPs when the transition kernels of training and testing environments enjoy factored structures.

%% file: tex/game.tex
\section{Multi-Agent Extensions: Offline Reinforcement Learning in Robust Markov Games}

In this section, we extend the theory of offline RL in robust single-agent MDPs to the multi-agent setting, i.e., Markov games (MGs).
To this end, we first introduce the robust counterpart of standard Markov games, known as robust Markov games (RMGs), which feature transition robustness.
Following the notation in Section~\ref{sec: preliminaries}, we propose a unified framework of RMGs in Section \ref{subsec: robust markov game}.
We define the learning objective and offline RL protocol in Sections~\ref{subsec: learning objective rmg} and \ref{subsec: offline rl rmg}, respectively.
In Section \ref{subsec: algorithm and theory rmg}, we extend the ``double pessimism" principle identified in Section \ref{sec: algorithm and theory} for RMDPs to RMGs and design a generic algorithm framework to solve RMGs sample-efficiently.

\subsection{A Unified Framework of Robust Markov Games}\label{subsec: robust markov game}

We propose a unified framework of episodic RMGs of $N$ players, denoted as $(\mathcal{S},\{\mathcal{A}^i\}_{i=1}^N, H, P^{\star}, \{R^i\}_{i=1}^N, \mathcal{P}_{\mathrm{M}}, \boldsymbol{\Phi})$.
The set $\mathcal{S}$ is the state space shared by all players, with a possibly \emph{infinite} cardinality.
The set $\mathcal{A}^{i}$ is the finite action space for player $i$.
The integer $H$ is the length of each episode.
The set $P^{\star}=\{P_h^{\star}\}_{h=1}^H$ is the collection of nominal transition kernels where $P_h^{\star}:\mathcal{S}\times\mathcal{A}^1\times\cdots\times\mathcal{A}^N\mapsto\Delta(\mathcal{S})$.
The set $R^i=\{R_h^i\}_{h=1}^H$ is the collection of reward functions for player $i$, where each $R_h^i:\mathcal{S}\times\mathcal{A}^1\times\cdots\times\mathcal{A}^N\mapsto[0,1]$. 
Let $\mathcal{A} = \mathcal{A}^{1}\times\cdots\times\mathcal{A}^N$, and we use $\boldsymbol{a} = (a^1,\cdots,a^N)\in\mathcal{A}$ to denote the joint action of $N$ players.
The space $\mathcal{P}_{\text{M}} \subseteq \mathcal{P}$ is a realizable model space which contains the nominal transition kernel $P^{\star}$, i.e., $P_h^{\star}\in\mathcal{P}_{\mathrm{M}}$ for each step $h\in[H]$.

Similar to RMDPs, the robust Markov game features a robust set of the transition kernels, which is induced by a mapping $\boldsymbol{\Phi}:\mathcal{P}_{\mathrm{M}}\mapsto 2^{\mathcal{P}}$.
For any transition kernel $P\in\mathcal{P}_\mathrm{M}$, we call $\mathbf{\Phi}(P)$ the \emph{robust set} of $P$.
For RMGs, we also focus on robust sets that are $\cS\times\cA$-rectangular, which is the following assumption.

\begin{assumption}[$\mathcal{S}\times\mathcal{A}$-rectangular robust set]\label{ass: s a rectangular}
    We assume that the mapping $\boldsymbol{\Phi}$ induces $\mathcal{S}\times\mathcal{A}$-rectangular robust sets. 
    Specifically, the mapping $\mathbf{\Phi}$ satisfies, for any $ P\in\mathcal{P}_{\mathrm{M}}$,
    \begin{align*}
        \mathbf{\Phi}(P) = \bigotimes_{(s,\boldsymbol{a})\in\mathcal{S}\times\mathcal{A}} \mathcal{P}_{\rho}(s,\boldsymbol{a}; P),\quad \text{where}\quad \mathcal{P}_{\rho}(s,\boldsymbol{a}; P) \subseteq\Delta(\cS). %= \left\{\tilde{P}(\cdot)\in\Delta(\mathcal{S}):D(\tilde{P}(\cdot)\|P(\cdot|s,\boldsymbol{a}))\leq \rho\right\},
    \end{align*}
\end{assumption}

\paragraph*{Joint policy and robust value function.}
Given an RMG $(\mathcal{S},\{\mathcal{A}^i\}_{i=1}^N, H, P^{\star}, \{R^i\}_{i=1}^N, \mathcal{P}_{\mathrm{M}}, \boldsymbol{\Phi})$, we consider all players using Markovian policies to play.
We denote a Markovian policy of player $i$ by $\pi^i=\{\pi_h^i\}_{h=1}^H$ with $\pi_h^i:\mathcal{S}\mapsto\Delta(\mathcal{A}^i)$ for each step $h\in[H]$.
A \emph{product} Markovian joint policy of the $N$ players is denoted by $\boldsymbol{\pi} = (\pi^1,\cdots,\pi^N)$.
We use \emph{joint policy} to refer to a product Markovian joint policy in the sequel.
For each player $i\in[N]$, we use $\boldsymbol{\pi}^{-i}$ to denote the joint policy of all players except player $i$, i.e., $\boldsymbol{\pi}^{-i} = (\pi^1,\cdots,\pi^{i-1},\pi^{i+1},\cdots,\pi^N)$.

Given any joint policy $\boldsymbol{\pi}$, we define the \emph{robust value function} of $\boldsymbol{\pi}$ and player $i\in[N]$ with respect to any set of transition kernels $P = \{P_h\}_{h=1}^H\subseteq\mathcal{P}_{\mathrm{M}}$ as the following, for each step $h\in[H]$,
\begin{align}
    V_{h, P, \mathbf{\Phi}}^{\boldsymbol{\pi},i}(s)&:= \inf_{\tilde{P}_h\in\mathbf{\Phi}(P_h), 1\leq h\leq H} V_h^{\boldsymbol{\pi}, i}(s; \{\tilde{P}_h\}_{h=1}^H),\quad \forall s\in\mathcal{S},\label{eq: robust V rmg}\\
    Q_{h, P, \mathbf{\Phi}}^{\boldsymbol{\pi},i}(s, \boldsymbol{a})&:= \inf_{\tilde{P}_h\in\mathbf{\Phi}(P_h), 1\leq h\leq H} Q_h^{\boldsymbol{\pi}, i}(s, \boldsymbol{a}; \{\tilde{P}_h\}_{h=1}^H),\quad \forall (s,\boldsymbol{a})\in\mathcal{S}\times\mathcal{A},\label{eq: robust Q rmg}
\end{align}
where $V_h^{\boldsymbol{\pi},i}(\cdot; \{\tilde{P}_h\}_{h=1}^H)$ and $Q_h^{\boldsymbol{\pi},i}(\cdot,\cdot; \{\tilde{P}_h\}_{h=1}^H)$ are the \emph{state-value function} and \emph{state-action value function} of policy $\boldsymbol{\pi}$ and player $i$ in a standard episodic MG given by $(\mathcal{S}, \{\mathcal{A}^i\}_{i=1}^N, H, \{\tilde{P}_h\}_{h=1}^H, \{R^i\}_{i=1}^N)$.
More specifically,
\begin{align}
    V_h^{\boldsymbol{\pi},i}(s; \{\tilde{P}_h\}_{h=1}^H) &:= \mathbb{E}_{\{\tilde{P}_h\}_{h=1}^H,\boldsymbol{\pi}}\left[\sum_{h'=h}^HR_{h'}^i(s_{h'},\boldsymbol{a}_{h'}) \,\middle|\, s_h=s\right],\quad \forall s\in\mathcal{S},\label{eq: V rmg}\\
    Q_h^{\boldsymbol{\pi},i}(s, \boldsymbol{a}; \{\tilde{P}_h\}_{h=1}^H) &:= \mathbb{E}_{\{\tilde{P}_h\}_{h=1}^H,\boldsymbol{\pi}}\left[\sum_{h'=h}^HR_{h'}^i(s_{h'},\boldsymbol{a}_{h'}) \,\middle|\, s_h=s,\boldsymbol{a}_h=\boldsymbol{a}\right],\quad \forall (s,\boldsymbol{a})\in\mathcal{S}\times\mathcal{A},\label{eq: Q rmg}
\end{align}
where the expectation $\mathbb{E}_{\{\tilde{P}_h\}_{h=1}^H,\boldsymbol{\pi}}[\cdot]$ is taken with respect to the trajectories induced by the transition kernel $\{\tilde{P}_h\}_{h=1}^H$ and the joint policy $\boldsymbol{\pi}$.
Parallel to the robust Bellman equation for single-agent RMDPs, we also have the following multi-agent robust Bellman equation.

\begin{prop}[Multi-agent robust Bellman equation]\label{prop: robust Bellman rmg}
    Under Assumption \ref{ass: s a rectangular}, for any transition kernels $P=\{P_h\}_{h=1}^H\subseteq\mathcal{P}_{\mathrm{M}}$ and any joint policy $\boldsymbol{\pi}=\{\boldsymbol{\pi}_h\}_{h=1}^H$, the following robust Bellman equations hold, 
    \begin{align}
        V_{h, P, \mathbf{\Phi}}^{\boldsymbol{\pi},i}(s) &= \mathbb{E}_{\boldsymbol{a}\sim \pi_h(\cdot|s)}[Q_{h, P, \mathbf{\Phi}}^{\boldsymbol{\pi},i}(s, \boldsymbol{a})],\quad \forall s\in\mathcal{S}.\label{eq: robust bellman V rmg}\\
        Q_{h, P, \mathbf{\Phi}}^{\boldsymbol{\pi},i}(s, \boldsymbol{a}) &= R_h^i(s,\boldsymbol{a}) + \inf_{\tilde{P}_h\in\mathbf{\Phi}(P_h)}\mathbb{E}_{s'\sim \tilde{P}_h(\cdot|s,\boldsymbol{a})}[V_{h+1, P, \mathbf{\Phi}}^{\boldsymbol{\pi},i}(s')],\quad \forall (s,\boldsymbol{a})\in\mathcal{S}\times\mathcal{A}.\label{eq: robust bellman Q rmg}
    \end{align}
    for each player $i\in[N]$.
\end{prop}

\begin{proof}[Proof of Proposition \ref{prop: robust Bellman rmg}]
    The proof of this proposition is the same as that of Proposition \ref{prop: robust Bellman}, and we omit it here to avoid repetition.
\end{proof}

\subsection{Robust Solution Concept: Robust Nash Equilibrium}\label{subsec: learning objective rmg}

In a standard MG, the players in the game seek to achieve the \emph{Nash equilibrium policy}, which maximizes each player's own value function given other players' policies \citep{filar2012competitive}.
To take transition robustness into consideration, in an RMG, the players want to maximize their own \emph{robust} value functions \citep{kardes2005robust, kardecs2011discounted, zhang2020robust, ma2023decentralized}, leading to \emph{robust Nash equilibrium}.
In the sequel, we give a formal definition of such a solution concept.
To this end, we first define the \emph{robust best response}.

\begin{definition}[Robust best response]\label{def: robust best response}
    Given transition kernel $P = \{P_h\}_{h\in[H]}\subseteq\mathcal{P}_{\mathrm{M}}$ and joint policy $\boldsymbol{\pi} = (\pi^i,\boldsymbol{\pi}^{-i})$, we say policy $\pi^i$ of player $i\in[N]$ is a robust best response policy with respect to $P$ and $\boldsymbol{\pi}^{-i}$ if for any state $s_1\in\mathcal{S}$, it holds that
    \begin{align*}
        V_{1,P,\boldsymbol{\Phi}}^{\boldsymbol{\pi},i}(s_1) = \sup_{\tilde{\pi}^i\in\Delta(\mathcal{A}^i|\mathcal{S},H)}V_{1,P,\boldsymbol{\Phi}}^{(\tilde{\pi}^i,\boldsymbol{\pi}^{-i}),i}(s_1).
    \end{align*}
    Correspondingly, we denote the best response policy as $\mathrm{br}_{P,\boldsymbol{\Phi}}(\boldsymbol{\pi}^{-i})$ and denote the robust value functions of the joint policy $\boldsymbol{\pi} = (\mathrm{br}_{P,\boldsymbol{\Phi}}(\boldsymbol{\pi}^{-i}), \boldsymbol{\pi}^{-i})$ as $V_{h,P,\boldsymbol{\Phi}}^{(\dagger,\boldsymbol{\pi}^{-i}),i}$ and $Q_{h,P,\boldsymbol{\Phi}}^{(\dagger,\boldsymbol{\pi}^{-i}),i}$.
\end{definition}

The robust best response policy extends the definition of \emph{best response policy} in standard MGs in the sense that it requires maximizing the \emph{robust} value function of player $i$ given other players' policies, thus taking the transition robustness into consideration.
For a joint policy $\boldsymbol{\pi} = (\pi^1,\cdots,\pi^N)$, when each player's policy $\pi^i$ is a robust best response policy against $\boldsymbol{\pi}^{-i}$, we call this joint policy a \emph{robust Nash equilibrium}.

\begin{definition}[Robust Nash equilibrium (RNE)]\label{def: robust nash equilibrium}
    Given transition kernel $P = \{P_h\}_{h\in[H]}\subseteq\mathcal{P}_{\mathrm{M}}$, we say a joint policy $\boldsymbol{\pi} = \{\boldsymbol{\pi}_h\}_{h\in[H]}$ a robust Nash equilibrium policy with respect to $P$ if for any state $s_1\in\mathcal{S}$ and player $i\in[N]$ it holds that
    \begin{align*}
        V_{1,P,\boldsymbol{\Phi}}^{\boldsymbol{\pi},i}(s_1)=V_{1,P,\boldsymbol{\Phi}}^{(\dagger,\boldsymbol{\pi}^{-i}),i}(s_1).
    \end{align*}
\end{definition}

As a special case, when the robust set mapping satisfies $\boldsymbol{\Phi}(P) = \{P\}$, then the RMG reduces to a standard episodic MG, and the definitions of robust best response and robust Nash equilibrium reduce to best response and Nash equilibrium in standard episode MGs, respectively.

The following Theorem shows that in an RMG, the robust Nash equilibrium always exists.

\begin{theorem}[Existence of robust Nash equilibrium]\label{thm: existence of robust nash equilibrium}
    i) Given an RMG $(\mathcal{S},\{\mathcal{A}^i\}_{i=1}^N, H, P, \{R^i\}_{i=1}^N, \mathcal{P}_{\mathrm{M}}, \boldsymbol{\Phi})$, under Assumption \ref{ass: s a rectangular}, the robust Nash equilibrium policy defined in Definition \ref{def: robust nash equilibrium} always exists.
    ii) Consider a joint policy $\boldsymbol{\pi} = \{\boldsymbol{\pi}_h\}_{h=1}^H$ defined as the following\footnote{The definition is in a backward fashion. The policy $\boldsymbol{\pi}_h$ is defined via $Q_{h,P,\boldsymbol{\Phi}}^{\boldsymbol{\pi},i}$ which only depends on $\{\boldsymbol{\pi}_{h'}\}_{h'=h+1}^H$},
    \begin{align}\label{eq: rne bellman equation}
        \boldsymbol{\pi}_h(\cdot|s) = \mathbf{NE}\Big(\big\{Q_{h,P,\boldsymbol{\Phi}}^{\boldsymbol{\pi},i}(s,\cdot)\big\}_{i=1}^N\Big),
    \end{align}
    for step $h = H,\cdots,1$, where $\mathbf{NE}(\cdot)$ denotes the Nash equilibrium of normal form games. 
    Then $\boldsymbol{\pi}$ is a robust Nash equilibrium policy defined in Definition \ref{def: robust nash equilibrium}.
\end{theorem}

\begin{proof}[Proof of Theorem \ref{thm: existence of robust nash equilibrium}]
    See Appendix \ref{subsec: proof prop rne bellman equation} for a detailed proof.
\end{proof}

The conclusion i) of Theorem~\ref{thm: existence of robust nash equilibrium} is a take-away of the conclusion ii) which gives a concrete construction of the robust Nash equilibrium.
By ii), given an RMG, to find its robust Nash equilibrium, it suffices to call a standard Nash equilibrium oracle iteratively, where we input the \emph{robust} value functions to the oracle.

\subsection{Robust Offline RL in Robust Markov Games}\label{subsec: offline rl rmg}

Now we study offline RL in RMGs which aims to learn the robust Nash equilibrium policy purely from an offline dataset.
Specifically, we assume access to an offline dataset $\mathbb{D}$, which consists of $n$ i.i.d. trajectories induced by the standard Markov game $(\cS, \{\cA\}_{i=1}^N, H, P^\star, \{R^i\})$ and some behavior policy $\boldsymbol{\pi}^{\mathrm{b}} = \{\boldsymbol{\pi}_h^{\mathrm{b}}\}_{h=1}^H$. 
In specific, for each $\tau \in [n]$, the trajectory $\{(s_h^\tau, \boldsymbol{a}_h^\tau, \boldsymbol{r}_{h}^\tau)\}_{h =1}^H$ satisfies $\boldsymbol{a}_h^\tau \sim \boldsymbol{\pi}_h^{\mathrm{b}}(\cdot \mid s_h)$, $\boldsymbol{r}_h^\tau = \{R_h^i(s_h^\tau, \boldsymbol{a}_h^i)\}_{i =1}^N$, and $s_{h+1}^\tau \sim P_h^\star(\cdot \mid s_h^\tau, \boldsymbol{a}_h^\tau)$ for each $h \in [H]$. 
We denote the distribution of $(s_h^\tau, \boldsymbol{a}_h^\tau)$ by $d_{P^{\star},h}^{\boldsymbol{\pi}^{\mathrm{b}}}$ (or simply $d_{P^{\star},h}^{\mathrm{b}}$) for each $\tau\in[n]$ and $h\in[H]$.

We evaluate the performance of offline algorithms by the following notion of \emph{RNE gap}. 
Suppose the learning algorithm outputs some policy $\hat{\boldsymbol{\pi}}$ based on the offline dataset $\mathbb{D}$, then the suboptimality of $\hat{\boldsymbol{\pi}}$ is defined as the violation of the equilibrium condition in Definition~\ref{def: robust nash equilibrium}.
See the following definition.

\begin{definition}[RNE Gap] \label{def:rne:gap} 
The suboptimality gap of ${\boldsymbol{\pi}}$ is defined by\footnote{Without loss of generality, we also assume that the initial state is fixed to some $s_1\in\mathcal{S}$. Our algorithm and theory can be directly extended to the case when $s_1\sim \rho\in\Delta(\mathcal{S})$.}
\begin{align}\label{eq: rne gap}
    \mathrm{RNEGap}_{\boldsymbol{\Phi}}(\boldsymbol{\pi}; s_1) = \max_{i \in [N]} \left\{   V_{1,P^\star,\boldsymbol{\Phi}}^{(\dagger,{\boldsymbol{\pi}}^{-i}),i}(s_1) - V_{1,P^\star,\boldsymbol{\Phi}}^{{\boldsymbol{\pi}},i}(s_1) \right\},\qquad \forall s_1\in\cS.
\end{align}
\end{definition}
According to Definition \ref{def:rne:gap}, the suboptimality of a joint policy $\boldsymbol{\pi}$ is the maximum suboptimality gap across each single player's policy $\pi^i$ against its \emph{robust} best response given of other players.
For an RNE policy $\boldsymbol{\pi}_{\mathrm{RNE}}$, it satisfies $\mathrm{RNEGap}_{\boldsymbol{\Phi}}(\boldsymbol{\pi}_{\mathrm{RNE}}; s_1)=0$.
When $N=1$, the notion of $\mathrm{RNEGap}_{\boldsymbol{\Phi}}$ coincides with that of $\mathrm{SubOpt}_{\boldsymbol{\Phi}}$ for single-agent RMDP \eqref{eq: suboptimality}.
When the robust set mapping satisfies $\boldsymbol{\Phi}(P) = \{P\}$, the condition $\mathrm{RNEGap}_{\boldsymbol{\Phi}}<\epsilon$ conincides with the notion of $\epsilon$-approximate NE \citep{cui2022offline, zhang2023offline} for standard MGs.
In conclusion, the goal of offline RL in RMGs is to learn from $\mathbb{D}$ a policy $\hat{\boldsymbol{\pi}}$ which minimizes the RNE gap.

\subsection{Generic Algorithm Framework and Unified Theory}\label{subsec: algorithm and theory rmg}

In this subsection, we generalize the idea of double pessimism of $\texttt{P}^2\texttt{MPO}$ (Algorithm \ref{alg: p2mpo}) for solving offline RL in RMDPs to solving offline RL in RMGs.
Our result is a new algorithm framework which we call the \underline{D}oubly \underline{P}essimistic \underline{M}odel-based \underline{M}ulti-agent \underline{P}olicy \underline{O}ptimization ($\texttt{P}^2\texttt{M}^2\texttt{PO}$, Algorithm \ref{alg:offline}).
In addition to the principle of double pessimism for value estimators, another optimistic-then-pessimistic value estimator is introduced to the new algorithm to achieve the goal of minimizing the RNE gap \eqref{eq: rne gap}.
We introduce the algorithm framework in Section~\ref{subsubsec: algorithm rmg} and we establish its theoretical analysis in Section \ref{subsubsec: theory rmg}.

\subsubsection{Algrotihm Framework: P$^2$M$^{2}$PO}\label{subsubsec: algorithm rmg}

We now present our proposed algorithm framework $\texttt{P}^2\texttt{M}^2\texttt{PO}$ (Algorithm~\ref{alg:offline}), which consists of a \textcolor{blue}{\emph{model estimation step}} and a \textcolor{blue}{\emph{surrogate objective minimization step}}.

\paragraph{Model estimation step (Line \ref{lin: model estimation rmg}).} 
The model estimation step follows the same routine as the single-agent setting.
Specifically, $\texttt{P}^2\texttt{M}^2\texttt{PO}$ implements a sub-algorithm \texttt{ModelEst}$(\mathbb{D}, \cP_{\mathrm{M}})$ to construct a confidence region $\hat{\cP}$ for the nominal transition kernel $P^{\star}$.
The confidence region $\hat{P}$ is in the form of $\hat{\mathcal{P}} = \{\hat{\cP}_h\}_{h = 1}^H$, with $\hat{\cP}_h \subseteq \cP_{\mathrm{M}}$ for each step $h \in [H]$. 
Similar to the RMDP case, the subroutine \texttt{ModelEst} can be flexibly chosen and should satisfy: (i) $\hat{\cP}$ contains the true model $P^\star$; and (ii) any model in $\hat{\cP}$ does not incur large ``robust model estimation error''. 
We quantify these two conditions in Conditions~\ref{cond: accuracy rmg} and~\ref{cond: model estimation rmg} in the coming theory section, respectively.

\paragraph{Surrogate objective minimization step (Line \ref{lin: surrogate minimization start} to \ref{lin: surrogate minimization end}).} 
In order to minimize the RNE gap \eqref{eq: rne gap}, our method is to construct a surrogate objective of the RNE gap and find the policy minimizing it. 
In specific, for any player $i \in [N]$ and any policy $\boldsymbol{\pi}$, we first define two functions of $\boldsymbol{\pi}$ and $i$ as 
\begin{align}
    J_{\texttt{Pess}^2}^i(\boldsymbol{\pi}) &:=  \inf_{ P_h \in \hat{\cP}_h,1\leq h\leq H} \,\,\inf_{\tilde{P}_h\in\boldsymbol{\Phi}(P_h),1\leq h\leq H}\,\,V_{1}^{\boldsymbol{\pi}, i}(s_1;\{\tilde{P}_h\}_{h=1}^H),\label{eq:underline:J}\\
    J_{\texttt{Opt-Pess}}^i(\boldsymbol{\pi}) &:=  \sup_{ P_h \in \hat{\cP}_h,1\leq h\leq H}\,\,\sup_{\tilde{\pi}^i\in\Delta(\mathcal{A}^i|\mathcal{S},H)} \,\,\inf_{\tilde{P}_h\in\boldsymbol{\Phi}(P_h),1\leq h\leq H}\,\,V_{1}^{(\tilde{\pi}^{i}, \boldsymbol{\pi}^{-i}), i}(s_1;\{\tilde{P}_h\}_{h=1}^H). \label{eq:bar:J} %V_{1}^{\boldsymbol{\pi}, i}(s_1;\{\tilde{P}_h\}_{h=1}^H).\\
\end{align}
Here $J_{\texttt{Pess}^2}^i(\boldsymbol{\pi})$ is the doubly pessimistic estimator for the robust value function $V_{1,P^\star,\boldsymbol{\Phi}}^{{\boldsymbol{\pi}},i}$ in the RNE gap \eqref{eq: rne gap}, which corresponds to the doubly pessimistic estimator \eqref{eq: J pess} for RMDPs.
Besides, since RNE compares each player's policy against its best response, the RNE gap \eqref{eq: rne gap} involves a robust best response term $V_{1,P^\star,\boldsymbol{\Phi}}^{(\dagger,{\boldsymbol{\pi}}^{-i}),i}$, for which we define the function $J_{\texttt{Opt-Pess}}^i(\boldsymbol{\pi})$.
It first performs optimism in the face of data uncertainty (supremum over confidence regions $\hat{\mathcal{P}}_h$) and then performs pessimism in the face of test environment uncertainty (infimum over robust sets $\boldsymbol{\Phi}(P_h)$).
The reason for being optimism in the face of data uncertainty is that to minimize the RNE gap \eqref{eq: rne gap} we actually need to minimize the robust best response value $V_{1,P^\star,\boldsymbol{\Phi}}^{(\dagger,{\boldsymbol{\pi}}^{-i}),i}$ (in contrast to maximizing the robust value function $V_{1,P^\star,\boldsymbol{\Phi}}^{{\boldsymbol{\pi}},i}$ when minimizing \eqref{eq: rne gap}, for which we perform pessimism).
Finally, we define the surrogate objective of the RNE gap \eqref{eq: rne gap} as the difference between $J_{\texttt{Opt-Pess}}^i(\boldsymbol{\pi})$ and $J_{\texttt{Pess}^2}^i(\boldsymbol{\pi})$,
\begin{align} \label{eq:surrogate}
    J_{\texttt{Surrogate}}(\boldsymbol{\pi}) := \max_{i \in [N]} \left\{J_{\texttt{Opt-Pess}}^i(\boldsymbol{\pi}) - J_{\texttt{Pess}^2}^i(\boldsymbol{\pi}) \right\},
\end{align}  
and then $\texttt{P}^2\texttt{M}^{2}$\texttt{PO} outputs a policy $\hat{\boldsymbol{\pi}}$ that minimizes the surrogate objective function in \eqref{eq:surrogate}. 

\begin{remark}
    The idea of minimizing the surrogate objective function also appears in the works on non-robust offline MGs \citep{cui2022provably,zhang2023offline}, but their algorithms are either restricted in the tabular case \citep{cui2022provably} or in a model-free fashion \citep{zhang2023offline}.
\end{remark}

\begin{algorithm}[t]
	\caption{\underline{D}oubly \underline{P}essimistic \underline{M}odel-based \underline{M}ulti-agent \underline{P}olicy \underline{O}ptimization ($\texttt{P}^2\texttt{M}^2\texttt{PO}$)}
	\label{alg:offline}
	\begin{algorithmic}[1]
	\STATE \textbf{Input}: model space $\mathcal{P}_{\mathrm{M}}$, mapping $\mathbf{\Phi}$, dataset $\mathbb{D}$, policy class $\Pi$, algorithm \texttt{ModelEst}.
    \STATE \textcolor{blue}{\texttt{Model estimation step:}}
    \STATE Obtain a confidence region $\hat{\mathcal{P}} = \texttt{ModelEst}(\mathbb{D}, \mathcal{P}_{\mathrm{M}})$.\label{lin: model estimation rmg}
    \STATE \textcolor{blue}{\texttt{Surrogate objective minimization step:}}
    \STATE Calculate $J_{\texttt{Pess}^2}^i(\boldsymbol{\pi})$,  $J_{\texttt{Opt-Pess}}^i(\boldsymbol{\pi})$, and $J_{\texttt{Surrogate}}(\boldsymbol{\pi})$ as \eqref{eq:underline:J}, \eqref{eq:bar:J}, and \eqref{eq:surrogate}.\label{lin: surrogate minimization start}
    \STATE Set policy $\hat{\boldsymbol{\pi}} \leftarrow \argmin_{\boldsymbol{\pi}} J_{\texttt{Surrogate}}(\boldsymbol{\pi})$. \label{lin: surrogate minimization end}
    \STATE \textbf{Output}: $\hat{\boldsymbol{\pi}}=\{\hat{\boldsymbol{\pi}}_h\}_{h=1}^H$.
	\end{algorithmic}
\end{algorithm}

\subsubsection{Unified Theoretical Analysis}\label{subsubsec: theory rmg}

In this subsection, we provide theoretical guarantees for Algorithm~\ref{alg:offline}. 
Before stating our main theorem, we first identify a new \emph{robust unilateral coverage} coefficient for offline RMGs, and then specify two accurate conditions for the model estimation sub-algorithm \texttt{ModelEst}, parallel to Section \ref{subsubsec: condition rmdp}.

As the key role played by coverage conditions in single-agent offline RL, coverage conditions are also critical for RL in MGs.
Parallel to the single-agent RL setting, previous works on multi-agent RL also aim to perform sample-efficient learning under certain minimal coverage conditions.
Recent works \citep{cui2022offline, zhong2022pessimistic} have proposed the \emph{unilateral coverage} assumption for non-robust MGs and show that offline RL in non-robust MGs can be solved in a sample-efficient manner under such an assumption.
For RMGs, we propose the following \emph{robust unilateral coverage} assumption.

\begin{assumption}[Robust unilateral coverage] \label{assumption:unilateral}
    Suppose that $\boldsymbol{\pi}_{\mathrm{RNE}}$ is a robust Nash equilibrium.
    We assume that following robust unilateral coverage coefficient is finite,
    \$
     \boldsymbol{C}^{\mathrm{RNE}}_{P^{\star},\mathbf{\Phi}} = \sup_{h\in[H], i \in [N]} \,\,\sup_{\pi^i\in\Delta(\mathcal{A}^i|\mathcal{S},H)} \,\,\sup_{P = \{P_h\}_{h=1}^H,P_h \in \mathbf{\Phi}(P_h^{\star})}\,\,\mathbb{E}_{(s_h,\boldsymbol{a}_h)\sim d^{\boldsymbol{\pi}^{\mathrm{b}}}_{P^{\star}, h}}\left[\left(\frac{d^{(\pi^i, \boldsymbol{\pi}_{\mathrm{RNE}}^{-i})}_{P, h}(s_h,\boldsymbol{a}_h)}{d^{\boldsymbol{\pi}^{\mathrm{b}}}_{P^{\star}, h}(s_h,\boldsymbol{a}_h)}\right)^2\right] < \infty.
    \$ 
    where $\boldsymbol{\pi}_{\mathrm{RNE}}$ is one of the robust Nash equilibrium (Definition \ref{def: robust nash equilibrium}).
\end{assumption}

Assumption~\ref{assumption:unilateral} requires that the dataset distribution has good coverage of trajectories induced by 
\begin{align*}
    \big\{\boldsymbol{\pi} = (\pi_i, (\boldsymbol{\pi}_{\mathrm{RNE}})^{-i}): \pi_i\in \Delta(\mathcal{A}^i|\mathcal{S},H), i \in [N]\big\}
\end{align*}
and any transition kernel $P$ in the robust set of the nominal transition kernel $\boldsymbol{\Phi}(P^\star)$. 
For degenerate non-robust MGs, i.e., $\boldsymbol{\Phi}(P) = \{P\}$, the robust unilateral coverage coefficient defined in Assumption \ref{assumption:unilateral} is consistent with the \emph{unilateral coverage coefficient} adopted by a line of previous works on offline non-robust MGs \citep{zhong2022pessimistic,cui2022provably,cui2022offline,xiong2022nearly,yan2022model,zhang2023offline}, and thus giving the name of \emph{robust unilateral coverage coefficient}.
%In our theory, we assume that the offline data $\mathbb{D}$ satisfies such a coverage assumption.

\paragraph*{Conditions on model estimation.} 
Now we specify the two accurate conditions of model estimation. Recall that $\hat{\cP} = \texttt{ModelEst}(\mathbb{D}, \cP_{\mathrm{M}})$ where $\hat{\cP}_h \subseteq \cP_{\mathrm{M}}$ for all $h \in [H]$. The first condition ensures that confidence region $\hat{\cP}$ contains the nominate model $P^\star$ with high probability,

\begin{cond}[$\delta$-accuracy]\label{cond: accuracy rmg}
    With probability at least $1-\delta$, it holds that $P^{\star}_h\in\hat{\mathcal{P}}_h$ for any $h\in[H]$.
\end{cond}

Besides Condition~\ref{cond: accuracy rmg}, the desired confidence region $\cP$ should satisfy that any transition kernel in it incurs a small ``robust estimation error". 
To be specific, we define the following \emph{robust Bellman error} with respect to some transition $P_h$ and value function $V: \cS \mapsto [0, H]$,
\begin{align} \label{eq:robust:bellman:error}
    \boldsymbol{\cE}_h^{\boldsymbol{\Phi}}(s,\boldsymbol{a};P_h, V) = \inf_{\tilde{P}_h\in\mathbf{\Phi}(P_h)}\mathbb{E}_{s'\sim \tilde{P}_h(\cdot| s,\boldsymbol{a})}[V(s')] - \inf_{\tilde{P}_h\in\mathbf{\Phi}(P^{\star}_h)}\mathbb{E}_{s'\sim \tilde{P}_h(\cdot| s,\boldsymbol{a})}[V(s')].
\end{align}

\begin{cond}[$\delta$-model estimation error]\label{cond: model estimation rmg}
    For some function of the sample size $n$ and failure probability $\delta$ denoted by $\mathbf{Err}_{h}^{\mathbf{\Phi}}(n,\delta)<+\infty$, with probability at least $1-\delta$, it holds that
    \begin{align}
       \max_{i\in\in[N]}\mathbb{E}_{(s,\boldsymbol{a})\sim d^{\boldsymbol{\pi}^{\mathrm{b}}}_{P^{\star},h}} \left[\left( \boldsymbol{\cE}_h^{\boldsymbol{\Phi}}(s,\boldsymbol{a};P_h, V_{h+1, P, \mathbf{\Phi}}^{( \pi^i, \boldsymbol{\pi}_{\mathrm{RNE}}^{-i} ), i})\right)^2 \right] \leq \mathbf{Err}_{h}^{\mathbf{\Phi}}(n,\delta).
    \end{align}
    for any policy $\pi^i\in\Delta(\cA^i|\cS,H)$, transition kernel $P = \{P_h\}_{h=1}^H$ with $P_h\in\hat{\mathcal{P}}_h$ for each step $h\in[H]$.
\end{cond}

To interpret, Condition~\ref{cond: model estimation} requires that the robust Bellman error \eqref{eq:robust:bellman:error} incurred by any $P_h \in \hat{\cP}_h$ is upper bounded by $\mathbf{Err}_{h}^{\mathbf{\Phi}}(n,\delta)$. 
As in the case of single-agent RMDPs, the error $\mathbf{Err}_{h}^{\mathbf{\Phi}}(n,\delta)$ generally diminishes at the rate of $\tilde{\mathcal{O}}(n^{-1})$, where $n$ is the size of dataset $\mathbb{D}$. 

Now we present our main result in the following theorem, which characterizes the RNE gap of Algorithm~\ref{alg:offline}.

\begin{theorem}[Suboptimality of $\texttt{P}^2\texttt{M}^2\texttt{PO}$]\label{thm:offline:FA}
     Suppose that Assumptions~\ref{ass: s a rectangular} and \ref{assumption:unilateral} hold, if the  model estimation sub-algorithm satisfies Conditions~\ref{cond: accuracy rmg} and~\ref{cond: model estimation rmg}, it holds with probability $1 - 2\delta$ that
     \begin{align*}
        \mathrm{RNEGap}_{\boldsymbol{\Phi}}(\hat{\boldsymbol{\pi}}; s_1) \le 2 \sqrt{\boldsymbol{C}^{\mathrm{RNE}}_{P^{\star},\mathbf{\Phi}} } \cdot \sum_{h = 1}^H \sqrt{\mathbf{Err}_h^{\mathbf{\Phi}}(n,\delta)}.
     \end{align*} 
\end{theorem}

\begin{proof}[Proof of Theorem \ref{thm:offline:FA}]
      See Appendix \ref{appendix:pf:offline:FA} for a detailed proof.
\end{proof}

As we did in Section \ref{sec: implementation} for RMDPs, we can use similar analysis to specify Theorem \ref{thm:offline:FA} to specific examples of RMGs, and can be coped with kernel and neural function approximations.
To illustrate, we only present a specification result for RMGs with finite state spaces. 
More corollaries can be derived without much difficulty given the techniques we presented in Section \ref{sec: implementation}.

\begin{corollary}[Suboptimality of $\texttt{P}^2\texttt{M}^2\texttt{PO}$: $\cS\times\cA$-rectangular robust tabular MG]
    Consider an RMG satisfying Assumption \ref{ass: s a rectangular} with a finite state space $\cS$.
    Moreover, its robust set mapping $\boldsymbol{\Phi}$ satisfy that 
    \begin{align}\label{eq: distribution ball rmg}
        \mathbf{\Phi}(P) = \bigotimes_{(s,\boldsymbol{a})\in\mathcal{S}\times\mathcal{A}} \mathcal{P}(s,\boldsymbol{a}; P),\quad \text{where}\quad \mathcal{P}_{\rho}(s,\boldsymbol{a}; P) = \left\{\tilde{P}(\cdot)\in\Delta(\mathcal{S}):D(\tilde{P}(\cdot)\|P(\cdot|s,\boldsymbol{a}))\leq \rho\right\},
    \end{align}
    where $D(\cdot\|\cdot)$ is either KL-divergence or TV-distance.
    Then by choosing the confidence region $\hat{\cP} = \{\hat{\cP}_h\}_{h=1}^H$ as  
    \begin{align}
        \hat{\mathcal{P}}_h = \bigg\{P\in\mathcal{P}_{\mathrm{M}}: \frac{1}{n}\sum_{\tau=1}^n \|\hat{P}_h(\cdot|s_h^\tau,\boldsymbol{a}_h^\tau) - P(\cdot|s_h^\tau,\boldsymbol{a}_h^\tau)\|_1^2 \leq \xi\bigg\},\quad \hat{P}_h = \argmax_{P \in \mathcal{P}_{\mathrm{M}}}\frac{1}{n}\sum_{\tau=1}^n\log P(s_{h+1}^\tau|s_h^\tau,\boldsymbol{a}_h^\tau),
    \end{align}
    with $\xi = C_1|\cS|^2|\cA|\log(C_2nH/\delta)/n$, the $\mathtt{P}^2\mathtt{M}^2\mathtt{PO}$ algorithm enjoys following results under Assumption~\ref{assumption:unilateral},
    \begin{itemize}
        \item [$\spadesuit$] when $D(\cdot\|\cdot)$ is KL-divergence and Assumption \ref{ass: kl regularity} holds with parameter $\underline{\lambda}$ (treating $\cA$ as the joint action space of the RMG), then with probability at least $1-2\delta$, 
        \begin{align*}
            \mathrm{RNEGap}_{\boldsymbol{\Phi}}(\hat{\boldsymbol{\pi}}; s_1)\leq \frac{\sqrt{\boldsymbol{C}^{\mathrm{RNE}}_{P^{\star},\mathbf{\Phi}} } \cdot H^2\exp(H/\underline{\lambda})}{\rho}\cdot\sqrt{\frac{C_1'|\cS|^2|\cA|\log(C_2'nH/\delta)}{n}}.
        \end{align*}
        \item [$\clubsuit$] when $D(\cdot\|\cdot)$ is TV-divergence, then with probability at least $1-2\delta$, 
        \begin{align*}
            \mathrm{RNEGap}_{\boldsymbol{\Phi}}(\hat{\boldsymbol{\pi}}; s_1)\leq \sqrt{\boldsymbol{C}^{\mathrm{RNE}}_{P^{\star},\mathbf{\Phi}} } \cdot H^2\cdot\sqrt{\frac{C_1'|\cS|^2|\cA|\log(C_2'nH/\delta)}{n}}.
        \end{align*}
    \end{itemize}
    Here $C_1$, $C_2$, $C_1'$, $C_2'>0$ stand for universal constants.
\end{corollary}

%% file: tex/discussion.tex
\section{Discussions}\label{sec: discussion}

In this section, we discuss and analysis some other types of RMDPs appearing in existing works that do not satisfy Assumption \ref{ass: sarmdp} ($\mathcal{S}\times\mathcal{A}$-rectangular), including $d$-rectangular robust linear MDPs \citep{ma2022distributionally} and RMDPs with $\mathcal{S}$-rectangular robust sets \citep{wiesemann2013robust}, see Section \ref{subsec: d rectangular robust linear MDP} and \ref{subsec: S rectangular robust MDP} respectively.

\subsection{$d$-rectangular robust linear MDPs}\label{subsec: d rectangular robust linear MDP}

Recently \cite{ma2022distributionally} proposed the $d$-rectangular robust linear MDP to study offline robust RL with linear structures.
We use the following example to show how a $d$-rectangular robust linear MDP is represented by our general framework of RMDP.

\begin{example}[$d$-rectangular robust linear MDP \citep{ma2022distributionally}]\label{exp: dlmdp}
    A $d$-rectangular robust linear MDP is equipped with $d$-rectangular robust sets.
    Linear MDP is an MDP that enjoys a $d$-dimensional linear decomposition of its reward function and transition kernel \citep{jin2020provably}.
    We define the model space $\mathcal{P}_{\mathrm{M}}$ as
    \begin{align*}
        \mathcal{P}_{\mathrm{M}} = \Big\{P(s'|s,a) = \boldsymbol{\phi}(s,a)^\top\boldsymbol{\mu}(s'):\mu_i(\cdot) \in \Delta(\mathcal{S}),\forall i\in[d]\Big\},
    \end{align*}
    where $\boldsymbol{\phi}:\mathcal{S}\times\mathcal{A}\mapsto\mathbb{R}^d$ is a known feature mapping satisfying that 
    \begin{align*}
        \sum_{i=1}^d\phi_i(s,a) = 1, \quad \phi_i(s,a)\geq 0,\quad \forall i\in[d].
    \end{align*}
    We then assume that $P^{\star}_h(s'|s,a)=\boldsymbol{\phi}(s,a)^\top\boldsymbol{\mu}^{\star}_h(s')\in\mathcal{P}_{\mathrm{M}}$, and $R_h(s,a)  = \boldsymbol{\phi}(s,a)^\top\boldsymbol{\theta}_h$ for some $\boldsymbol{\theta}_h\in\mathbb{R}^d$ with $\|\boldsymbol{\theta}_h\|_2\leq \sqrt{d}$. 
    We define the mapping $\boldsymbol{\Phi}$ as  
    \begin{align*}
        \boldsymbol{\Phi}(P) = \left\{\sum_{i=1}^d\phi_i(s,a)\tilde{\mu}_i(s'):\tilde{\mu}_i(\cdot)\in\Delta(\mathcal{S}),D(\tilde{\mu}(\cdot)\|\mu_i(\cdot))\leq \rho,\forall i\in[d]\right\},
    \end{align*}
    where $D(\cdot\|\cdot)$ is some (pseudo-)distance such as KL-divergence or TV-distance. 
    This is called a $d$-rectangular robust set and is first considered by \cite{ma2022distributionally}.
    As is argued in \cite{ma2022distributionally}, $d$-rectangular robust set is not so conservative as $\mathcal{S}\times\mathcal{A}$-rectangular robust set in certain cases, which is more natural for linear MDPs due to the special linear structure.
\end{example}

While not satisfying Assumption \ref{ass: sarmdp} ($\mathcal{S}\times\mathcal{A}$-rectangular robust sets), it can still be proved that RMDP in Example \ref{exp: dlmdp} also satisfies the robust Bellman equation in Proposition \ref{prop: robust Bellman} (similar to the proof in Appendix~\ref{subsec: robust bellman sarmdp} for $\mathcal{S}\times\mathcal{A}$-rectangular robust MDPs).
Our algorithm $\texttt{P}^2\texttt{MPO}$ (Algorithm \ref{alg: p2mpo}) can also be applied to offline solve robust RL with RMDP in Example \ref{exp: dlmdp}, under certain robust partial coverage assumption (see Assumption~\ref{ass: partial coverage cov}).  In the following, we give a specific implementation of the model estimation step for $d$-rectangular RMDPs in Example \ref{exp: dlmdp}, and we provide theoretical guarantees for this specification of our algorithm $\texttt{P}^2\texttt{MPO}$.

\paragraph{Model estimation.}
% We consider the following design of confidence region $\hat{\mathcal{P}}$.
Suppose we are given a function class $\mathcal{V} \subseteq \{v:\mathcal{S}\mapsto\mathbb{R}\}$ 
%\han{$\cS \mapsto [0, 1]$ implies $\|v\|_{\infty} \le 1$?} 
which depends on the choice of distance $D(\cdot\|\cdot)$ of the robust set. 
Then, we define that
\begin{align}\label{eq: confidence region drlmdp}
    \hat{\mathcal{P}}_h = \left\{P\in\mathcal{P}_{\mathrm{M}}: \sup_{v\in\mathcal{V}} \frac{1}{n}\sum_{\tau=1}^n \left|\int_{\mathcal{S}} P(\mathrm{d}s'|s_h^\tau,a_h^\tau)v(s') - \boldsymbol{\phi}(s_h^\tau,a_h^\tau)^\top \hat{\boldsymbol{\theta}}_v\right|^2\leq \xi \right\},
\end{align}
where $\xi>0$ is a tuning parameter that controls the size of the confidence region, and the vector $\hat{\boldsymbol{\theta}}_{h,v}$ depends on the specific function $v\in\mathcal{V}$, given by 
\begin{align}
    \hat{\boldsymbol{\theta}}_{h, v} = \argmin_{\boldsymbol{\theta}\in\mathbb{R}^d} \frac{1}{n}\sum_{\tau=1}^n \left(\boldsymbol{\phi}(s_h^\tau,a_h^\tau)^\top\boldsymbol{\theta} - v(s_{h+1}^\tau)\right)^2 + \frac{\alpha}{n}\cdot \|\boldsymbol{\theta}\|_2^2 =\boldsymbol{\Lambda}_{h,\alpha}^{-1}\left(\frac{1}{n}\sum_{\tau=1}^n\boldsymbol{\phi}(s_h^{\tau},a_h^{\tau})v(s_{h+1}^{\tau})\right), \label{eq: hat theta v}
\end{align}
for some tuning parameter $\alpha>0$, where $\boldsymbol{\Lambda}_{h,\alpha}$ is the regularized covariance matrix, defined as
\begin{align*}
    \boldsymbol{\Lambda}_{h,\alpha} = \frac{1}{n}\sum_{\tau=1}^n\boldsymbol{\phi}(s_h^{\tau},a_h^{\tau})\boldsymbol{\phi}(s_h^{\tau},a_h^{\tau})^\top + \frac{\alpha}{n}\cdot\boldsymbol{I}_d.
\end{align*}
Similar constructions for standard linear MDPs are also considered by \cite{sun2019model, neu2020unifying, uehara2021pessimistic}.
We will specify the choice of the function class $\mathcal{V}$ in the theoretical guarantees of this implementation.

\paragraph*{Suboptimality analysis.}
In the following, we provide suboptimality bounds for the above implementation of $\texttt{P}^2\texttt{MPO}$ for $d$-rectangular robust linear MDPs.
Regarding the offline dataset, we impose the following robust partial coverage assumption.

\begin{assumption}[Robust partial coverage covariance matrix]\label{ass: partial coverage cov}
    We assume that for some constant $c^{\dagger}>0$,
    \begin{align}
        \boldsymbol{\Lambda}_{h,\alpha} \succeq \frac{\alpha}{n}\cdot\boldsymbol{I}_d + c^{\dagger}\cdot\mathbb{E}_{(s_h,a_h)\sim d_{P,h}^{\pi^{\star}}}[(\phi_i(s_h,a_h)\mathbf{1}_i)(\phi_i(s_h,a_h)\mathbf{1}_i)^\top]
    \end{align}
    for any $i\in[d]$, $h\in[H]$, and $P_h\in\boldsymbol{\Phi}(P_h^{\star})$.
\end{assumption}

\begin{theorem}[Suboptimality of $\texttt{P}^2\texttt{MPO}$: $d$-rectangular robust linear MDP]\label{cor: suboptimality drlmdp}
    Suppose that the RMDP is $d$-rectangular robust linear MDP in Example \ref{exp: dlmdp} with $D(\cdot\|\cdot)$ being KL-divergence or TV-distance and that Assumption \ref{ass: partial coverage cov} holds, choosing the tuning parameter $\alpha=1$ in \eqref{eq: hat theta v}.
    \begin{itemize}
        \item [$\spadesuit$] when $D(\cdot\|\cdot)$ is KL-divergence and Assumption \ref{ass: kl regularity drlmdp} holds with parameter $\underline{\lambda}$, then by setting 
        \begin{align*}
            \mathcal{V} = \left\{v(s) = \exp\left\{-\left\{\max_{a\in\mathcal{A}}\boldsymbol{\phi}(s,a)^\top\boldsymbol{w}/\lambda\right\}_+\right\}:\|\boldsymbol{w}\|_2\leq H\sqrt{d}, \lambda\in [\underline{\lambda},H/\rho]\right\},
        \end{align*}
        and choosing
        \begin{align*}
            \xi = \frac{C_1d^2\big(\log(1+C_2nH/\delta) + \log(1+C_3ndH/(\rho\underline{\lambda}^2))\big)}{n},
        \end{align*}
        for some constants $C_1,C_2,C_3>0$, it holds with probability at least $1-2\delta$ that, 
        \begin{align*}
            \mathrm{SubOpt}(\hat{\pi};s_1) \leq \frac{d^2H^2\exp(H/\underline{\lambda})}{c^{\dagger}\rho}\cdot\sqrt{\frac{C_1'\big(\log(1+C_2'nH/\delta) + \log(1+C_3'ndH/(\rho\underline{\lambda}^2))\big)}{n}}.
        \end{align*}
        \item [$\clubsuit$] when $D(\cdot\|\cdot)$ is TV-distance, then by setting  
        \begin{align*}
            \!\!\mathcal{V} = \left\{v(s) = \left\{\lambda - \max_{a\in\mathcal{A}}\boldsymbol{\phi}(s,a)^\top\boldsymbol{w}\right\}_+\!\!\!\!:\!\|\boldsymbol{w}\|_2\leq H\sqrt{d}, \lambda\in[0,H]\right\},\quad \xi = \frac{C_1d^2H^2\log(C_2ndH/\delta)}{n},
        \end{align*}
        % and choosing
        % \begin{align*}
        % \end{align*}
        for some constants $C_1,C_2>0$, it holds with probability at least $1-2\delta$ that, 
        \begin{align*}
           \mathrm{SubOpt}(\hat{\pi};s_1) \leq \frac{d^2H^2}{c^{\dagger}}\cdot \sqrt{\frac{C_1'\log(C_2' ndH/\delta)}{n}}.
        \end{align*}
    \end{itemize}
    Here $\underline{c}$ is defined in Assumption \ref{ass: partial coverage cov} and $C_1',C_2',C_3'>0$ are universal constants.
\end{theorem}

\begin{proof}[Proof of Theorem \ref{cor: suboptimality drlmdp}]
    See Appendix \ref{sec: proof drlmdp} for a detailed proof.
\end{proof}

\subsection{RMDPs with $\mathcal{S}$-rectangular robust sets} \label{subsec: S rectangular robust MDP}
Besides $\mathcal{S}\times\mathcal{A}$-rectangular, there exists another type of generic rectangular assumption on robust sets called $\mathcal{S}$-rectangular \citep{wiesemann2013robust,yang2019fine}.
See the following assumption.

\begin{assumption}[$\mathcal{S}$-rectangular robust sets \citep{wiesemann2013robust}]\label{ass: srmdp}
    An $\mathcal{S}$-rectangular robust MDP is equipped with $\mathcal{S}$-rectangular robust sets.
    The mapping $\mathbf{\Phi}$ is defined as, for $\forall P\in\mathcal{P}_{\mathrm{M}}$,
    \begin{align*}
        \mathbf{\Phi}(P) = \bigotimes_{s\in\mathcal{S}} \mathcal{P}_{\rho}(s; P),\quad \mathcal{P}_{\rho}(s; P) = \left\{\tilde{P}(\cdot|\cdot):\mathcal{A}\mapsto\Delta(\mathcal{S}):\sum_{a\in\mathcal{A}}D(\tilde{P}(\cdot|a)\|P(\cdot|s,a))\leq \rho|\mathcal{A}|\right\},
    \end{align*}
    for some (pseudo-)distance $D(\cdot\|\cdot)$ on $\Delta(\mathcal{S})$ and some real number $\rho\in\mathbb{R}_+$.
\end{assumption}

RMDP with $\mathcal{S}$-rectangular robust sets (Assumption \ref{ass: srmdp}) also satisfies Proposition \ref{prop: robust Bellman} \citep{wiesemann2013robust}.
Unfortunately, our algorithm framework is unable to deal with this kind of rectangular robust sets under partial coverage data due to some technical problems in applying the robust partial coverage coefficient $C_{P^{\star},\boldsymbol{\Phi}}^{\star}$ (Assumption \ref{ass: partial coverage}) under this kind of robust sets.
To our best knowledge, how to design sample-efficient algorithms for $\mathcal{S}$-rectangular RMDP with robust partial coverage data is still unknown.
It is an exciting future work to fill this gap for robust offline reinforcement learning with function approximations.

%% file: tex/appendix/robust_bellman.tex
\section{Proof of Robust Bellman Equation}\label{sub: robust bellman sarmdp}

\subsection{Proof of Proposition \ref{prop: robust Bellman}}\label{subsec: robust bellman sarmdp}

\begin{proof}[Proof of Proposition \ref{prop: robust Bellman} for $\mathcal{S}\times\mathcal{A}$-rectangular robust MDP]
    Instead of directly proving the robust Bellman equation \eqref{eq: robust bellman V}, we prove the following stronger results via induction from step $h=H$ to $1$: \emph{there exists a set of transition kernels $P^{\pi,\dagger}=\{P^{\pi,\dagger}_h\}_{h=1}^H$ with $P^{\pi,\dagger}_h\in\boldsymbol{\Phi}(P_h)$ such that
    \begin{enumerate}
        \item Robust Bellman equation holds, i.e., 
        \begin{align*}
            V_{h, P, \mathbf{\Phi}}^{\pi}(s) &= \mathbb{E}_{a\sim \pi_h(\cdot|s)}[Q_{h, P, \mathbf{\Phi}}^{\pi}(s, a)],\\
            Q_{h, P, \mathbf{\Phi}}^{\pi}(s, a) &= R_h(s,a) + \inf_{\tilde{P}_h\in\mathbf{\Phi}(P_h)}\mathbb{E}_{s'\sim \tilde{P}_h(\cdot|s,a)}[V_{h+1, P, \mathbf{\Phi}}^{\pi}(s')].
        \end{align*}
        \item The following expressions for robust value functions hold,
        \begin{align*}
            V_{h, P, \mathbf{\Phi}}^{\pi}(s) & = V_h^{\pi}(s;\{P^{\pi,\dagger}_i\}_{i=h}^H), \\
            Q_{h, P, \mathbf{\Phi}}^{\pi}(s, a) & = Q_h^{\pi}(s, a;\{P^{\pi,\dagger}_i\}_{i=h}^H).
        \end{align*}
        \end{enumerate}
    }
    
    Firstly, for step $h=H$, the conclusion 1. and 2. hold directly because no transitions are involved.
    Now supposing that the conclusion 1. and 2. hold for some step $h+1$, which means that there exist transition kernels $\{P^{\pi,\dagger}_i\}_{i=h+1}^H$ such that the following condition hold for any $s\in\mathcal{S}$,
    \begin{align}\label{eq: proof robust bellman sarmdp 0}
        V_{h+1, P, \mathbf{\Phi}}^{\pi}(s) & = V_{h+1}^{\pi}(s;\{P^{\pi,\dagger}_i\}_{i=h+1}^H).
    \end{align}
    By the definition of robust value function $Q_{h, P, \mathbf{\Phi}}^{\pi}$ in \eqref{eq: robust Q}, we can derive that for any $(s,a)\in\mathcal{S}\times\mathcal{A}$, 
    \begin{align}
        Q_{h, P, \mathbf{\Phi}}^{\pi}(s,a) &= \inf_{\tilde{P}_i\in\mathbf{\Phi}(P_i), h\leq i\leq H} \mathbb{E}_{\{\tilde{P}_i\}_{i=h}^H,\pi}\left[\sum_{i=h}^HR_i(s_i,a_i) \middle| s_h=s,a_h=a\right]\notag\\
            &= R_h(s,a) + \inf_{\tilde{P}_i\in\mathbf{\Phi}(P_i), h\leq i\leq H} \int_{\mathcal{S}}\tilde{P}_h(\mathrm{d}s'|s,a)\mathbb{E}_{\{\tilde{P}_i\}_{i=h+1}^H,\pi}\left[\sum_{i=h+1}^HR_i(s_i,a_i) \middle| s_{h+1}=s'\right]\notag\\
            &\leq R_h(s,a) + \inf_{\tilde{P}_h\in\mathbf{\Phi}(P_h)} \int_{\mathcal{S}}\tilde{P}_h(\mathrm{d}s'|s,a)\mathbb{E}_{\{P_i^{\pi,\dagger}\}_{i=h+1}^H,\pi}\left[\sum_{i=h+1}^HR_i(s_i,a_i) \middle| s_{h+1}=s'\right].\label{eq: proof robust bellman sarmdp 1}
    \end{align}
    On the one hand, for $\mathcal{S}\times\mathcal{A}$-rectangular robust MDP, the robust set $\mathbf{\Phi}(P_h)$ is decoupled for different $(s,a)$ pairs, i.e., 
    \begin{align*}
        \mathbf{\Phi}(P_h) = \bigotimes_{(s,a)\in\mathcal{S}\times\mathcal{A}} \mathcal{P}_{\rho}(s,a; P_h),
    \end{align*}
    and therefore we can find a \emph{single} transition kernel $P_h^{\pi,\dagger}$ such that for \emph{any} $(s,a)\in\mathcal{S}\times\mathcal{A}$,
    \begin{align}\label{eq: proof robust bellman sarmdp 2}
        P_h^{\pi,\dagger}(\cdot|s,a) = \arginf_{\tilde{P}_h\in\mathbf{\Phi}(P_h)}\int_{\mathcal{S}}\tilde{P}(\mathrm{d}s'|s,a)\mathbb{E}_{\{P_i^{\pi,\dagger}\}_{i=h+1}^H,\pi}\left[\sum_{i=h+1}^HR_i(s_i,a_i) \middle| s_{h+1}=s'\right].
    \end{align}
    On the other hand, using condition \eqref{eq: proof robust bellman sarmdp 0} and the definition of (robust) value function $V_{h, P, \mathbf{\Phi}}^{\pi}$ and $V_h^{\pi}$ in \eqref{eq: robust V} and \eqref{eq: V}, we can also deduce that,
    \begin{align}
        Q_{h, P, \mathbf{\Phi}}^{\pi}(s,a) &\leq R_h(s,a) + \inf_{\tilde{P}_h\in\mathbf{\Phi}(P_h)} \int_{\mathcal{S}}\tilde{P}_h(\mathrm{d}s'|s,a)V_{h+1}^{\pi}(s';\{P^{\pi,\dagger}_i\}_{i=h+1}^H)\notag\\
        &=R_h(s,a) + \inf_{\tilde{P}_h\in\mathbf{\Phi}(P_h)} \int_{\mathcal{S}}\tilde{P}_h(\mathrm{d}s'|s,a)V_{h+1, P, \mathbf{\Phi}}^{\pi}(s')\label{eq: proof robust bellman sarmdp 3}\\
        &=R_h(s,a) + \inf_{\tilde{P}_h\in\mathbf{\Phi}(P_h)} \int_{\mathcal{S}}\tilde{P}_h(\mathrm{d}s'|s,a)\inf_{\tilde{P}_i\in\mathbf{\Phi}(P_i), h+1\leq i\leq H}V_{h+1}^{\pi}(s';\{\tilde{P}_i\}_{i=h+1}^H)\notag\\
        &\leq R_h(s,a) + \inf_{\tilde{P}_i\in\mathbf{\Phi}(P_i), h\leq i\leq H}\int_{\mathcal{S}}\tilde{P}_h(\mathrm{d}s'|s,a)V_{h+1}^{\pi}(s';\{\tilde{P}_i\}_{i=h+1}^H),\label{eq: proof robust bellman sarmdp 4}
    \end{align}
    where the first inequality follows from inequality \eqref{eq: proof robust bellman sarmdp 1} and the definition of $V_{h+1}^{\pi}$ in \eqref{eq: V}, the first equality follows from condition \eqref{eq: proof robust bellman sarmdp 0}, and the second equality follows from the definition of $V_{h+1, P, \mathbf{\Phi}}^{\pi}$ in \eqref{eq: robust V}.
    Note that the right hand side of \eqref{eq: proof robust bellman sarmdp 4} equals to $Q_{h, P, \mathbf{\Phi}}^{\pi}(s,a)$.
    Therefore, all the inequalities are actually equalities.
    On the one hand, from \eqref{eq: proof robust bellman sarmdp 3}, we can know that,
    \begin{align*}
        Q_{h, P, \mathbf{\Phi}}^{\pi}(s,a) = R_h(s,a) + \inf_{\tilde{P}_h\in\mathbf{\Phi}(P_h)} \int_{\mathcal{S}}\tilde{P}_h(\mathrm{d}s'|s,a)V_{h+1, P, \mathbf{\Phi}}^{\pi}(s').
    \end{align*}
    This proves the $Q_{h, P, \mathbf{\Phi}}^{\pi}$ part of the conclusion 1. for step $h$.
    On the other hand, by combining \eqref{eq: proof robust bellman sarmdp 2} and \eqref{eq: proof robust bellman sarmdp 1}, one can further obtain that,
    \begin{align}\label{eq: proof robust bellman sarmdp 5}
        Q_{h, P, \mathbf{\Phi}}^{\pi}(s,a) = \mathbb{E}_{\{P_i^{\pi,\dagger}\}_{i=h}^H,\pi}\left[\sum_{i=h}^HR_i(s_i,a_i) \middle| s_{h}=s,a_h=a\right] = Q_h^{\pi}(s, a;\{P^{\pi,\dagger}_i\}_{i=h}^H).
    \end{align}
    This proves the existence of $\{P^{\pi,\dagger}_i\}_{i=h}^H$ in the
    conclusion 2. for step $h$ and $Q_{h, P, \mathbf{\Phi}}^{\pi}$.
    The remaining of the proof is to prove the $V_{h, P, \mathbf{\Phi}}^{\pi}$ part of the conclusion 1. and 2. for step $h$ using $\{P^{\pi,\dagger}_i\}_{i=h}^H$ found in the previous proof.
    Specifically, by the definition of $V_{h, P, \mathbf{\Phi}}^{\pi}$ in \eqref{eq: robust V}, we have that,
    \begin{align}
        V_{h, P, \mathbf{\Phi}}^{\pi}(s) &= \inf_{\tilde{P}_i\in\mathbf{\Phi}(P_i), h\leq i\leq H}\mathbb{E}_{\{\tilde{P}_i\}_{i=h}^H,\pi}\left[\sum_{i=h}^HR_i(s_i,a_i) \middle| s_h=s\right]\notag\\
        &=\inf_{\tilde{P}_i\in\mathbf{\Phi}(P_i), h\leq i\leq H}\sum_{a\in\mathcal{A}}\pi_h(a|s)\mathbb{E}_{\{\tilde{P}_i\}_{i=h}^H,\pi}\left[\sum_{i=h}^HR_i(s_i,a_i) \middle| s_h=s,a_h=a\right]\notag\\
        &\leq \sum_{a\in\mathcal{A}}\pi_h(a|s)\mathbb{E}_{\{P^{\pi,\dagger}_i\}_{i=h}^H,\pi}\left[\sum_{i=h}^HR_i(s_i,a_i) \middle| s_h=s,a_h=a\right].\label{eq: proof robust bellman sarmdp 6}
    \end{align}
    Now applying \eqref{eq: proof robust bellman sarmdp 5} to \eqref{eq: proof robust bellman sarmdp 6}, we can further obtain that
    \begin{align}
        V_{h, P, \mathbf{\Phi}}^{\pi}(s) & \leq \sum_{a\in\mathcal{A}}\pi_h(a|s)Q_{h, P, \mathbf{\Phi}}^{\pi}(s,a) \label{eq: proof robust bellman sarmdp 7}\\
        & = \sum_{a\in\mathcal{A}}\pi_h(a|s)\inf_{\tilde{P}_i\in\mathbf{\Phi}(P_i), h\leq i\leq H} \mathbb{E}_{\{\tilde{P}_i\}_{i=h}^H,\pi}\left[\sum_{i=h}^HR_i(s_i,a_i) \middle| s_h=s,a_h=a\right]\notag\\ 
        & \leq \inf_{\tilde{P}_i\in\mathbf{\Phi}(P_i), h\leq i\leq H} \sum_{a\in\mathcal{A}}\pi_h(a|s)\mathbb{E}_{\{\tilde{P}_i\}_{i=h}^H,\pi}\left[\sum_{i=h}^HR_i(s_i,a_i) \middle| s_h=s,a_h=a\right],\label{eq: proof robust bellman sarmdp 8}
    \end{align}
    where the equality follows from the definition of $Q_{h, P, \mathbf{\Phi}}^{\pi}$ in \eqref{eq: robust Q}. 
    Now note that the right hand side of \eqref{eq: proof robust bellman sarmdp 8} equals to $V^{\pi}_{h, P, \mathbf{\Phi}}$.
    Therefore, all the inequalities are actually equalities.
    On the one hand, by \eqref{eq: proof robust bellman sarmdp 7}, we know that,
    \begin{align}\label{eq: proof robust bellman sarmdp 9}
        V_{h, P, \mathbf{\Phi}}^{\pi}(s) = \sum_{a\in\mathcal{A}}\pi_h(a|s)Q_{h, P, \mathbf{\Phi}}^{\pi}(s,a).
    \end{align}
    This proves the $V_{h, P, \mathbf{\Phi}}^{\pi}$ part of the conclusion 1. for step $h$.
    On the other hand, by combining \eqref{eq: proof robust bellman sarmdp 9} with \eqref{eq: proof robust bellman sarmdp 5}, we can further deduce that,
    \begin{align*}
        V_{h, P, \mathbf{\Phi}}^{\pi}(s) = \mathbb{E}_{\{P_i^{\pi,\dagger}\}_{i=h}^H,\pi}\left[\sum_{i=h}^HR_i(s_i,a_i) \middle| s_{h}=s\right].
    \end{align*}
    This proves the $V_{h, P, \mathbf{\Phi}}^{\pi}$ part of the conclusion 2. for step $h$.
    Finally, by using an induction argument, we can finish the proof of the conclusion 1. and 2.
    
    Now according to the conclusion 1., we have that 
    \begin{align}\label{eq: proof robust bellman sarmdp 10}
        V_{h,P,\boldsymbol{\Phi}}^{\pi}(s) = \mathbb{E}_{a\sim \pi_h(\cdot|s)}[R_h(s,a)] + \mathbb{E}_{a\sim \pi_h(\cdot|s)}\left[\inf_{\tilde{P}_h\in\boldsymbol{\Phi}(P_h)}\mathbb{E}_{s'\sim \tilde{P}_h(\cdot|s,a)}[V_{h+1,P,\boldsymbol{\Phi}}^{\pi}(s')\right].
    \end{align}
    By the conclusion 2. and the definition of $P_h^{\pi,\dagger}$ in \eqref{eq: proof robust bellman sarmdp 2}, we can obtain from \eqref{eq: proof robust bellman sarmdp 10} that
    \begin{align*}
         V_{h,P,\boldsymbol{\Phi}}^{\pi}(s) & = \mathbb{E}_{a\sim \pi_h(\cdot|s)}[R_h(s,a)] + \mathbb{E}_{a\sim \pi_h(\cdot|s),s'\sim P_h^{\pi,\dagger}(\cdot|s,a)}[V_{h+1,P,\boldsymbol{\Phi}}^{\pi}(s')]\\
         & = \mathbb{E}_{a\sim \pi_h(\cdot|s)}[R_h(s,a)] + \inf_{\tilde{P}_h\in\boldsymbol{\Phi}(P_h)} \mathbb{E}_{a\sim \pi_h(\cdot|s),s'\sim \tilde{P}_h(\cdot|s,a)}[V_{h+1,P,\boldsymbol{\Phi}}^{\pi}(s')].
    \end{align*}
    This finishes the proof of Proposition \ref{prop: robust Bellman} under Assumption \ref{ass: sarmdp}.
\end{proof}

\subsection{Proof of Theorem \ref{thm: existence of robust nash equilibrium}}\label{subsec: proof prop rne bellman equation}

\begin{proof}[Proof of Theorem \ref{thm: existence of robust nash equilibrium}]
    To show that $\boldsymbol{\pi}$ is a RNE policy, we prove the following stronger result:
    \begin{align}\label{eq: proof rne bellman equation 0}
        V_{h,P,\boldsymbol{\Phi}}^{\boldsymbol{\pi},i}(s_h) = \sup_{\tilde{\pi}^i\in\Delta(\mathcal{A}|\mathcal{S},H)}V_{h,P,\boldsymbol{\Phi}}^{(\tilde{\pi}^i,\boldsymbol{\pi}^{-i}),i}(s_h),\quad \forall h\in[H],s_h\in\mathcal{S}, i\in[N].
    \end{align}
    We prove this result from step $h = H$ to $1$ by induction.
    For step $h=H$, according to \eqref{eq: rne bellman equation}, 
    \begin{align*}
        \boldsymbol{\pi}_H(\cdot|s) = \mathbf{NE}\Big(\big\{R_H^{i}(s,\cdot)\big\}_{i=1}^N\Big),
    \end{align*}
    This directly implies that for any $i\in[N]$ and $s_H\in\mathcal{S}$, 
    \begin{align*}
        V_{H,P,\boldsymbol{\Phi}}^{\boldsymbol{\pi},i}(s_H) = \mathbb{D}_{(\pi_H^i,\boldsymbol{\pi}_H^{-i})}[R_H^{i}(s_H,\cdot)] = \sup_{\tilde{\pi}^i_H\in\Delta(\mathcal{A}|\mathcal{S})} \mathbb{D}_{(\tilde{\pi}_H^i,\boldsymbol{\pi}_H^{-i})}[R_H^{i}(s_H,\cdot)] =  \sup_{\tilde{\pi}^i\in\Delta(\mathcal{A}|\mathcal{S},H)}V_{H,P,\boldsymbol{\Phi}}^{(\tilde{\pi}^i,\boldsymbol{\pi}^{-i}),i}(s_H).
    \end{align*}
    This proves \eqref{eq: proof rne bellman equation 0} for step $H$.
    Now suppose that \eqref{eq: proof rne bellman equation 0} holds for step $h+1,\cdots,H$. 
    Then for step $h$, according to \eqref{eq: rne bellman equation}, 
    \begin{align}
        \boldsymbol{\pi}_h(\cdot|s) = \mathbf{NE}\Big(\big\{Q_{h,P,\boldsymbol{\Phi}}^{\boldsymbol{\pi},i}(s,\cdot)\big\}_{i=1}^N\Big),\notag
    \end{align}
    This means that for any $i\in[N]$ and $s_h\in\mathcal{S}$, it holds that 
    \begin{align}\label{eq: proof rne bellman equation 1}
        V_{h,P,\boldsymbol{\Phi}}^{\boldsymbol{\pi},i}(s_h) = \mathbb{D}_{(\pi_h^i,\boldsymbol{\pi}_h^{-i})}[Q_{h,P,\boldsymbol{\Phi}}^{\boldsymbol{\pi},i}(s_h,\cdot)]= \sup_{\tilde{\pi}^i_h\in\Delta(\mathcal{A}|\mathcal{S})} \mathbb{D}_{(\tilde{\pi}_h^i,\boldsymbol{\pi}_h^{-i})}[Q_{h,P,\boldsymbol{\Phi}}^{\boldsymbol{\pi},i}(s_h,\cdot)].
    \end{align}
    Now applying the multi-agent Bellman equation (Proposition \ref{prop: robust Bellman}) to the right hand side of \eqref{eq: proof rne bellman equation 1}, we have the following sequence of inequalities, 
    \begin{align}
        V_{h,P,\boldsymbol{\Phi}}^{\boldsymbol{\pi},i}(s_h) &= \sup_{\tilde{\pi}^i_h\in\Delta(\mathcal{A}|\mathcal{S})} \mathbb{D}_{(\tilde{\pi}_h^i,\boldsymbol{\pi}_h^{-i})}\left[R_h^i(s_h,\boldsymbol{a}_h) + \inf_{\tilde{P}_h\in\mathbf{\Phi}(P_h)}\mathbb{E}_{s'\sim \tilde{P}_h(\cdot|s_h,\boldsymbol{a}_h)}[V_{h+1, P, \mathbf{\Phi}}^{\boldsymbol{\pi},i}(s')]\right]\notag \\
        &\leq \sup_{\tilde{\pi}^i\in\Delta(\mathcal{A}|\mathcal{S},H)} \mathbb{D}_{(\tilde{\pi}_h^i,\boldsymbol{\pi}_h^{-i})}\left[R_h^i(s_h,\boldsymbol{a}_h) + \inf_{\tilde{P}_h\in\mathbf{\Phi}(P_h)}\mathbb{E}_{s'\sim \tilde{P}_h(\cdot|s_h,\boldsymbol{a}_h)}[V_{h+1, P, \mathbf{\Phi}}^{(\tilde{\pi}^i,\boldsymbol{\pi}^{-i}),i}(s')]\right] \label{eq: proof rne bellman equation 2}\\ 
        &= \sup_{\tilde{\pi}^i_h\in\Delta(\mathcal{A}|\mathcal{S})} \mathbb{D}_{(\tilde{\pi}_h^i,\boldsymbol{\pi}_h^{-i})}\left[R_h^i(s_h,\boldsymbol{a}_h) + \sup_{\tilde{\pi}^i\in\Delta(\mathcal{A}|\mathcal{S},H)}\inf_{\tilde{P}_h\in\mathbf{\Phi}(P_h)}\mathbb{E}_{s'\sim \tilde{P}_h(\cdot|s_h,\boldsymbol{a}_h)}[V_{h+1, P, \mathbf{\Phi}}^{(\tilde{\pi}^i,\boldsymbol{\pi}^{-i}),i}(s')]\right]\notag\\
        &\leq \sup_{\tilde{\pi}^i_h\in\Delta(\mathcal{A}|\mathcal{S})} \mathbb{D}_{(\tilde{\pi}_h^i,\boldsymbol{\pi}_h^{-i})}\left[R_h^i(s_h,\boldsymbol{a}_h) + \inf_{\tilde{P}_h\in\mathbf{\Phi}(P_h)}\mathbb{E}_{s'\sim \tilde{P}_h(\cdot|s_h,\boldsymbol{a}_h)}\left[\sup_{\tilde{\pi}^i\in\Delta(\mathcal{A}|\mathcal{S},H)}V_{h+1, P, \mathbf{\Phi}}^{(\tilde{\pi}^i,\boldsymbol{\pi}^{-i}),i}(s')\right]\right]\notag\\
        &=\sup_{\tilde{\pi}^i_h\in\Delta(\mathcal{A}|\mathcal{S})} \mathbb{D}_{(\tilde{\pi}_h^i,\boldsymbol{\pi}_h^{-i})}\left[R_h^i(s_h,\boldsymbol{a}_h) + \inf_{\tilde{P}_h\in\mathbf{\Phi}(P_h)}\mathbb{E}_{s'\sim \tilde{P}_h(\cdot|s_h,\boldsymbol{a}_h)}[V_{h+1, P, \mathbf{\Phi}}^{\boldsymbol{\pi},i}(s')]\right]\notag\\
        &=\sup_{\tilde{\pi}^i_h\in\Delta(\mathcal{A}|\mathcal{S})} \mathbb{D}_{(\tilde{\pi}_h^i,\boldsymbol{\pi}_h^{-i})}[Q_{h,P,\boldsymbol{\Phi}}^{\boldsymbol{\pi},i}(s_h,\cdot)]\notag \\
        &=V_{h,P,\boldsymbol{\Phi}}^{\boldsymbol{\pi},i}(s_h),
    \end{align}
    where the second inequality is due the minimax inequality, the third equality uses the correctness of \eqref{eq: proof rne bellman equation 0} at step $h+1$, and the last equality is due to \eqref{eq: proof rne bellman equation 1}.
    Therefore, we conclude that all the above inequalities are actually equalities.
    Especially, we have that 
    \begin{align*}
        V_{h,P,\boldsymbol{\Phi}}^{\boldsymbol{\pi},i}(s_h) = \eqref{eq: proof rne bellman equation 2} = \sup_{\tilde{\pi}^i\in\Delta(\mathcal{A}|\mathcal{S},H)}V_{h,P,\boldsymbol{\Phi}}^{(\tilde{\pi}^i,\boldsymbol{\pi}^{-i}),i}(s_h),
    \end{align*}
    which proves \eqref{eq: proof rne bellman equation 0} for step $h$.
    An induction finishes the proof of Theorem \ref{thm: existence of robust nash equilibrium}.
\end{proof}

%% file: tex/appendix/sketch.tex
\section{Proof of Main Results for RMDP (Theorem \ref{thm: subopt general})}\label{sec: sketch}

In this section, we prove Theorem \ref{thm: subopt general}.
Let $\mathcal{E}^\dagger$ denote the event that both Condition \ref{cond: accuracy} and \ref{cond: model estimation} hold, which happens with probability at least $1-2\delta$.
In the following, we always assume that $\mathcal{E}^{\dagger}$ holds.

\begin{proof}[Proof of Theorem \ref{thm: subopt general}]
By the definition of $\text{SubOpt}(\hat{\pi}; s)$ in \eqref{eq: suboptimality}, we have that
\begin{align}
    \text{SubOpt}(\hat{\pi}; s_1) &= V_{1,P^{\star},\mathbf{\Phi}}^{\pi^{\star}}(s_1) - V_{1,P^{\star},\mathbf{\Phi}}^{\hat{\pi}}(s_1) \notag\\
    &=  V_{1,P^{\star},\mathbf{\Phi}}^{\pi^{\star}}(s_1) 
        - \inf_{P\in\hat{\mathcal{P}}}V_{1,P,\mathbf{\Phi}}^{\pi^{\star}}(s_1)
        + \inf_{P\in\hat{\mathcal{P}}}V_{1,P,\mathbf{\Phi}}^{\pi^{\star}}(s_1)
        - V_{1,P^{\star},\mathbf{\Phi}}^{\hat{\pi}}(s_1) \notag\\
    &\leq V_{1,P^{\star},\mathbf{\Phi}}^{\pi^{\star}}(s_1) 
        - \inf_{P\in\hat{\mathcal{P}}}V_{1,P,\mathbf{\Phi}}^{\pi^{\star}}(s_1)
        + \inf_{P\in\hat{\mathcal{P}}}V_{1,P,\mathbf{\Phi}}^{\hat{\pi}}(s_1)
        - V_{1,P^{\star},\mathbf{\Phi}}^{\hat{\pi}}(s_1)\label{eq: sketch optimality}\\
    &\leq V_{1,P^{\star},\mathbf{\Phi}}^{\pi^{\star}}(s_1) 
        - \inf_{P\in\hat{\mathcal{P}}}V_{1,P,\mathbf{\Phi}}^{\pi^{\star}}(s_1)\label{eq: sketch P star in hat P}\\
    & = \sup_{P\in\hat{\mathcal{P}}}\Big\{V_{1,P^{\star},\mathbf{\Phi}}^{\pi^{\star}}(s_1)  
        - V_{1,P,\mathbf{\Phi}}^{\pi^{\star}}(s_1)\Big\}\label{eq: sketch sup}.
\end{align}
Here \eqref{eq: sketch optimality} follows from our choice of $\hat{\pi}$ in \eqref{eq: hat pi}, and \eqref{eq: sketch P star in hat P} follows from Condition \ref{cond: accuracy}.
In the sequel, we present the upper bound on the right hand side of \eqref{eq: sketch sup}.
For notational simplicity, for any $P$ in the confidence region $\hat{\mathcal{P}}$ and any step $h\in[H]$, we denote that
\begin{align}
    \Delta_{h,P,\mathbf{\Phi}}(s_h, a_h) = Q_{h,P^{\star},\mathbf{\Phi}}^{\pi^{\star}}(s_h, a_h)  
    - Q_{h,P,\mathbf{\Phi}}^{\pi^{\star}}(s_h, a_h).
\end{align}
Using the robust Bellman equation in Proposition \ref{prop: robust Bellman}, we can derive that
\begin{align*}
    &\Delta_{h,P,\mathbf{\Phi}}(s_h, a_h) \\
    &\qquad = \inf_{\tilde{P}_h\in\mathbf{\Phi}(P_h^{\star})}\mathbb{E}_{s'\sim \tilde{P}_h(\cdot|s_h,a_h)}[V_{h+1,P^{\star},\mathbf{\Phi}}^{\pi^{\star}}(s')] - \inf_{\tilde{P}_h\in\mathbf{\Phi}(P_h)}\mathbb{E}_{s'\sim \tilde{P}_h(\cdot|s_h,a_h)}[V_{h+1,P,\mathbf{\Phi}}^{\pi^{\star}}(s')\\
    &\qquad =  \underbrace{\inf_{\tilde{P}_h\in\mathbf{\Phi}(P_h^{\star})}\mathbb{E}_{s'\sim \tilde{P}_h(\cdot|s_h,a_h)}[V_{h+1,P^{\star},\mathbf{\Phi}}^{\pi^{\star}}(s')] - \inf_{\tilde{P}_h\in\mathbf{\Phi}(P_h^{\star})}\mathbb{E}_{s'\sim \tilde{P}_h(\cdot|s_h,a_h)}[V_{h+1,P,\mathbf{\Phi}}^{\pi^{\star}}(s')]}_{\text{Term (i)}}\\
    &\qquad \qquad +\underbrace{\inf_{\tilde{P}_h\in\mathbf{\Phi}(P_h^{\star})}\mathbb{E}_{s'\sim \tilde{P}_h(\cdot|s_h,a_h)}[V_{h+1,P,\mathbf{\Phi}}^{\pi^{\star}}(s')] - \inf_{\tilde{P}_h\in\mathbf{\Phi}(P_h)}\mathbb{E}_{s'\sim \tilde{P}_h(\cdot|s_h,a_h)}[V_{h+1,P,\mathbf{\Phi}}^{\pi^{\star}}(s')]}_{\text{Term (ii)}}.
\end{align*}
\paragraph*{Term (i).} For the term (i), considering denote that
\begin{align}\label{eq: P sigma}
    P_{h}^{\pi^{\star},\dagger} = \arginf_{\tilde{P}_h\in\mathbf{\Phi}(P_h^{\star})}\mathbb{E}_{s'\sim \tilde{P}_h(\cdot|s,a)}[V_{h+1, P, \mathbf{\Phi}}^{\pi^{\star}}(s')],\quad \forall (s,a)\in\mathcal{S}\times\mathcal{A}.
\end{align}
This notation is consistent with the notation of $P_h^{\pi,\dagger}$ in \eqref{eq: proof robust bellman sarmdp 2} in the proof of Proposition \ref{prop: robust Bellman} (robust Bellman equation). 
It is because Assumption \ref{ass: sarmdp} ($\mathcal{S}\times\mathcal{A}$-rectangular robust set) that we can choose a \emph{single} transition kernel $P_{h}^{\pi^{\star},\dagger}$ that satisfies \eqref{eq: P sigma} for each $(s,a)$-pair.
Using the definition of $P_{h}^{\pi^{\star},\dagger}$, we observe that the following two relationships hold for any state $(s_h, a_h)\in\mathcal{S}$, 
\begin{align*}
    &\inf_{\tilde{P}_h\in\mathbf{\Phi}(P_h^{\star})}\mathbb{E}_{s'\sim \tilde{P}_h(\cdot|s_h,a_h)}[V_{h+1,P^{\star},\mathbf{\Phi}}^{\pi^{\star}}(s')] \leq \mathbb{E}_{s'\sim P_{h}^{\pi^{\star},\dagger}(\cdot|s_h,a_h)}[V_{h+1, P^{\star},\mathbf{\Phi}}^{\pi^{\star}}(s')],\\
    &\inf_{\tilde{P}_h\in\mathbf{\Phi}(P_h^{\star})}\mathbb{E}_{s'\sim \tilde{P}_h(\cdot|s_h,a_h)}[V_{h+1,P,\mathbf{\Phi}}^{\pi^{\star}}(s')] = \mathbb{E}_{s'\sim P_{h}^{\pi^{\star},\dagger}(\cdot|s_h,a_h)}[V_{h+1, P,\mathbf{\Phi}}^{\pi^{\star}}(s')].
\end{align*}
Using these two observations, we can upper bound the term (i) as 
\begin{align}
    \text{Term (i)} &\leq  \mathbb{E}_{s'\sim P_{h}^{\pi^{\star},\dagger}(\cdot|s_h,a_h)}[V_{h+1, P^{\star},\mathbf{\Phi}}^{\pi^{\star}}(s')] - \mathbb{E}_{s'\sim P_{h}^{\pi^{\star},\dagger}(\cdot|s_h,a_h)}[V_{h+1, P,\mathbf{\Phi}}^{\pi^{\star}}(s')]\notag\\
    & = \mathbb{E}_{s'\sim P_{h}^{\pi^{\star},\dagger}(\cdot|s_h,a_h), a'\sim \pi^{\star}_{h+1}(\cdot|s')}[\Delta_{h+1,P,\mathbf{\Phi}}(s',a')],\label{eq: sketch bound i}
\end{align}
where in the equality we use the robust Bellman equation (Proposition \ref{prop: robust Bellman}).

\paragraph*{Term (ii).} For the term (ii), currently we simply denote this term by $\Delta_{h,P,\mathbf{\Phi}}^{\mathrm{(ii)}}(s_h, a_h)$.
Combining this with \eqref{eq: sketch bound i}, we can derive that,
\begin{align}
    \Delta_{h,P,\mathbf{\Phi}}(s_h, a_h)  &= \text{Term (i)} + \text{Term (ii)}\notag\\
    &\leq \mathbb{E}_{s'\sim P_{h}^{\pi^{\star},\dagger}(\cdot|s_h,a_h), a'\sim \pi^{\star}_{h+1}(\cdot|s')}[\Delta_{h+1,P,\mathbf{\Phi}}(s', a')] + \Delta_{h,P,\mathbf{\Phi}}^{\mathrm{(ii)}}(s_h, a_h).\label{eq: sketch bound ii}
\end{align}
By recursively applying \eqref{eq: sketch bound ii} and then plugging in the definition of $\Delta_{h,P,\mathbf{\Phi}}^{\mathrm{(ii)}}$, we can obtain that 
\begin{align}
    \mathbb{E}_{a_1\sim \pi^{\star}_1(\cdot|s_1)}[\Delta_{1,P,\mathbf{\Phi}}(s_1, a_1)] & \leq \sum_{h=1}^H\mathbb{E}_{(s_h,a_h)\sim d_{P^{\pi^{\star},\dagger},h}^{\pi^{\star}}}[\Delta_{h,P,\mathbf{\Phi}}^{\mathrm{(ii)}}(s_h,a_h)]\notag\\
    & = \sum_{h=1}^H\mathbb{E}_{(s_h,a_h)\sim d_{P^{\pi^{\star},\dagger},h}^{\pi^{\star}}}\bigg[\inf_{\tilde{P}_h\in\mathbf{\Phi}(P_h^{\star})}\mathbb{E}_{s'\sim \tilde{P}_h(\cdot|s_h,a_h)}[V_{h+1,P,\mathbf{\Phi}}^{\pi^{\star}}(s')] \notag\\
    &\qquad - \inf_{\tilde{P}_h\in\mathbf{\Phi}(P_h)}\mathbb{E}_{s'\sim \tilde{P}_h(\cdot|s_h,a_h)}[V_{h+1,P,\mathbf{\Phi}}^{\pi^{\star}}(s')]\bigg],\label{eq: sketch delta 1 bound}
\end{align}
where $d_{P^{\pi^{\star},\dagger},h}^{\pi^{\star}}(\cdot,\cdot)$ is the state-action visitation distribution induced by the transition kernels $P^{\pi^{\star},\dagger} = \{P^{\pi^{\star},\dagger}_h\}_{h=1}^H$ and the optimal policy $\pi^{\star}$. 
Now we bound the right hand side of \eqref{eq: sketch delta 1 bound} using Condition~\ref{cond: model estimation}.
By Cauchy-Schwartz inequality, we have that for each $h\in[H]$, 
\begin{align}
    &\mathbb{E}_{(s_h, a_h)\sim d_{P^{\pi^{\star},\dagger},h}^{\pi^{\star}}}\bigg[\inf_{\tilde{P}_h\in\mathbf{\Phi}(P_h^{\star})}\mathbb{E}_{s'\sim \tilde{P}_h(\cdot|s_h,a_h)}[V_{h+1,P,\mathbf{\Phi}}^{\pi^{\star}}(s')] - \inf_{\tilde{P}_h\in\mathbf{\Phi}(P_h)}\mathbb{E}_{s'\sim \tilde{P}_h(\cdot|s_h,a_h)}[V_{h+1,P,\mathbf{\Phi}}^{\pi^{\star}}(s')]\bigg]\notag\\
    &\quad =\mathbb{E}_{(s_h, a_h)\sim d_{P^{\star},h}^{\pi^{\mathrm{b}}}}\Bigg[\frac{d_{P^{\pi^{\star},\dagger},h}^{\pi^{\star}}(s_h,a_h)}{d_{P^{\star},h}^{\pi^{\mathrm{b}}}(s_h,a_h)}\cdot\bigg(\inf_{\tilde{P}_h\in\mathbf{\Phi}(P_h^{\star})}\mathbb{E}_{s'\sim \tilde{P}_h(\cdot|s_h,a_h)}[V_{h+1,P,\mathbf{\Phi}}^{\pi^{\star}}(s')] \notag\\
    &\quad\quad\quad- \inf_{\tilde{P}_h\in\mathbf{\Phi}(P_h)}\mathbb{E}_{s'\sim \tilde{P}_h(\cdot|s_h,a_h)}[V_{h+1,P,\mathbf{\Phi}}^{\pi^{\star}}(s')]\bigg)\Bigg]\notag\\
    &\quad \leq \sqrt{\mathbb{E}_{(s_h, a_h)\sim d_{P^{\star},h}^{\pi^{\mathrm{b}}}}\left[\left(\frac{d_{P^{\pi^{\star},\dagger},h}^{\pi^{\star}}(s_h,a_h)}{d_{P^{\star},h}^{\pi^{\mathrm{b}}}(s_h,a_h)}\right)^2\right]}\cdot\sqrt{\mathrm{Err}_{h}^{\mathbf{\Phi}}(n,\delta)},\label{eq: sketch delta 1 bound ii}
\end{align}
where the last inequality follows from Condition \ref{cond: model estimation}. 
Furthermore, by Assumption \ref{ass: partial coverage}, we know that 
\begin{align*}
    \mathbb{E}_{(s_h, a_h)\sim d_{P^{\star},h}^{\pi^{\mathrm{b}}}}\left[\left(\frac{d_{P^{\pi^{\star},\dagger},h}^{\pi^{\star}}(s_h,a_h)}{d_{P^{\star},h}^{\pi^{\mathrm{b}}}(s_h,a_h)}\right)^2\right] &\leq \sup_{P = \{P_h\}_{h=1}^H, P_h\in\mathbf{\Phi}(P_h^{\star})}\mathbb{E}_{(s_h,a_h)\sim d^{\pi^{\mathrm{b}}}_{P^{\star}, h}}\left[\left(\frac{d^{\pi^{\star}}_{P, h}(s_h,a_h)}{d^{\pi^{\mathrm{b}}}_{P^{\star}, h}(s_h,a_h)}\right)^2\right]\\
    &\leq C^{\star}_{P^{\star},\mathbf{\Phi}},
\end{align*}
where $C^{\star}_{P^{\star},\mathbf{\Phi}}$ is defined in Assumption \ref{ass: partial coverage}. 
Applying this to \eqref{eq: sketch delta 1 bound} and \eqref{eq: sketch delta 1 bound ii}, we can derive that
\begin{align*}
    \sup_{P\in\hat{\mathcal{P}}}\Big\{V_{1,P^{\star},\mathbf{\Phi}}^{\pi^{\star}}(s_1)  
        - V_{1,P,\mathbf{\Phi}}^{\pi^{\star}}(s_1)\Big\} = \sup_{P\in\hat{\mathcal{P}}} \{\mathbb{E}_{a_1\sim\pi^{\star}(\cdot|s_1)}[\Delta_{1,P,\mathbf{\Phi}}(s_1, a_1)]\} \leq \sqrt{C^{\star}_{P^{\star},\mathbf{\Phi}}}\cdot \sum_{h=1}^H\sqrt{\mathrm{Err}_{h}^{\mathbf{\Phi}}(n,\delta)}.
\end{align*}
Finally, by inequality \eqref{eq: sketch sup}, we finish the proof of Theorem \ref{thm: subopt general}. 
\end{proof}

%% file: tex/appendix/offline_rmg.tex
\section{Proof of Main Results for RMG (Theorem~\ref{thm:offline:FA})} \label{appendix:pf:offline:FA}

\begin{proof}[Proof of Theorem~\ref{thm:offline:FA}]
      Under Condition~\ref{cond: accuracy rmg}, we have that for any policy $\boldsymbol{\pi}$ and player $i$,
      \# 
      J_{\texttt{Pess}^2}^i(\boldsymbol{\pi}) &=  \inf_{ P_h \in \hat{\cP}_h,1\leq h\leq H} \,\,\inf_{\tilde{P}_h\in\boldsymbol{\Phi}(P_h),1\leq h\leq H}\,\,V_{1}^{\boldsymbol{\pi}, i}(s_1;\{\tilde{P}_h\}_{h=1}^H)\notag \\
      &\leq \inf_{\tilde{P}_h\in\boldsymbol{\Phi}(P_h^{\star}),1\leq h\leq H}\,\,V_{1}^{\boldsymbol{\pi}, i}(s_1;\{\tilde{P}_h\}_{h=1}^H)  \notag \\
      & = V_{1,P^\star,\boldsymbol{\Phi}}^{\boldsymbol{\pi},i}(s_1), \label{eq:3101}
      \#  
      and that 
      \#
        J_{\texttt{Opt-Pess}}^i(\boldsymbol{\pi}) &=  \sup_{ P_h \in \hat{\cP}_h,1\leq h\leq H}\,\,\sup_{\tilde{\pi}^i\in\Delta(\mathcal{A}^i|\mathcal{S},H)} \,\,\inf_{\tilde{P}_h\in\boldsymbol{\Phi}(P_h),1\leq h\leq H}\,\,V_{1}^{(\tilde{\pi}^{i}, \boldsymbol{\pi}^{-i}), i}(s_1;\{\tilde{P}_h\}_{h=1}^H)\notag \\
        &\geq \sup_{\tilde{\pi}^i\in\Delta(\mathcal{A}^i|\mathcal{S},H)} \,\,\inf_{\tilde{P}_h\in\boldsymbol{\Phi}(P_h^{\star}),1\leq h\leq H}\,\,V_{1}^{(\tilde{\pi}^{i}, \boldsymbol{\pi}^{-i}), i}(s_1;\{\tilde{P}_h\}_{h=1}^H) \notag\\
        & =V_{1,P^\star,\boldsymbol{\Phi}}^{(\dagger,\boldsymbol{\pi}^{-i}),i}(s_1),\label{eq:3101+}
      \#
      which further implies that
      \begin{align} 
      \mathrm{RNEGap}_{\boldsymbol{\Phi}}(\hat{\boldsymbol{\pi}};s_1) & = \max_{i \in [N]} \left\{ V_{1,P^\star,\boldsymbol{\Phi}}^{(\dagger,\hat{\boldsymbol{\pi}}^{-i}),i}(s_1) -  V_{1,P^\star,\boldsymbol{\Phi}}^{\hat{\boldsymbol{\pi}},i}(s_1) \right\} \notag \\ 
      & \le \max_{i \in [N]} \left\{ J_{\texttt{Opt-Pess}}^i(\hat{\boldsymbol{\pi}}) - J_{\texttt{Pess}^2}^i(\hat{\boldsymbol{\pi}}) \right\} \notag\\ 
      & \le \max_{i \in [N]} \left\{ J_{\texttt{Opt-Pess}}^i({\boldsymbol{\pi}}_{\mathrm{RNE}}) - J_{\texttt{Pess}^2}^i({\boldsymbol{\pi}}_{\mathrm{RNE}}) \right\} \notag \\ 
      & = \max_{i \in [N]} \Big\{ \underbrace{J_{\texttt{Opt-Pess}}^i({\boldsymbol{\pi}}_{\mathrm{RNE}}) - V_{1,P^\star,\boldsymbol{\Phi}}^{\boldsymbol{\pi}_{\mathrm{RNE}},i}(s_1)}_{\displaystyle\mathrm{(I)}} + \underbrace{V_{1,P^\star,\boldsymbol{\Phi}}^{\boldsymbol{\pi}_{\mathrm{RNE}},i}(s_1) - J_{\texttt{Pess}^2}^i({\boldsymbol{\pi}}_{\mathrm{RNE}})}_{\displaystyle\mathrm{(II)}} \Big\},\label{eq:3102}
       \end{align}
      where the the first inequality uses \eqref{eq:3101} and \eqref{eq:3101+}, and the second inequality follows from the definition of $\hat{\boldsymbol{\pi}}$ that $\hat{\boldsymbol{\pi}} = \argmin_{\boldsymbol{\pi}} \max_{i \in [N]} \{ J_{\texttt{Opt-Pess}}^i({\boldsymbol{\pi}}) - J_{\texttt{Pess}^2}^i({\boldsymbol{\pi}})\}$. 
      Here $\boldsymbol{\pi}_{\mathrm{RNE}}$ is the RNE policy in Assumption \ref{assumption:unilateral}.

      \paragraph{Term (I).}
      For Term (I) in \eqref{eq:3102}, we have
      \begin{equation}
          \begin{aligned} \label{eq:3103}
      {\displaystyle\mathrm{(I)}} & = J_{\texttt{Opt-Pess}}^i({\boldsymbol{\pi}}_{\mathrm{RNE}}) - V_{1,P^\star,\boldsymbol{\Phi}}^{\boldsymbol{\pi}_{\mathrm{RNE}},i}(s_1) \\ 
      &= \sup_{P \in \hat{\cP}} \sup_{\pi^i} V_{1,P,\boldsymbol{\Phi}}^{(\pi^i, (\boldsymbol{\pi}_{\mathrm{RNE}})^{-i}),i}(s_1) - V_{1,P^\star,\boldsymbol{\Phi}}^{\boldsymbol{\pi}_{\mathrm{RNE}},i}(s_1) \\ 
      & \le \sup_{P \in \hat{\cP}} \sup_{\pi^i} \left\{ V_{1,P,\boldsymbol{\Phi}}^{(\pi^i, (\boldsymbol{\pi}_{\mathrm{RNE}})^{-i}),i}(s_1) -  V_{1,P^\star,\boldsymbol{\Phi}}^{(\pi^i, (\boldsymbol{\pi}_{\mathrm{RNE}})^{-i}),i}(s_1) \right\},
          \end{aligned}
      \end{equation}
      where the first equality follows from the definition of $J_{\texttt{Opt-Pess}}^i$ in \eqref{eq:bar:J} and the second inequality uses the fact that $\boldsymbol{\pi}_{\mathrm{RNE}}$ is an RNE. 
      Fix $(i, P, \pi^i)$, we use the notation
      \# \label{eq:3104}
       \Delta_{h, \mathbf{\Phi}}(s_h,\boldsymbol{a}_h) = Q_{h, P,\boldsymbol{\Phi}}^{(\pi^i, (\boldsymbol{\pi}_{\mathrm{RNE}})^{-i}),i}(s_h,\boldsymbol{a}_h) -  Q_{h, P^\star,\boldsymbol{\Phi}}^{(\pi^i, (\boldsymbol{\pi}_{\mathrm{RNE}})^{-i}),i}(s_h,\boldsymbol{a}_h), \quad \tilde{\boldsymbol{\pi}}_{\mathrm{RNE}} = (\pi^i, (\boldsymbol{\pi}_{\mathrm{RNE}})^{-i}) .
      \#  
      By the multi-agent robust Bellman equation in \eqref{eq: robust bellman V rmg} and  \eqref{eq: robust bellman Q rmg}, we have 
      \# \label{eq:3105}
      & \Delta_{h, \boldsymbol{\Phi}}(s_h, \boldsymbol{a}_h) \\ 
      & \quad = \inf_{\tilde{P}_h \in \boldsymbol{\Phi}(P_h) } \EE_{s' \sim \tilde{P}_h(\cdot \mid s_h, \boldsymbol{a}_h) } [V_{h+1, P,\boldsymbol{\Phi}}^{\tilde{\boldsymbol{\pi}}_{\mathrm{RNE}},i}(s')] - \inf_{\tilde{P}_h \in \boldsymbol{\Phi}(P_h^\star) } \EE_{s' \sim \tilde{P}_h(\cdot \mid s_h, \boldsymbol{a}_h) } [V_{h+1, P^\star,\boldsymbol{\Phi}}^{\tilde{\boldsymbol{\pi}}_{\mathrm{RNE}},i}(s')] \notag \\
      & \quad = \underbrace{\inf_{\tilde{P}_h \in \boldsymbol{\Phi}(P_h) } \EE_{s' \sim \tilde{P}_h(\cdot \mid s_h, \boldsymbol{a}_h) } [V_{h+1, P,\boldsymbol{\Phi}}^{\tilde{\boldsymbol{\pi}}_{\mathrm{RNE}},i}(s')] - \inf_{\tilde{P}_h \in \boldsymbol{\Phi}(P_h^\star) } \EE_{s' \sim \tilde{P}_h(\cdot \mid s_h, \boldsymbol{a}_h) } [V_{h+1, P,\boldsymbol{\Phi}}^{\tilde{\boldsymbol{\pi}}_{\mathrm{RNE}},i}(s')]}_{\Delta_{h, \boldsymbol{\Phi}}^{\mathrm{(i)}}(s_h, \boldsymbol{a}_h) }  \notag \\ 
      & \quad \quad  + \underbrace{\inf_{\tilde{P}_h \in \boldsymbol{\Phi}(P_h^\star) } \EE_{s' \sim \tilde{P}_h(\cdot \mid s_h, \boldsymbol{a}_h) } [V_{h+1, P,\boldsymbol{\Phi}}^{\tilde{\boldsymbol{\pi}}_{\mathrm{RNE}},i}(s')] - \inf_{\tilde{P}_h \in \boldsymbol{\Phi}(P_h^\star) } \EE_{s' \sim \tilde{P}_h(\cdot \mid s_h, \boldsymbol{a}_h) } [V_{h+1, P^\star,\boldsymbol{\Phi}}^{\tilde{\boldsymbol{\pi}}_{\mathrm{RNE}},i}(s')]}_{\Delta_{h, \boldsymbol{\Phi}}^{\mathrm{(ii)}}(s_h, \boldsymbol{a}_h) } . \notag 
      \#  
      To facilitate the following analysis, we define
      \# \label{eq:3106}
      P_h^{\tilde{\boldsymbol{\pi}}_{\mathrm{RNE}}, \dagger} = \arginf_{\tilde{P}_h \in \boldsymbol{\Phi}(P_h^\star) } \EE_{s' \sim \tilde{P}_h(\cdot \mid s_h, \boldsymbol{a}_h) } [V_{h+1, P^\star,\boldsymbol{\Phi}}^{\tilde{\boldsymbol{\pi}}_{\mathrm{RNE}},i}(s_h)],
      \#  
      which is well defined due to Assumption \ref{ass: s a rectangular} ($\cS\times\cA$-rectangular). 
      With this notation, we have
      \# \label{eq:3107}
      \Delta_{h, \boldsymbol{\Phi}}^{\mathrm{(ii)}}(s_h, \boldsymbol{a}_h) &\le \EE_{s' \sim {P}_h^{\tilde{\boldsymbol{\pi}}_{\mathrm{RNE}}, \dagger}(\cdot \mid s_h, \boldsymbol{a}_h) } [V_{h+1, P,\boldsymbol{\Phi}}^{\tilde{\boldsymbol{\pi}}_{\mathrm{RNE}},i}(s')] - \EE_{s' \sim {P}_h^{\tilde{\boldsymbol{\pi}}_{\mathrm{RNE}}, \dagger}(\cdot \mid s_h, \boldsymbol{a}_h) }[V_{h+1, P^\star,\boldsymbol{\Phi}}^{\tilde{\boldsymbol{\pi}}_{\mathrm{RNE}},i}(s')]  \notag \\ 
      & = \EE_{s' \sim {P}_h^{\tilde{\boldsymbol{\pi}}_{\mathrm{RNE}}, \dagger}(\cdot \mid s_h, \boldsymbol{a}_h) , \boldsymbol{a}'\sim \tilde{\boldsymbol{\pi}}_{\mathrm{RNE},h+1}(\cdot|s')} [\Delta_{h+1, \boldsymbol{\Phi}}(s', \boldsymbol{a}')] ,
      \# 
      where the last equality uses the definition of $\Delta_{h, \boldsymbol{\Phi}}$ in \eqref{eq:3104}. Plugging \eqref{eq:3107} into \eqref{eq:3105} yields that
      \# \label{eq:3108}
      \Delta_{h, \boldsymbol{\Phi}}(s_h, \boldsymbol{a}_h) \le  \EE_{s' \sim {P}_h^{\tilde{\boldsymbol{\pi}}_{\mathrm{RNE}}, \dagger}(\cdot \mid s_h, \boldsymbol{a}_h) , \boldsymbol{a}'\sim \tilde{\boldsymbol{\pi}}_{\mathrm{RNE},h+1}(\cdot|s')} [\Delta_{h+1, \boldsymbol{\Phi}}(s', \boldsymbol{a}')] + \Delta_{h, \boldsymbol{\Phi}}^{\mathrm{(i)}}(s_h, \boldsymbol{a}_h).
      \# 
      Recursively expanding \eqref{eq:3107} across $h \in [H]$ gives that
      \# 
      &\mathbb{E}_{\boldsymbol{a}_1\sim\tilde{\boldsymbol{\pi}}_{\mathrm{RNE},1}(\cdot|s_1)}[\Delta_{1, \boldsymbol{\Phi}}(s_1,\boldsymbol{a}_1)]\notag\\
      &\qquad \leq \sum_{h = 1}^H \EE_{(s_h,\boldsymbol{a}_h) \sim d_{P^{\tilde{\boldsymbol{\pi}}_{\mathrm{RNE}}, \dagger}, h}^{\tilde{\boldsymbol{\pi}}_{\mathrm{RNE}}} } [\Delta_{h, \boldsymbol{\Phi}}^{\mathrm{(i)}}(s_h, \boldsymbol{a}_h)] \label{eq:3109}\\
      &\qquad = \sum_{h = 1}^H \EE_{(s_h, \boldsymbol{a}_h) \sim d_{P^{\tilde{\boldsymbol{\pi}}_{\mathrm{RNE}}, \dagger}, h}^{\tilde{\boldsymbol{\pi}}_{\mathrm{RNE}}} } \left[ \inf_{\tilde{P}_h \in \boldsymbol{\Phi}(P_h) } \EE_{s' \sim \tilde{P}_h(\cdot \mid s_h, \boldsymbol{a}_h) } [V_{h+1, P,\boldsymbol{\Phi}}^{\tilde{\boldsymbol{\pi}}_{\mathrm{RNE}},i}(s')] - \inf_{\tilde{P}_h \in \boldsymbol{\Phi}(P_h^\star) } \EE_{ s' \sim \tilde{P}_h(\cdot \mid s_h, \boldsymbol{a}_h) } [V_{h+1, P^\star,\boldsymbol{\Phi}}^{\tilde{\boldsymbol{\pi}}_{\mathrm{RNE}},i}(s')]   \right],\notag
      \# 
      where $d_{P^{\tilde{\boldsymbol{\pi}}_{\mathrm{RNE}}, \dagger}, h}^{\tilde{\boldsymbol{\pi}}_{\mathrm{RNE}}}$ denotes the state-action distribution induced by the joint policy $\tilde{\boldsymbol{\pi}}_{\mathrm{RNE}}$ defined in \eqref{eq:3104} and $P^{\tilde{\boldsymbol{\pi}}_{\mathrm{RNE}}, \dagger} = \{P_h^{\tilde{\boldsymbol{\pi}}_{\mathrm{RNE}}, \dagger}\}_{h \in [H]}$ defined in   \eqref{eq:3106}. 
      In the equality we apply the definition of $\Delta_{h,\boldsymbol{\Phi}}(s_h,\boldsymbol{a}_h)$ in \eqref{eq:3105}.
       Furthermore, by Cauchy-Schwarz inequality, we have that
           \begin{align} 
       & \EE_{(s_h, \boldsymbol{a}_h) \sim d_{P^{\tilde{\boldsymbol{\pi}}_{\mathrm{RNE}}, \dagger}, h}^{\tilde{\boldsymbol{\pi}}_{\mathrm{RNE}}} } \left[ \inf_{\tilde{P}_h \in \boldsymbol{\Phi}(P_h) } \EE_{s' \sim \tilde{P}_h(\cdot \mid s_h, \boldsymbol{a}_h) } [V_{h+1, P,\boldsymbol{\Phi}}^{\tilde{\boldsymbol{\pi}}_{\mathrm{RNE}},i}(s')] - \inf_{\tilde{P}_h \in \boldsymbol{\Phi}(P_h^\star) } \EE_{ s' \sim \tilde{P}_h(\cdot \mid s_h, \boldsymbol{a}_h) } [V_{h+1, P^\star,\boldsymbol{\Phi}}^{\tilde{\boldsymbol{\pi}}_{\mathrm{RNE}},i}(s')]   \right] \notag \\ 
       & \quad = \EE_{(s_h, \boldsymbol{a}_h) \sim d_{P^\star, h}^{\boldsymbol{\pi}^{\mathrm{b}}} } \Bigg[ \frac{d_{P^{\tilde{\boldsymbol{\pi}}_{\mathrm{RNE}}, \dagger}, h}^{\tilde{\boldsymbol{\pi}}_{\mathrm{RNE}}}(s_h, \boldsymbol{a}_h)}{d_{P^\star, h}^{\boldsymbol{\pi}^{\mathrm{b}}}(s_h, \boldsymbol{a}_h)} \cdot \bigg( \inf_{\tilde{P}_h \in \boldsymbol{\Phi}(P_h) } \EE_{s' \sim \tilde{P}_h(\cdot \mid s_h, \boldsymbol{a}_h) } [V_{h+1, P,\boldsymbol{\Phi}}^{\tilde{\boldsymbol{\pi}}_{\mathrm{RNE}},i}(s')]\notag \\ 
       & \quad \qquad - \inf_{\tilde{P}_h \in \boldsymbol{\Phi}(P_h^\star) } \EE_{ s' \sim \tilde{P}_h(\cdot \mid s_h, \boldsymbol{a}_h) } [V_{h+1, P^\star,\boldsymbol{\Phi}}^{\tilde{\boldsymbol{\pi}}_{\mathrm{RNE}},i}(s')]  \bigg) \Bigg], \notag\\ 
       & \quad \le \sqrt{ \EE_{(s_h, \boldsymbol{a}_h) \sim d_{P^\star, h}^{\boldsymbol{\pi}^{\mathrm{b}}} } \left[ \left( \frac{d_{P^{\tilde{\boldsymbol{\pi}}_{\mathrm{RNE}}, \dagger}, h}^{\tilde{\boldsymbol{\pi}}_{\mathrm{RNE}}}(s_h, \boldsymbol{a}_h)}{d_{P^\star, h}^{\boldsymbol{\pi}^{\mathrm{b}}}(s_h, \boldsymbol{a}_h)} \right)^2 \right]  } \cdot \sqrt{\mathbf{Err}_{h}^{\mathbf{\Phi}}(n,\delta)},\label{eq:3113}
           \end{align}
       where the last inequality uses Cauchy-Schwarz inequality and the definition of $\mathbf{Err}_{h}^{\mathbf{\Phi}}(n,\delta)$ in Condition~\ref{cond: model estimation rmg}. Meanwhile, by Assumption~\ref{assumption:unilateral}, we obtain that
       \# \label{eq:3114}
       \EE_{(s_h, \boldsymbol{a}_h) \sim d_{P^\star, h}^{\boldsymbol{\pi}^{\mathrm{b}}} } \left[ \left( \frac{d_{P^{\tilde{\boldsymbol{\pi}}_{\mathrm{RNE}}, \dagger}, h}^{\tilde{\boldsymbol{\pi}}_{\mathrm{RNE}}}(s_h, \boldsymbol{a}_h)}{d_{P^\star, h}^{\boldsymbol{\pi}^{\mathrm{b}}}(s_h, \boldsymbol{a}_h)} \right)^2 \right] &\leq \sup_{P = \{P_h\}_{h=1}^H, P_h\in\mathbf{\Phi}(P_h^{\star})}\mathbb{E}_{(s_h,\boldsymbol{a}_h)\sim d^{\pi^{\mathrm{b}}}_{P^{\star}, h}}\left[\left(\frac{d^{\tilde{\boldsymbol{\pi}}_{\mathrm{RNE}}}_{P, h}(s_h,\boldsymbol{a}_h)}{d^{\boldsymbol{\pi}^{\mathrm{b}}}_{P^{\star}, h}(s_h,\boldsymbol{a}_h)}\right)^2\right] \notag \\
       &\le \boldsymbol{C}^{\mathrm{RNE}}_{P^{\star},\mathbf{\Phi}},
       \#  
       where $\boldsymbol{C}^{\mathrm{RNE}}_{P^{\star},\mathbf{\Phi}}$ is the robust unilateral coverage coefficient defined in Assumption~\ref{assumption:unilateral}. 
       By plugging \eqref{eq:3113} and~\eqref{eq:3114} into \eqref{eq:3109} we obtain that,
       \#
       \mathbb{E}_{\boldsymbol{a}_1\sim\tilde{\boldsymbol{\pi}}_{\mathrm{RNE},1}(\cdot|s_1)}[\Delta_{1, \boldsymbol{\Phi}}(s_1,\boldsymbol{a}_1)]\le \sqrt{ \boldsymbol{C}^{\mathrm{RNE}}_{P^{\star},\mathbf{\Phi}}} \cdot \sum_{h = 1}^H \sqrt{\mathbf{Err}_{h}^{\mathbf{\Phi}}(n,\delta)}.\label{eq:3115-}
       \# 
       Together with \eqref{eq:3103} and \eqref{eq:3104}, we can establish the following upper bound for Term (I) in \eqref{eq:3103}, 
       \# \label{eq:3115}
        {\displaystyle\mathrm{(I)}} \le \sqrt{ \boldsymbol{C}^{\mathrm{RNE}}_{P^{\star},\mathbf{\Phi}}} \cdot \sum_{h = 1}^H \sqrt{\mathbf{Err}_{h}^{\mathbf{\Phi}}(n,\delta)}.
       \# 
       \paragraph{Term (II).} We can tackle Term (II) by a similar way of bounding Term (I) in \eqref{eq:3102}. 
       In specific,
       \#
       {\displaystyle\mathrm{(II)}}  =  V_{1,P^\star,\boldsymbol{\Phi}}^{\boldsymbol{\pi}_{\mathrm{RNE}},i}(s_1) - \inf_{P \in \hat{\cP}} V_{1,P,\boldsymbol{\Phi}}^{\boldsymbol{\pi}_{\mathrm{RNE}},i}(s_1) = \sup_{P \in \hat{\cP}} \left\{ V_{1,P^\star,\boldsymbol{\Phi}}^{\boldsymbol{\pi}_{\mathrm{RNE}},i}(s_1) -  V_{1,P,\boldsymbol{\Phi}}^{\boldsymbol{\pi}_{\mathrm{RNE}},i}(s_1) \right\},
       \# 
       where the first equality follows from the definition of $J_{\texttt{Pess}^2}^i(\cdot)$ in \eqref{eq:underline:J}. Fix a model $P$ and player $i$. Similar to \eqref{eq:3104}, we denote
       \# \label{eq:3117}
       \bar{\Delta}_{h, \boldsymbol{\Phi}}(s_h, \boldsymbol{a}_h) =  Q_{h,P^\star,\boldsymbol{\Phi}}^{\boldsymbol{\pi}_{\mathrm{RNE}},i}(s_h, \boldsymbol{a}_h) -  Q_{h,P,\boldsymbol{\Phi}}^{\boldsymbol{\pi}_{\mathrm{RNE}},i}(s_h, \boldsymbol{a}_h).
       \# 
       By the multi-agent robust Bellman equation in \eqref{eq: robust bellman V rmg} and \eqref{eq: robust bellman Q rmg}, we have
       \# \label{eq:3118}
       &\bar{\Delta}_{h, \boldsymbol{\Phi}}(s_h, \boldsymbol{a}_h) \\
       &\quad = \inf_{\tilde{P}_h\in\mathbf{\Phi}(P_h^{\star})}\mathbb{E}_{s'\sim \tilde{P}_h(\cdot\mid s_h,\boldsymbol{a}_h)}[V_{h+1,P^{\star},\mathbf{\Phi}}^{\boldsymbol{\pi}_{\mathrm{RNE}}, i}(s')] - \inf_{\tilde{P}_h\in\mathbf{\Phi}(P_h)}\mathbb{E}_{s'\sim \tilde{P}_h(\cdot\mid s_h,\boldsymbol{a}_h)}[V_{h+1,P,\mathbf{\Phi}}^{\boldsymbol{\pi}_{\mathrm{RNE}}, i}(s')] \notag \\
       &\quad =  \underbrace{\inf_{\tilde{P}_h\in\mathbf{\Phi}(P_h^{\star})}\mathbb{E}_{s'\sim \tilde{P}_h(\cdot\mid s_h,\boldsymbol{a}_h)}[V_{h+1,P^{\star},\mathbf{\Phi}}^{\boldsymbol{\pi}_{\mathrm{RNE}}, i}(s')] - \inf_{\tilde{P}_h\in\mathbf{\Phi}(P_h^{\star})}\mathbb{E}_{s'\sim \tilde{P}_h(\cdot\mid s_h,\boldsymbol{a}_h)}[V_{h+1,P,\mathbf{\Phi}}^{\boldsymbol{\pi}_{\mathrm{RNE}}, i}(s')]}_{\bar{\Delta}_{h, \boldsymbol{\Phi}}^{\mathrm{(i)}}(s_h, \boldsymbol{a}_h)} \notag \\
       &\qquad +\underbrace{\inf_{\tilde{P}_h\in\mathbf{\Phi}(P_h^{\star})}\mathbb{E}_{s'\sim \tilde{P}_h(\cdot\mid s_h,\boldsymbol{a}_h)}[V_{h+1,P,\mathbf{\Phi}}^{\boldsymbol{\pi}_{\mathrm{RNE}}, i}(s')] - \inf_{\tilde{P}_h\in\mathbf{\Phi}(P_h)}\mathbb{E}_{s'\sim \tilde{P}_h(\cdot\mid s_h,\boldsymbol{a}_h)}[V_{h+1,P,\mathbf{\Phi}}^{\boldsymbol{\pi}_{\mathrm{RNE}}, i}(s')]}_{ \bar{\Delta}_{h, \boldsymbol{\Phi}}^{\mathrm{(ii)}}(s_h, \boldsymbol{a}_h)}. \notag 
       \# 
       For ease of presentation, we define 
       \# \label{eq:3119}
       P_h^{\boldsymbol{\pi}_{\mathrm{RNE}}, \dagger} = \arginf_{\tilde{P}_h\in\mathbf{\Phi}(P_h^{\star})}\mathbb{E}_{\boldsymbol{a}_h\sim\boldsymbol{\pi}^{\star}_h(\cdot\mid s_h),s'\sim \tilde{P}_h(\cdot\mid s_h,\boldsymbol{a}_h)}[V_{h+1,P,\mathbf{\Phi}}^{\boldsymbol{\pi}_{\mathrm{RNE}}, i}(s')].
       \#  
       Then we have that
       \# \label{eq:3120}
       \bar{\Delta}_{h, \boldsymbol{\Phi}}^{\mathrm{(i)}}(s_h, \boldsymbol{a}_h) & \le \mathbb{E}_{s'\sim P_h^{\boldsymbol{\pi}_{\mathrm{RNE}}, \dagger}(\cdot\mid s_h,\boldsymbol{a}_h)}[V_{h+1,P^{\star},\mathbf{\Phi}}^{\boldsymbol{\pi}_{\mathrm{RNE}}, i}(s')] - \mathbb{E}_{s'\sim P_h^{\boldsymbol{\pi}_{\mathrm{RNE}}, \dagger}(\cdot\mid s_h,\boldsymbol{a}_h)}[V_{h+1,P,\mathbf{\Phi}}^{\boldsymbol{\pi}_{\mathrm{RNE}}, i}(s')] \notag \\
       & = \mathbb{E}_{s'\sim P_h^{\boldsymbol{\pi}_{\mathrm{RNE}}, \dagger}(\cdot\mid s_h,\boldsymbol{a}_h), \boldsymbol{a}'\sim \boldsymbol{\pi}_{\mathrm{RNE},h+1}(\cdot|s')}[\bar{\Delta}_{h+1, \boldsymbol{\Phi}}(s',\boldsymbol{a}')] , 
       \# 
       where the first inequality uses the definition of $P_h^{\boldsymbol{\pi}_{\mathrm{RNE}}, \dagger}$ in \eqref{eq:3119} and the equality follows from the definition of $\bar{\Delta}_{h+1, \boldsymbol{\Phi}}(\cdot)$ in \eqref{eq:3117}. Plugging \eqref{eq:3120} into \eqref{eq:3118} and recursively expanding across $h \in [H]$ give that
       \#
       \mathbb{E}_{\boldsymbol{a}_1\sim\boldsymbol{\pi}_{\mathrm{RNE},1}(\cdot|s_1)}[\bar{\Delta}_{1, \boldsymbol{\Phi}}(s_1,\boldsymbol{a}_1)] \le \sum_{h = 1}^H \mathbb{E}_{(s_h,\boldsymbol{a}_h)\sim d_{P^{\boldsymbol{\pi}^{\star},\dagger},h}^{\boldsymbol{\pi}^{\star}}}[\bar{\Delta}_{h,\mathbf{\Phi}}^{\mathrm{(ii)}}(s_h,\boldsymbol{a}_h)],
       \# 
       where $d_{P^{\boldsymbol{\pi}_{\mathrm{RNE}},\dagger},h}^{\boldsymbol{\pi}_{\mathrm{RNE}}}$ denotes the state-action distribution induced by the policy ${\boldsymbol{\pi}}_{\mathrm{RNE}}$ and $P^{{\boldsymbol{\pi}}_{\mathrm{RNE}}, \dagger} = \{P_h^{{\boldsymbol{\pi}}_{\mathrm{RNE}}, \dagger}\}_{h \in [H]}$ defined in \eqref{eq:3119}. Following the derivation of \eqref{eq:3115-}, we can obtain that
       \# \label{eq:3122}
       \mathbb{E}_{\boldsymbol{a}_1\sim\boldsymbol{\pi}_{\mathrm{RNE},1}(\cdot|s_1)}[\bar{\Delta}_{1, \boldsymbol{\Phi}}(s_1,\boldsymbol{a}_1)]\le \sqrt{\boldsymbol{C}^{\mathrm{RNE}}_{P^{\star},\mathbf{\Phi}} } \cdot \sum_{h = 1}^H \sqrt{\mathbf{Err}_{h}^{\mathbf{\Phi}}(n,\delta)},
       \# 
       which further implies that 
       \# \label{eq:3123}
       {\displaystyle \mathrm{(II)}} \le \sqrt{\boldsymbol{C}^{\mathrm{RNE}}_{P^{\star},\mathbf{\Phi}} } \cdot \sum_{h = 1}^H \sqrt{\mathbf{Err}_{h}^{\mathbf{\Phi}}(n,\delta)}.
       \#

       \paragraph{Combining Term (I) and Term (II).} Plugging \eqref{eq:3115} and \eqref{eq:3123} into \eqref{eq:3102} yields that
       \$
        \mathrm{RNEGap}_{\boldsymbol{\Phi}}(\hat{\boldsymbol{\pi}};s_1) \le 2 \sqrt{\boldsymbol{C}^{\mathrm{RNE}}_{P^{\star},\mathbf{\Phi}} } \cdot \sum_{h = 1}^H \sqrt{\mathbf{Err}_{h}^{\mathbf{\Phi}}(n,\delta)},
       \$ 
       which finishes the proof of Theorem~\ref{thm:offline:FA}.
\end{proof}

%% file: tex/appendix/sarmdp.tex
\section{Proofs for General RMDPs with $\mathcal{S}\times\mathcal{A}$-rectangular Robust Sets}\label{sec: proof sarmdp}

\subsection{Proof of Proposition \ref{prop: sarmdp estimation}}\label{subsec: proof prop sarmdp estimation}

\begin{lemma}[Duality for KL-robust set]\label{lem: kl}
    The following duality for KL-robust set holds,
    \begin{align*}
        \inf_{Q(\cdot):D_{\mathrm{KL}}(Q(\cdot)\|Q^{\star}(\cdot))\leq \sigma}\int f(x)Q(\mathrm{d}x) 
        = \sup_{\lambda\in\mathbb{R}_+} \left\{ - \lambda\log\left(\int \exp\left\{-f(x)/\lambda\right\} Q^{\star}(\mathrm{d}x)\right) -\lambda\sigma\right\}.
    \end{align*}
\end{lemma}
\begin{proof}[Proof of Lemma \ref{lem: kl}]
    See \cite{hu2013kullback, yang2021towards} for a detailed proof.
\end{proof}

\begin{assumption}[Regularity of KL-divergence duality variable]\label{ass: kl regularity}
    We assume that the optimal dual variable $\lambda^{\star}$ for the following optimization problem 
    \begin{align*}
        \sup_{\lambda\in\mathbb{R}_+}\left\{-\lambda\log\left(\mathbb{E}_{s'\sim P_h(\cdot|s_h,a_h)}\left[\exp\left\{-V_{h+1,Q,\mathbf{\Phi}}^{\pi^{\star}}(s')/\lambda\right\}\right]\right)-\lambda\rho\right\},
    \end{align*}
    is lower bounded by $\underline{\lambda}>0$ for any transition kernels $P_h\in\mathcal{P}_{\mathrm{M}}$, $Q = \{Q_h\}_{h=1}^H\subseteq\mathcal{P}_{\mathrm{M}}$,  and step $h\in[H]$.
\end{assumption}

\begin{lemma}[Duality for TV-robust set]\label{lem: tv}
    The following duality for TV-robust set holds,
    \begin{align*}
        \inf_{Q(\cdot):D_{\mathrm{TV}}(Q(\cdot)\|Q^{\star}(\cdot))\leq \sigma}\int f(x)Q(\mathrm{d}x) 
        = \sup_{\lambda\in\mathbb{R}} \left\{ - \int (\lambda-f(x))_+ Q^{\star}(\mathrm{d}x)- \frac{\sigma}{2}(\lambda-\inf_x f(x))_+ + \lambda \right\}.
    \end{align*}
\end{lemma}
\begin{proof}[Proof of Lemma \ref{lem: tv}]
    See \cite{yang2021towards} for a detailed proof.
\end{proof}

\begin{proof}[Proof of Proposition \ref{prop: sarmdp estimation} with KL-divergence]
    Firstly, by invoking the first conclusion of Lemma \ref{lem: mle guarantee sarmdp}, we know that the Condition \ref{cond: accuracy} holds.
    In the following, we prove the Condition \ref{cond: model estimation}.
    By applying the dual formulation of the KL-robust set (Lemma \ref{lem: kl}), we can derive that 
    \begin{align}
        &\inf_{\tilde{P}_h\in\mathbf{\Phi}(P_h^{\star})}\mathbb{E}_{s'\sim \tilde{P}_h(\cdot|s_h,a_h)}[V_{h+1, P, \mathbf{\Phi}}^{\pi^{\star}}(s')] - \inf_{\tilde{P}_h\in\mathbf{\Phi}(P_h)}\mathbb{E}_{s'\sim \tilde{P}_h(\cdot|s_h,a_h)}[V_{h+1, P, \mathbf{\Phi}}^{\pi^{\star}}(s')]\notag\\
        &\qquad = \sup_{\lambda\geq 0}\left\{-\lambda\log\left(\mathbb{E}_{s'\sim P_h^{\star}(\cdot|s_h,a_h)}\left[\exp\left\{-V_{h+1,P,\mathbf{\Phi}}^{\pi^{\star}}(s')/\lambda\right\}\right]\right)-\lambda\rho\right\}\notag\\
        &\qquad\qquad  - \sup_{\lambda\geq 0}\left\{-\lambda\log\left(\mathbb{E}_{s'\sim P_h(\cdot|s_h,a_h)}\left[\exp\left\{-V_{h+1,P,\mathbf{\Phi}}^{\pi^{\star}}(s')/\lambda\right\}\right]\right)-\lambda\rho\right\}.\label{eq: sketch term ii}
    \end{align}
    By Assumption \ref{ass: kl regularity} and Lemma \ref{lem: bound lambda kl}, we know that the optimal value of $\lambda$ for both two optimization problems in \eqref{eq: sketch term ii} lies in $[\underline{\lambda}, H/\rho]$ for some $\underline{\lambda}>0$. 
    Thus we can further upper bound the right hand side of \eqref{eq: sketch term ii} as 
    \begin{align}
        \eqref{eq: sketch term ii} & = \sup_{\underline{\lambda}\leq \lambda\leq H/\rho}\left\{-\lambda\log\left(\mathbb{E}_{s'\sim P_h^{\star}(\cdot|s_h,a_h)}\left[\exp\left\{-V_{h+1,P,\mathbf{\Phi}}^{\pi^{\star}}(s')/\lambda\right\}\right]\right)-\lambda\rho\right\}\notag\\
        &\qquad  - \sup_{\underline{\lambda}\leq \lambda\leq H/\rho}\left\{-\lambda\log\left(\mathbb{E}_{s'\sim P_h(\cdot|s_h,a_h)}\left[\exp\left\{-V_{h+1,P,\mathbf{\Phi}}^{\pi^{\star}}(s')/\lambda\right\}\right]\right)-\lambda\rho\right\}\notag\\
        &\leq \sup_{\underline{\lambda}\leq \lambda \leq H/\rho}\left\{\lambda\log\left(\frac{\mathbb{E}_{s'\sim P_h(\cdot|s_h,a_h)}\left[\exp\left\{-V_{h+1,P,\mathbf{\Phi}}^{\pi^{\star}}(s')/\lambda\right\}\right]}{\mathbb{E}_{s'\sim P_h^{\star}(\cdot|s_h,a_h)}\left[\exp\left\{- V_{h+1,P,\mathbf{\Phi}}^{\pi^{\star}}(s')/\lambda\right\}\right]}\right)\right\}\label{eq: sketch term ii 2},
    \end{align}
    where in the second inequality we use the basic fact that $\sup_{x}f(x) - \sup_{x}g(x) \leq \sup_{x}\{f(x) - g(x)\}$.
    Now we work on the right hand side of \eqref{eq: sketch term ii 2} and obtain that 
    \begin{align}
        \eqref{eq: sketch term ii 2} &= 
        \sup_{\underline{\lambda}\leq \lambda \leq H/\rho}\left\{\lambda\log\left(1+\frac{\left(\mathbb{E}_{s'\sim P_h(\cdot|s_h,a_h)} - \mathbb{E}_{s'\sim P_h^{\star}(\cdot|s_h,a_h)}\right)\left[\exp\left\{-V_{h+1,P,\mathbf{\Phi}}^{\pi^{\star}}(s')/\lambda\right\}\right]}{\mathbb{E}_{s'\sim P_h^{\star}(\cdot|s_h,a_h)}\left[\exp\left\{- V_{h+1,P,\mathbf{\Phi}}^{\pi^{\star}}(s')/\lambda\right\}\right]}\right)\right\}\notag\\
        &\leq \sup_{\underline{\lambda}\leq \lambda \leq H/\rho}\left\{\lambda\cdot\frac{\left(\mathbb{E}_{s'\sim P_h(\cdot|s_h,a_h)} - \mathbb{E}_{s'\sim P_h^{\star}(\cdot|s_h,a_h)}\right) \left[\exp\left\{-V_{h+1,P,\mathbf{\Phi}}^{\pi^{\star}}(s')/\lambda\right\}\right]}{\mathbb{E}_{s'\sim P_h^{\star}(\cdot|s_h,a_h)}\left[\exp\left\{-V_{h+1,P,\mathbf{\Phi}}^{\pi^{\star}}(s')/\lambda\right\}\right]}\right\}, \label{eq: sketch term ii 3}
    \end{align}
    where we use the fact of $\log(1+x)\leq x$ in the second inequality. 
    Now we can further bound the right hand side of \eqref{eq: sketch term ii 3} by 
    \begin{align}
        \eqref{eq: sketch term ii 3} &\leq \frac{H\exp(H/\underline{\lambda})}{\rho}\cdot\left|\left(\mathbb{E}_{s'\sim P_h(\cdot|s_h,a_h)} - \mathbb{E}_{s'\sim P_h^{\star}(\cdot|s_h,a_h)}\right) \left[\exp\left\{-V_{h+1,P,\mathbf{\Phi}}^{\pi}(s')/\lambda\right\}\right]\right|\notag\\
            & \leq \frac{H\exp(H/\underline{\lambda})}{\rho}\cdot\int_{\mathcal{S}}|P_h(\mathrm{d}s'|s_h,a_h) - P_h^{\star}(\mathrm{d}s'|s_h,a_h)|\notag\\
            & = \frac{H\exp(H/\underline{\lambda})}{\rho}\cdot\|P_h(\cdot|s_h,a_h) - P_h^{\star}(\cdot|s_h,a_h)\|_{\mathrm{TV}}.\label{eq: sketch term ii 4}
    \end{align} 
    Thus by combining \eqref{eq: sketch term ii}, \eqref{eq: sketch term ii 2}, \eqref{eq: sketch term ii 3}, and \eqref{eq: sketch term ii 4} we obtain that 
    \begin{align}
        &\inf_{\tilde{P}_h\in\mathbf{\Phi}(P_h^{\star})}\mathbb{E}_{s'\sim \tilde{P}_h(\cdot|s_h,a_h)}[V_{h+1, P, \mathbf{\Phi}}^{\pi^{\star}}(s')] - \inf_{\tilde{P}_h\in\mathbf{\Phi}(P_h)}\mathbb{E}_{s'\sim \tilde{P}_h(\cdot|s_h,a_h)}[V_{h+1, P, \mathbf{\Phi}}^{\pi^{\star}}(s')]\notag\\
        &\qquad \leq \frac{H\exp(H/\underline{\lambda})}{\rho}\cdot\|P_h(\cdot|s_h,a_h) - P_h^{\star}(\cdot|s_h,a_h)\|_{\mathrm{TV}}.\label{eq: sketch term ii 5}
    \end{align}
    By using a same argument for deriving \eqref{eq: sketch term ii 5}, we can also obtain that 
    \begin{align}
        &\inf_{\tilde{P}_h\in\mathbf{\Phi}(P_h)}\mathbb{E}_{s'\sim \tilde{P}_h(\cdot|s_h,a_h)}[V_{h+1, P, \mathbf{\Phi}}^{\pi^{\star}}(s')] - \inf_{\tilde{P}_h\in\mathbf{\Phi}(P_h^{\star})}\mathbb{E}_{s'\sim \tilde{P}_h(\cdot|s_h,a_h)}[V_{h+1, P, \mathbf{\Phi}}^{\pi^{\star}}(s')]\notag\\
        &\qquad \leq \frac{H\exp(H/\underline{\lambda})}{\rho}\cdot\|P_h(\cdot|s_h,a_h) - P_h^{\star}(\cdot|s_h,a_h)\|_{\mathrm{TV}}.\label{eq: sketch term ii 6}
    \end{align}
    Therefore, due to \eqref{eq: sketch term ii 5} and \eqref{eq: sketch term ii 6}, we can finally arrive at the following upper bound,
    \begin{align}
        &\mathbb{E}_{(s_h,a_h)\sim d^{\pi^\mathrm{b}}_{P^{\star},h}}\left[\left(\inf_{\tilde{P}_h\in\mathbf{\Phi}(P_h)}\mathbb{E}_{s'\sim \tilde{P}_h(\cdot|s_h,a_h)}[V_{h+1, P, \mathbf{\Phi}}^{\pi^{\star}}(s')] - \inf_{\tilde{P}_h\in\mathbf{\Phi}(P^{\star}_h)}\mathbb{E}_{s'\sim \tilde{P}_h(\cdot|s_h,a_h)}[V_{h+1, P, \mathbf{\Phi}}^{\pi^{\star}}(s')]\right)^2\right]\notag \\
        &\qquad \leq \frac{H^2\exp(2H/\underline{\lambda})}{\rho^2}\cdot \mathbb{E}_{(s_h,a_h)\sim d^{\pi^\mathrm{b}}_{P^{\star},h}}[\|P_h(\cdot|s_h,a_h) - P_h^{\star}(\cdot|s_h,a_h)\|_{\mathrm{TV}}^2]\label{eq: sketch term ii 7}.
    \end{align}
    By invoking the second conclusion of Lemma \ref{lem: mle guarantee sarmdp}, we have that with probability at least $1-\delta$, 
    \begin{align}\label{eq: sketch term ii 8}
        \mathbb{E}_{(s_h,a_h)\sim d^{\pi^\mathrm{b}}_{P^{\star},h}}[\|P_h(\cdot|s_h,a_h) - P_h^{\star}(\cdot|s_h,a_h)\|_{\mathrm{TV}}^2] \leq \frac{C_1'\log(C_2'H\mathcal{N}_{[]}(1/n^2,\mathcal{P}_{\mathrm{M}},\|\cdot\|_{1,\infty})/\delta)}{n},
    \end{align}
    for some absolute constant $C_1',C_2'>0$. 
    Now combining \eqref{eq: sketch term ii 7} and \eqref{eq: sketch term ii 8}, we have that
    \begin{align*}
        \sqrt{\mathrm{Err}_{h,\mathrm{KL}}^{\mathbf{\Phi}}(n)} = \frac{H\exp(H/\underline{\lambda})}{\rho}\cdot\sqrt{\frac{C_1'\log(C_2'H\mathcal{N}_{[]}(1/n^2,\mathcal{P}_{\mathrm{M}},\|\cdot\|_{1,\infty})/\delta)}{n}}.
    \end{align*}
    This finishes the proof of Proposition \ref{prop: sarmdp estimation} under KL-divergence.
\end{proof}

\begin{proof}[Proof of Proposition \ref{prop: sarmdp estimation} with TV-distance]
    Firstly, by invoking the first conclusion of Lemma \ref{lem: mle guarantee sarmdp}, we know that the Condition \ref{cond: accuracy} holds.
    In the following, we prove the Condition \ref{cond: model estimation}.
    By applying the dual formulation of the TV-robust set (Lemma \ref{lem: tv}), we can similarly derive that 
    \begin{align}
        &\left|\inf_{\tilde{P}_h\in\mathbf{\Phi}(P_h^{\star})}\mathbb{E}_{s'\sim \tilde{P}_h(\cdot|s_h,a_h)}[V_{h+1, P, \mathbf{\Phi}}^{\pi^{\star}}(s')] - \inf_{\tilde{P}_h\in\mathbf{\Phi}(P_h)}\mathbb{E}_{s'\sim \tilde{P}_h(\cdot|s_h,a_h)}[V_{h+1, P, \mathbf{\Phi}}^{\pi^{\star}}(s')]\right|\notag\\
        &\qquad = \Bigg|\sup_{\lambda\in\mathbb{R}}\left\{-\mathbb{E}_{s'\sim P_h^{\star}(\cdot|s_h,a_h)}\left[\left(\lambda - V_{h+1,P,\mathbf{\Phi}}^{\pi^{\star}}(s')\right)_+\right] - \frac{\rho}{2}\left(\lambda - \inf_{s''\in\mathcal{S}}V_{h+1,P,\mathbf{\Phi}}^{\pi^{\star}}(s'')\right)+\lambda\right\}\notag\\
        &\qquad\qquad - \sup_{\lambda\in\mathbb{R}}\left\{-\mathbb{E}_{s'\sim P_h(\cdot|s_h,a_h)}\left[\left(\lambda - V_{h+1,P,\mathbf{\Phi}}^{\pi^{\star}}(s')\right)_+\right] - \frac{\rho}{2}\left(\lambda - \inf_{s''\in\mathcal{S}}V_{h+1,P,\mathbf{\Phi}}^{\pi^{\star}}(s'')\right)+\lambda\right\}\Bigg|\label{eq: sketch tv term ii -}\\
        &\qquad \leq \left|\sup_{\lambda\in\mathbb{R}}\left\{\left(\mathbb{E}_{s'\sim P_h^{\star}(\cdot|s_h,a_h)} - \mathbb{E}_{s'\sim P_h(\cdot|s_h,a_h)}\right)\left[\left(\lambda - V_{h+1,P,\mathbf{\Phi}}^{\pi^{\star}}(s')\right)_+\right]\right\}\right|\label{eq: sketch tv term ii}
    \end{align}
    As is shown in Lemma \ref{lem: bound lambda tv}, the optimal value of $\lambda$ for both two optimization problems in \eqref{eq: sketch tv term ii -} lies in $[0,H]$.
    Thus we can further upper bound the right hand side of \eqref{eq: sketch tv term ii} as 
    \begin{align}
        \eqref{eq: sketch tv term ii} \leq H\cdot\|P_h(\cdot|s_h,a_h) - P_h^{\star}(\cdot|s_h,a_h)\|_{\mathrm{TV}}.\label{eq: sketch tv term ii 2}
    \end{align}
    By applying the second conclusion of Lemma \ref{lem: mle guarantee sarmdp}, we conclude that with probability at least $1-\delta$, 
    \begin{align}
        &\mathbb{E}_{(s_h,a_h)\sim d^{\pi^\mathrm{b}}_{P^{\star},h}}\left[\left(\inf_{\tilde{P}_h\in\mathbf{\Phi}(P_h)}\mathbb{E}_{s'\sim \tilde{P}_h(\cdot|s_h,a_h)}[V_{h+1, P, \mathbf{\Phi}}^{\pi^{\star}}(s')] - \inf_{\tilde{P}_h\in\mathbf{\Phi}(P^{\star}_h)}\mathbb{E}_{s'\sim \tilde{P}_h(\cdot|s_h,a_h)}[V_{h+1, P, \mathbf{\Phi}}^{\pi^{\star}}(s')]\right)^2\right]\notag \\
        &\qquad \leq H^2\cdot \mathbb{E}_{(s_h,a_h)\sim d^{\pi^\mathrm{b}}_{P^{\star},h}}[\|P_h(\cdot|s_h,a_h) - P_h^{\star}(\cdot|s_h,a_h)\|_{\mathrm{TV}}^2]\notag\\
        &\qquad \leq \frac{C_1'H^2\log(C_2'H\mathcal{N}_{[]}(1/n^2,\mathcal{P}_{\mathrm{M}},\|\cdot\|_{1,\infty})/\delta)}{n}.
    \end{align}
    Therefore, it suffices to choose $\mathrm{Err}_{h,\mathrm{TV}}^{\mathbf{\Phi}}(n)$ as 
    \begin{align*}
        \sqrt{\mathrm{Err}_{h,\mathrm{TV}}^{\mathbf{\Phi}}(n)} = H\cdot\sqrt{\frac{C_1'\log(C_2'H\mathcal{N}_{[]}(1/n^2,\mathcal{P}_{\mathrm{M}},\|\cdot\|_{1,\infty})/\delta)}{n}}.
    \end{align*}
    This finishes the proof of Proposition \ref{prop: sarmdp estimation} under TV-distance.
\end{proof}

\subsection{Proofs for $\mathcal{S}\times\mathcal{A}$-rectangular Robust Tabular MDP (Equation \eqref{eq: bracket number tabular})}\label{subsec: proof sartmdp}

The model class $\mathcal{P}_{\mathrm{M}}$ can be considered as a subspace of $\mathcal{F} = \{f(s,a,s'):\|f\|_{\infty}\leq 1\}$ with finite $\mathcal{S}$ and $\mathcal{A}$.
Consider the collection of brackets $\mathcal{B}$ containing brackets in the form of $[g,g+1/n^2]$, where $g(s,a,s')\in\{0,1/n^2,2/n^2,\cdots,(n^2-1)/n^2\}$.
Then we can see that $\mathcal{B}$ is actually a $1/n^2$-bracket of $\mathcal{F}$.
Thus we know that the bracket number of $\mathcal{P}_{\mathrm{M}}$ is bounded by,
\begin{align*}
    \mathcal{N}_{[]}(1/n^2,\mathcal{P}_{\mathrm{M}},\|\cdot\|_{1,\infty}) \leq \mathcal{N}_{[]}(1/n^2,\mathcal{F}_{\mathrm{M}},\|\cdot\|_{\infty}) \leq |\mathcal{B}| \leq n^{2|\mathcal{S}|^2|\mathcal{A}|}.
\end{align*}
This finishes the proof of \eqref{eq: bracket number tabular}.

\subsection{Proofs for $\mathcal{S}\times\mathcal{A}$-rectangular Robust MDPs with Kernel Function Approximations}\label{subsec: proof sarmdp kernel}

\subsubsection{A Basic Review of Reproducing Kernel Hilbert Space}\label{subsubsec: rkhs}

We briefly review the basic knowledge of a reproducing kernel Hilbert space (RKHS). 
We say $\mathcal{H}$ is a RKHS
on a set $\mathcal{Y}$ with the reproducing kernel $\mathcal{K}:\mathcal{Y}\times\mathcal{Y}\rightarrow\mathbb{R}$ 
if its inner product $\langle\cdot,\cdot \rangle_\mathcal{H}$ satisfies, 
for any $f\in\mathcal{H}$ and $y\in\mathcal{Y}$, we have that $f(y)=\langle f, \mathcal{K}(y,\cdot)\rangle_\mathcal{H}$. 
The mapping $\mathcal{K}(y,\cdot):\mathcal{Y}\mapsto\mathcal{H}$ is called the feature mapping of $\mathcal{H}$, denoted by $\boldsymbol{\psi}(y):\mathcal{Y}\mapsto\mathcal{H}$.

When the reproducing kernel $\mathcal{K}$ is continuous, symmetric, 
and positive definite, Mercer's theorem \citep{steinwart2008support} says that
$\mathcal{K}$ has the following representation,
\begin{align*}
    \mathcal{K}(x,y) = \sum_{j=1}^\infty \lambda_j \psi_j(x)\psi_j(y),\quad  \forall  x,y\in\mathcal{Y},
\end{align*} 
%\begin{equation}\label{equ:ortho_basis}
%    \mathcal{K}(x,y) = \sum_{j=1}^\infty \lambda_j \psi_j(x)\psi_j(y), ~\text{for any}~x,y\in\mathcal{Y}.
%\end{equation}
where $\psi_j:\mathcal{Y}\mapsto\mathbb{R}$ and $\{\sqrt{\lambda_j}\cdot\psi_j\}_{j=1}^\infty$ forms an orthonormal basis of $\mathcal{H}$ with 
$\lambda_1\geq\lambda_2\geq\cdots\geq0$.
Also, the feature mapping $\boldsymbol{\psi}(y)$ can be represented as
\begin{align*}
    \boldsymbol{\psi}(y) = \sum_{j=1}^{+\infty}\lambda_j\psi_j(y)\psi_j,\quad\forall y\in\mathcal{Y}.
\end{align*}

\subsubsection{Proof of Equation \eqref{eq: bracket number kernel}}\label{subsubsec: proof bracket kernel}

We invoke the following lemma to bound the bracket number of $\mathcal{P}_{\mathrm{M}}$ in Example \ref{exp: sarmdp kernel}.

\begin{lemma}[Bracket number of kernel function class \citep{liu2022welfare}]\label{lem: bracket kernel}
    Under Assumption \ref{ass: rkhs}, the bracket number of $\mathcal{P}_{\mathrm{M}}$ given by 
    \begin{align*}
        \mathcal{P}_{\mathrm{M}} = \big\{P(s'|s,a) = \langle\boldsymbol{\psi}(s,a,s'), \boldsymbol{f}\rangle_{\mathcal{H}}:\boldsymbol{f}\in\mathcal{H},\|f\|_{\mathcal{H}}\leq B_{\mathrm{K}} \big\}
    \end{align*}
    is bounded by, for any $\epsilon>0$,
    \begin{align*}
        \log(\mathcal{N}_{[]}(\epsilon,\mathcal{P}_{\mathrm{M}},\|\cdot\|_{1,\infty})) \leq C_{\mathrm{K}}\cdot 1/\gamma\cdot\log^2(1/\gamma)\cdot\log^{1+1/\gamma}(\mathtt{Vol}(\mathcal{S})B_{\mathrm{K}}/\epsilon).
    \end{align*}
\end{lemma}
\begin{proof}[Proof of Lemma \ref{lem: bracket kernel}]
    We refer to Lemma B.11 in \cite{liu2022welfare} for a detailed proof.
\end{proof}

By taking $\epsilon = 1/n^2$ in Lemma \ref{lem: bracket kernel}, we can finish the proof of \eqref{eq: bracket number kernel}.

\subsection{Proofs for $\mathcal{S}\times\mathcal{A}$-rectangular Robust MDPs with Neural Function Approximations}\label{subsec: proof sarmdp neural}

\subsubsection{Neural Tangent Kernel and Implicit Linearization}\label{subsubec: ntk and linearization}

We consider the overparameterized paradigm of the neural network \eqref{eq: nn} in the sense that the neural network is very wide, i.e., the number of hidden units $m$ is very large.
The following lemma shows that in this paradigm, neural networks in  $\mathcal{P}_{\mathrm{M}}$ are well approximated by a linear expansion at initialization.

\begin{lemma}[Implicit Linearization \citep{cai2020optimistic}]\label{lem: implicit linearization}
    Consider the two-layer neural network $\mathrm{NN}$ defined in \eqref{eq: nn}.
    Assuming that the activation function $\sigma(\cdot)$ is $1$-Lipschitz continuous and the input space $\mathcal{X}$ is normalized via $\|\mathbf{x}\|_2\leq 1$ for any $\mathbf{x}\in\mathcal{X}$.
    Then it holds that 
    \begin{align*}
        \sup_{\mathbf{x}\in\mathcal{X},\mathrm{NN}(\cdot;\mathbf{W},\mathbf{a}^0)\in\mathcal{P}_{\mathrm{M}}}\left|\mathrm{NN}(\mathbf{x};\mathbf{W},\mathbf{a}^0) - \nabla_{\mathbf{W}}\mathrm{NN}(\mathbf{x};\mathbf{W}^0,\mathbf{a}^0)^\top (\mathbf{W} - \mathbf{W}^0)\right| \leq d_{\mathcal{X}}^{1/2}B_{\mathrm{N}}^2m^{-1/2}.
    \end{align*}
\end{lemma}
\begin{proof}[Proof of Lemma \ref{lem: implicit linearization}]
    See the proof of Lemma 4.5 in \cite{cai2020optimistic} for a detailed proof.
\end{proof}

In view of Lemma \ref{lem: implicit linearization}, we can study the linearization of the neural networks in $\mathcal{P}_{\mathrm{M}}$ as a surrogate.
To this end, we introduce the neural tangent kernel $\mathcal{K}_{\mathrm{NTK}}$ of $\mathrm{NN}$ as 
\begin{align*}
    \mathcal{K}_{\mathrm{NTK}}(x,y) := \nabla_{\mathbf{W}}\mathrm{NN}(x,\mathbf{W}^0,\mathbf{a}^0)^\top\nabla_{\mathbf{W}}\mathrm{NN}(y,\mathbf{W}^0,\mathbf{a}^0),\quad \forall x,y\in\mathcal{X}.
\end{align*}
The idea is to approximate the functions in $\mathcal{P}_{\mathrm{M}}$ via the RKHS induced by the kernel $\mathcal{K}_{\mathrm{NTK}}$.
According to Lemma \ref{lem: implicit linearization}, when the width of the neural network is large enough, i.e., $m\rightarrow\infty$, the approximation error is negligible. 
See the following Section \ref{subsubsec: proof bracket neural} for detailed proofs.

\subsubsection{Proof of Equation \eqref{eq: bracket number neural}}\label{subsubsec: proof bracket neural}

Now we use Lemma \ref{lem: implicit linearization} to bound the bracket number of $\mathcal{P}_{\mathrm{M}}$ in Example \ref{exp: sarmdp neural}.

\begin{lemma}[Bracket number of neural function class]\label{lem: bracket neural}
    Under Assumption \ref{ass: ntk regularity}, for the number of hidden units $m\geq d_{\mathcal{X}}B_{\mathrm{N}}^4/\epsilon^2$, the bracket number of $\mathcal{P}_{\mathrm{M}}$ given by 
    \begin{align*}
        \mathcal{P}_{\mathrm{M}} = \left\{P(s'|s,a) = \mathrm{NN}((s,a,s');\mathbf{W},\mathbf{a}^{0}): \|\mathbf{W} - \mathbf{W}^{\mathrm{0}}\|_2\leq B_{\mathrm{N}} \right\},
    \end{align*}
    is bounded by, for any $\epsilon>0$,
    \begin{align*}
        \log(\mathcal{N}_{[]}(\epsilon,\mathcal{P}_{\mathrm{M}},\|\cdot\|_{1,\infty})) \leq C_{\mathrm{N}}\cdot 1/\gamma_{\mathrm{N}}\cdot\log^2(1/\gamma_{\mathrm{N}})\cdot\log^{1+1/\gamma_{\mathrm{N}}}(\mathtt{Vol}(\mathcal{S})B_{\mathrm{K}}/\epsilon).
    \end{align*}
\end{lemma}

\begin{proof}[Proof of Lemma \ref{lem: bracket neural}]
    We denote the RKHS induced by the neural tangent kernel $\mathcal{K}_{\mathrm{NTK}}$ as $\mathcal{P}_{\mathrm{NTK}}$
    \begin{align}
        \mathcal{P}_{\mathrm{NTK}} = \left\{\bar{P}(\mathbf{x}) = \nabla_{\mathbf{W}}\mathrm{NN}(\mathbf{x};\mathbf{W}^0,\mathbf{a}^0)^\top (\mathbf{W} - \mathbf{W}^0):\|\mathbf{W} - \mathbf{W}^0\|_{2}\leq B_{\mathrm{N}}\right\}.
        % \mathcal{P}_{\mathrm{NTK}} = \left\{\bar{P}(s'|s,a) = \langle \boldsymbol{\psi}_{\mathrm{NTK}}(s,a,s'),\boldsymbol{f}\rangle_{\mathcal{H}_{\mathrm{NTK}}}:\boldsymbol{f}\in\mathcal{H}_{\mathrm{NTK}},\|\boldsymbol{f}\|_{\mathcal{H}_{\mathrm{NTK}}}\leq B_{\mathrm{N}}\right\}
    \end{align}
    For any $\mathrm{NN}(\cdot;\mathbf{W},\mathbf{a}^0)\in\mathcal{P}_{\mathrm{M}}$, we denote its linear expansion at initialization as $\overline{\mathrm{NN}}(\cdot;\mathbf{W},\mathrm{a}^0)\in\mathcal{P}_{\mathrm{NTK}}$.
    Here we use the fact that for $\mathrm{NN}(\cdot;\mathbf{W},\mathbf{a}^0)\in\mathcal{P}_{\mathrm{M}}$, $\|\mathbf{W} - \mathbf{W}^0\|_{2}\leq B_{\mathrm{N}}$.
    Now according to Lemma \ref{lem: bracket kernel} and Assumption \ref{ass: ntk regularity}, we know that the bracket number of $\mathcal{P}_{\mathrm{NTK}}$ is bounded by 
    \begin{align}\label{eq: proof bracket number neural 1}
        \log(\mathcal{N}_{[]}(\epsilon,\mathcal{P}_{\mathrm{NTK}},\|\cdot\|_{1,\infty})) \leq C\cdot 1/\gamma_\mathrm{N}\cdot\log^2(1/\gamma_\mathrm{N})\cdot\log^{1+1/\gamma_\mathrm{N}}(\mathtt{Vol}(\mathcal{S})B_{\mathrm{N}}/\epsilon),
    \end{align}
    for some constant $C>0$.
    Therefore, we can find a collect of brackets $\mathcal{B}_0 = \{[g_j^{\mathrm{l}},g_j^\mathrm{u}]\}_{j\in [\mathcal{N}_{[]}(\epsilon,\mathcal{P}_{\mathrm{NTK}},\|\cdot\|_{1,\infty})]}$ such that for any $\bar{P}\in\mathcal{P}_{\mathrm{NTK}}$, there exists a bracket $[g_j^{\mathrm{l}},g_j^\mathrm{u}]\in\mathcal{B}_0$ such that $g_j^{\mathrm{l}}(\mathbf{x})\leq \bar{P}(\mathbf{x})\leq g_j^\mathrm{u}(\mathbf{x})$ and $\|g_j^{\mathrm{l}}-g_j^\mathrm{u}\|_{1,\infty}\leq \epsilon$.
    Now for any $P = \mathrm{NN}(\cdot;\mathbf{W},\mathbf{a}^0)\in\mathcal{P}_{\mathrm{M}}$, by Lemma \ref{lem: implicit linearization}, we have that
    \begin{align*}
        \overline{\mathrm{NN}}(\mathbf{x};\mathbf{W},\mathbf{a}^0) - \epsilon_{\mathrm{N}} \leq \mathrm{NN}(\mathbf{x};\mathbf{W},\mathbf{a}^0)\leq \overline{\mathrm{NN}}(\mathbf{x};\mathbf{W},\mathbf{a}^0) + \epsilon_{\mathrm{N}}, 
    \end{align*}
    where $\epsilon_{\mathrm{N}} = d_{\mathcal{X}}^{1/2}B_{\mathrm{N}}^2m^{-1/2}$.
    By previous arguments, there exists a bracket $[g_j^{\mathrm{l}},g_j^\mathrm{u}]\in\mathcal{B}_0$ such that
    \begin{align*}
         g_j^{\mathrm{l}}(\mathbf{x}) - \epsilon_{\mathrm{N}} \leq \mathrm{NN}(\mathbf{x};\mathbf{W},\mathbf{a}^0)\leq g_j^{\mathrm{u}}(\mathbf{x}) + \epsilon_{\mathrm{N}}.
    \end{align*}
    Now it suffices to define a new collect of brackets $\mathcal{B} = \{[g_j^{\mathrm{l}} - \epsilon_{\mathrm{N}}, g_j^\mathrm{u}+ \epsilon_{\mathrm{N}}]\}_{j\in [\mathcal{N}_{[]}(\epsilon,\mathcal{P}_{\mathrm{NTK}},\|\cdot\|_{1,\infty})]}$.
    For any $P = \mathrm{NN}(\cdot;\mathbf{W},\mathbf{a}^0)\in\mathcal{P}_{\mathrm{M}}$, there exists a bracket $[\tilde{g}_j^{\mathrm{l}},\tilde{g}_j^{\mathrm{u}}]\in\mathcal{B}$ such that $\tilde g_j^{\mathrm{l}}(\mathbf{x})\leq P(\mathbf{x})\leq \tilde g_j^\mathrm{u}(\mathbf{x})$, and 
    \begin{align*}
        \|\tilde{g}_j^{\mathrm{l}}(\mathbf{x}) - \tilde{g}_j^{\mathrm{u}}(\mathbf{x})\|_{1,\infty}\leq \|g_j^{\mathrm{l}}(\mathbf{x}) - g_j^{\mathrm{u}}(\mathbf{x})\|_{\infty} + 2\epsilon_{\mathrm{N}} \leq \epsilon + 2\epsilon_{\mathrm{N}}.
    \end{align*}
    By taking $m \geq d_{\mathcal{X}}B^4_{\mathrm{N}} / \epsilon^2$, we obtain that $\|\tilde{g}_j^{\mathrm{l}}(\mathbf{x}) - \tilde{g}_j^{\mathrm{u}}(\mathbf{x})\|_{1,\infty}\leq 3\epsilon$. 
    Therefore, we can conclude that the bracket number of $\mathcal{P}_{\mathrm{M}}$ is bounded by,
    \begin{align}\label{eq: proof bracket number neural 2}
        \mathcal{N}_{[]}(\epsilon,\mathcal{P}_{\mathrm{M}},\|\cdot\|_{1,\infty}) =  \mathcal{N}_{[]}(\epsilon/3,\mathcal{P}_{\mathrm{NTK}},\|\cdot\|_{1,\infty}).
    \end{align}
    Finally, by combining \eqref{eq: proof bracket number neural 1} and \eqref{eq: proof bracket number neural 2}, we have that, for $m \geq d_{\mathcal{X}}B^4_{\mathrm{N}} / \epsilon^2$,
    \begin{align*}
        \log(\mathcal{N}_{[]}(\epsilon,\mathcal{P}_{\mathrm{M}},\|\cdot\|_{1,\infty})) \leq C_{\mathrm{N}}\cdot 1/\gamma_\mathrm{N}\cdot\log^2(1/\gamma_\mathrm{N})\cdot\log^{1+1/\gamma_\mathrm{N}}(\mathtt{Vol}(\mathcal{S})B_{\mathrm{N}}/\epsilon),
    \end{align*}
    for some constant $C_{\mathrm{N}}>0$. 
    This finishes the proof of Lemma \ref{lem: bracket neural}.
\end{proof}

Now by taking $\epsilon = 1/n^2$, i.e., $m\geq d_{\mathrm{X}}n^4B_{\mathrm{N}}^4$, we can derive the desired result in \eqref{eq: bracket number neural}.

%% file: tex/appendix/safrmdp.tex
\section{Proofs for $\mathcal{S}\times\mathcal{A}$-rectangular Robust Factored MDPs}\label{sec: proof factored rmdp}

\subsection{Proof of Proposition \ref{prop: safrmdp estimation}}\label{subsec: proof prop safrmdp estimation}

\begin{assumption}[Regularity of KL-divergence duality variable]\label{ass: kl regularity frmdp}
    We assume that the optimal dual variable $\lambda^{\star}$ for the following optimization problem 
    \begin{align*}
        \sup_{\lambda\in\mathbb{R}_+}\left\{-\lambda\log\left(\mathbb{E}_{s'[j]\sim P_{h,j}(\cdot|s_h[\mathrm{pa}_j],a_h)}\left[\exp\left\{-\upsilon^j_{h,T,Q,\mathbf{\Phi}}(s'[j])/\lambda\right\}\right]\right)-\lambda\rho\right\},
    \end{align*}
    is lower bounded by $\underline{\lambda}>0$ for any transition kernel $P_h\in\mathcal{P}_{\mathrm{M}}$, $T = \{T_h\}_{h=1}^H\subseteq\mathcal{P}_{\mathrm{M}}$, $Q=\{Q_h\}_{h=1}^H\subseteq\mathcal{P}_{\mathrm{M}}$, step $h\in[H]$, and factor $j\in[d]$.
    Here the function $\upsilon^j_{h,T,Q,\mathbf{\Phi}}(s'[j])$ is defined as 
    \begin{align*}
        \upsilon_{h,T,Q,\boldsymbol{\Phi}}^j(s'[j]) =  \int_{\mathcal{O}^{d-1}}\prod_{\substack{i=1\\i\neq j}}^dT_{h,i}(\mathrm{d}s'[i])V_{h+1, Q, \mathbf{\Phi}}^{\pi^{\star}}(s'[1],\cdots,s'[j-1],s[j],s'[j+1],\cdots,s'[d]).
    \end{align*}
\end{assumption}

\begin{proof}[Proof of Proposition \ref{prop: safrmdp estimation} with KL-divergence]
    %Consider the following decomposition of our target,
    Firstly, by invoking the first conclusion of Lemma \ref{lem: mle guarantee safrmdp}, we know that the Condition \ref{cond: accuracy} holds.
    In the following, we prove the Condition \ref{cond: model estimation}.
    By the definition of robust set in Example \ref{exp: safrmdp},
    \begin{align}
        & \inf_{\tilde{P}_h\in\mathbf{\Phi}(P_h)}\mathbb{E}_{s'\sim \tilde{P}_h(\cdot|s_h,a_h)}[V_{h+1, P, \mathbf{\Phi}}^{\pi^{\star}}(s')] - \inf_{\tilde{P}_h\in\mathbf{\Phi}(P_h^{\star})}\mathbb{E}_{s'\sim \tilde{P}_h(\cdot|s_h,a_h)}[V_{h+1, P, \mathbf{\Phi}}^{\pi^{\star}}(s')]\notag\\
        & \qquad = \inf_{\tilde{P}_{h,i}\in\Delta(\mathcal{O}):D_{\mathrm{KL}}(\tilde{P}_{h,i}(\cdot)\|P_{h,i}(\cdot|s_h[\mathrm{pa}_i],a_h))\leq \rho_i, i\in[d]}\int_{\mathcal{O}^d}\prod_{i=1}^d\tilde{P}_{h,i}(\mathrm{d}s'[i])V_{h+1, P, \mathbf{\Phi}}^{\pi^{\star}}(s')\notag\\ 
        & \qquad\qquad - \inf_{\tilde{P}_{h,i}\in\Delta(\mathcal{O}):D_{\mathrm{KL}}(\tilde{P}_{h,i}(\cdot)\|P_{h,i}^{\star}(\cdot|s_h[\mathrm{pa}_i],a_h))\leq \rho_i, i\in[d]}\int_{\mathcal{O}^d}\prod_{i=1}^d\tilde{P}_{h,i}(\mathrm{d}s'[i])V_{h+1, P, \mathbf{\Phi}}^{\pi^{\star}}(s').\label{eq: safrmdp estimation proof 1}
    \end{align}
    Consider the following decomposition of the right hand side of \eqref{eq: safrmdp estimation proof 1},
    \begin{align}
        \eqref{eq: safrmdp estimation proof 1} &= \sum_{j=1}^d \inf_{\substack{\tilde{P}_{h,i}\in\Delta(\mathcal{O}):D_{\mathrm{KL}}(\tilde{P}_{h,i}(\cdot)\|P_{h,i}(\cdot|s_h[\mathrm{pa}_i],a_h))\leq \rho_i, 1\leq i\leq j \\ \tilde{P}_{h,i}\in\Delta(\mathcal{O}):D_{\mathrm{KL}}(\tilde{P}_{h,i}(\cdot)\|P_{h,i}^{\star}(\cdot|s_h[\mathrm{pa}_i],a_h))\leq \rho_i, j+1\leq i \leq d}}\int_{\mathcal{O}^d}\prod_{i=1}^d\tilde{P}_{h,i}(\mathrm{d}s'[i])V_{h+1, P, \mathbf{\Phi}}^{\pi^{\star}}(s')\notag\\ 
        &\qquad - \inf_{\substack{\tilde{P}_{h,i}\in\Delta(\mathcal{O}):D_{\mathrm{KL}}(\tilde{P}_{h,i}(\cdot)\|P_{h,i}(\cdot|s_h[\mathrm{pa}_i],a_h))\leq \rho_i, 1\leq i\leq j-1 \\ \tilde{P}_{h,i}\in\Delta(\mathcal{O}):D_{\mathrm{KL}}(\tilde{P}_{h,i}(\cdot)\|P_{h,i}^{\star}(\cdot|s_h[\mathrm{pa}_i],a_h))\leq \rho_i, j\leq i \leq d}}\int_{\mathcal{O}^d}\prod_{i=1}^d\tilde{P}_{h,i}(\mathrm{d}s'[i])V_{h+1, P, \mathbf{\Phi}}^{\pi^{\star}}(s').\notag
    \end{align}
    For each $1\leq j\leq d$, we denote that 
    \begin{align*}
        (\tilde{P}_{h,1}^{\ast,j},\cdots,\tilde{P}_{h,d}^{\ast,j}) = \arginf_{\substack{\tilde{P}_{h,i}\in\Delta(\mathcal{O}):D_{\mathrm{KL}}(\tilde{P}_{h,i}(\cdot)\|P_{h,i}(\cdot|s_h[\mathrm{pa}_i],a_h))\leq \rho_i, 1\leq i\leq j-1 \\ \tilde{P}_{h,i}\in\Delta(\mathcal{O}):D_{\mathrm{KL}}(\tilde{P}_{h,i}(\cdot)\|P_{h,i}^{\star}(\cdot|s_h[\mathrm{pa}_i],a_h))\leq \rho_i, j\leq i \leq d}}\int_{\mathcal{O}^d}\prod_{i=1}^d\tilde{P}_{h,i}(\mathrm{d}s'[i])V_{h+1, P, \mathbf{\Phi}}^{\pi^{\star}}(s')
    \end{align*}
    By the definition of taking infimum over $d$ variables, we can conclude that 
    \begin{align}
        &\inf_{\substack{\tilde{P}_{h,i}\in\Delta(\mathcal{O}):D_{\mathrm{KL}}(\tilde{P}_{h,i}(\cdot)\|P_{h,i}(\cdot|s_h[\mathrm{pa}_i],a_h))\leq \rho_i, 1\leq i\leq j-1 \\ \tilde{P}_{h,i}\in\Delta(\mathcal{O}):D_{\mathrm{KL}}(\tilde{P}_{h,i}(\cdot)\|P_{h,i}^{\star}(\cdot|s_h[\mathrm{pa}_i],a_h))\leq \rho_i, j\leq i \leq d}}\int_{\mathcal{O}^d}\prod_{i=1}^d\tilde{P}_{h,i}(\mathrm{d}s'[i])V_{h+1, P, \mathbf{\Phi}}^{\pi^{\star}}(s') \notag\\
        &\qquad = \inf_{\tilde{P}_{h,j}\in\Delta(\mathcal{O}):D_{\mathrm{KL}}(\tilde{P}_{h,j}(\cdot)\|P_{h,j}^{\star}(\cdot|s_h[\mathrm{pa}_j],a_h))\leq \rho_j} \int_{\mathcal{O}^d}\tilde{P}_{h,j}(\mathrm{d}s'[j])\prod_{\substack{i=1\\i\neq j}}^d\tilde{P}_{h,i}^{\ast,j}(\mathrm{d}s'[i])V_{h+1, P, \mathbf{\Phi}}^{\pi^{\star}}(s'). \label{eq: safrmdp estimation proof 2}
    \end{align}
    Meanwhile, it naturally holds that for each $1\leq j\leq d$,
    \begin{align}
        &\inf_{\substack{\tilde{P}_{h,i}\in\Delta(\mathcal{O}):D_{\mathrm{KL}}(\tilde{P}_{h,i}(\cdot)\|P_{h,i}(\cdot|s_h[\mathrm{pa}_i],a_h))\leq \rho_i, 1\leq i\leq j \\ \tilde{P}_{h,i}\in\Delta(\mathcal{O}):D_{\mathrm{KL}}(\tilde{P}_{h,i}(\cdot)\|P_{h,i}^{\star}(\cdot|s_h[\mathrm{pa}_i],a_h))\leq \rho_i, j+1\leq i \leq d}}\int_{\mathcal{O}^d}\prod_{i=1}^d\tilde{P}_{h,i}(\mathrm{d}s'[i])V_{h+1, P, \mathbf{\Phi}}^{\pi^{\star}}(s')\notag\\
        &\qquad \leq \inf_{\tilde{P}_{h,j}\in\Delta(\mathcal{O}):D_{\mathrm{KL}}(\tilde{P}_{h,j}(\cdot)\|P_{h,j}(\cdot|s_h[\mathrm{pa}_j],a_h))\leq \rho_j} \int_{\mathcal{O}^d}\tilde{P}_{h,j}(\mathrm{d}s'[j])\prod_{\substack{i=1\\i\neq j}}^d\tilde{P}_{h,i}^{\ast,j}(\mathrm{d}s'[i])V_{h+1, P, \mathbf{\Phi}}^{\pi^{\star}}(s'). \label{eq: safrmdp estimation proof 3}
    \end{align}
    Thus by combining \eqref{eq: safrmdp estimation proof 2} and \eqref{eq: safrmdp estimation proof 3}, we have that
    \begin{align}
        \eqref{eq: safrmdp estimation proof 1}&\leq \sum_{j=1}^d\inf_{\tilde{P}_{h,j}\in\Delta(\mathcal{O}):D_{\mathrm{KL}}(\tilde{P}_{h,j}(\cdot)\|P_{h,j}(\cdot|s_h[\mathrm{pa}_j],a_h))\leq \rho_j} \int_{\mathcal{O}^d}\tilde{P}_{h,j}(\mathrm{d}s'[j])\prod_{\substack{i=1\\i\neq j}}^d\tilde{P}_{h,i}^{\ast,j}(\mathrm{d}s'[i])V_{h+1, P, \mathbf{\Phi}}^{\pi^{\star}}(s')\notag\\ 
        &\qquad - \inf_{\tilde{P}_{h,j}\in\Delta(\mathcal{O}):D_{\mathrm{KL}}(\tilde{P}_{h,j}(\cdot)\|P_{h,j}^{\star}(\cdot|s_h[\mathrm{pa}_j],a_h))\leq \rho_j} \int_{\mathcal{O}^d}\tilde{P}_{h,j}(\mathrm{d}s'[j])\prod_{\substack{i=1\\i\neq j}}^d\tilde{P}_{h,i}^{\ast,j}(\mathrm{d}s'[i])V_{h+1, P, \mathbf{\Phi}}^{\pi^{\star}}(s'). \label{eq: safrmdp estimation proof 4}
    \end{align}
    Now for simplicity, for each $1\leq j\leq d$, we denote a function $\upsilon_h^j(s'[j]):\mathcal{O}\mapsto\mathbb{R}$ as 
    \begin{align}\label{eq: upsilon}
        \upsilon_h^j(s'[j]) =  \int_{\mathcal{O}^{d-1}}\prod_{\substack{i=1\\i\neq j}}^d\tilde{P}_{h,i}^{\ast,j}(\mathrm{d}s'[i])V_{h+1, P, \mathbf{\Phi}}^{\pi^{\star}}(s'[1],\cdots,s'[j-1],s[j],s'[j+1],\cdots,s'[d]),
    \end{align}
    which satisfies $0\leq \upsilon_h^j \leq H$. 
    For each $1\leq j\leq d$, we can then upper bound 
    \begin{align}
        \Delta_{h}^j(s_h,a_h) &= \inf_{\tilde{P}_{h,j}\in\Delta(\mathcal{O}):D_{\mathrm{KL}}(\tilde{P}_{h,j}(\cdot)\|P_{h,j}(\cdot|s_h[\mathrm{pa}_j],a_h))\leq \rho_j} \int_{\mathcal{O}}\tilde{P}_{h,j}(\mathrm{d}s'[j])\upsilon_h^j(s'[j])\notag\\
        &\qquad - \inf_{\tilde{P}_{h,j}\in\Delta(\mathcal{O}):D_{\mathrm{KL}}(\tilde{P}_{h,j}(\cdot)\|P_{h,j}^{\star}(\cdot|s_h[\mathrm{pa}_j],a_h))\leq \rho_j} \int_{\mathcal{O}}\tilde{P}_{h,j}(\mathrm{d}s'[j])\upsilon_h^j(s'[j])\label{eq: safrmdp estimation proof delta}
    \end{align}
    using the same argument as in the proof of Proposition \ref{prop: sarmdp estimation} under KL-divergence in Appendix \ref{subsec: proof prop sarmdp estimation}, in which we apply Assumption \ref{ass: kl regularity frmdp} and Lemma \ref{lem: bound lambda kl}.
    The corresponding result is given by 
    \begin{align}\label{eq: safrmdp estimation proof 5}
        \Delta_{h}^j(s_h,a_h) \leq \frac{H\exp(H/\underline{\lambda})}{\rho_j}\cdot\|P_{h,j}(\cdot|s_h[\mathrm{pa}_j],a_h) - P^{\star}_{h,j}(\cdot|s_h[\mathrm{pa}_j],a_h)\|_{\mathrm{TV}}.
    \end{align}
    Thus plugging \eqref{eq: safrmdp estimation proof 5} into \eqref{eq: safrmdp estimation proof 4} and \eqref{eq: safrmdp estimation proof 1}, we can arrive at 
    \begin{align}
        &\inf_{\tilde{P}_h\in\mathbf{\Phi}(P_h)}\mathbb{E}_{s'\sim \tilde{P}_h(\cdot|s_h,a_h)}[V_{h+1, P, \mathbf{\Phi}}^{\pi^{\star}}(s')] - \inf_{\tilde{P}_h\in\mathbf{\Phi}(P_h^{\star})}\mathbb{E}_{s'\sim \tilde{P}_h(\cdot|s_h,a_h)}[V_{h+1, P, \mathbf{\Phi}}^{\pi^{\star}}(s')]\notag\\ 
        &\qquad \leq \sum_{j=1}^d\frac{H\exp(H/\underline{\lambda})}{\rho_j}\cdot\|P_{h,j}(\cdot|s_h[\mathrm{pa}_j],a_h) - P^{\star}_{h,j}(\cdot|s_h[\mathrm{pa}_j],a_h)\|_{\mathrm{TV}}.\label{eq: safrmdp estimation proof 6}
    \end{align}
    By using the same argument for deriving \eqref{eq: safrmdp estimation proof 6}, we can also obtain that 
    \begin{align}
        &\inf_{\tilde{P}_h\in\mathbf{\Phi}(P_h^{\star})}\mathbb{E}_{s'\sim \tilde{P}_h(\cdot|s_h,a_h)}[V_{h+1, P, \mathbf{\Phi}}^{\pi^{\star}}(s')] - \inf_{\tilde{P}_h\in\mathbf{\Phi}(P_h)}\mathbb{E}_{s'\sim \tilde{P}_h(\cdot|s_h,a_h)}[V_{h+1, P, \mathbf{\Phi}}^{\pi^{\star}}(s')]\notag\\
        &\qquad \leq \sum_{j=1}^d\frac{H\exp(H/\underline{\lambda})}{\rho_j}\cdot\|P_{h,j}(\cdot|s_h[\mathrm{pa}_j],a_h) - P^{\star}_{h,j}(\cdot|s_h[\mathrm{pa}_j],a_h)\|_{\mathrm{TV}}.\label{eq: safrmdp estimation proof 7}
    \end{align}
    Therefore, due to \eqref{eq: safrmdp estimation proof 6} and \eqref{eq: safrmdp estimation proof 7}, we can finally arrive at the following upper bound,
    \begin{align}
        &\mathbb{E}_{(s_h,a_h)\sim d^{\pi^\mathrm{b}}_{P^{\star},h}}\left[\left(\inf_{\tilde{P}_h\in\mathbf{\Phi}(P_h)}\mathbb{E}_{s'\sim \tilde{P}_h(\cdot|s_h,a_h)}[V_{h+1, P, \mathbf{\Phi}}^{\pi^{\star}}(s')] - \inf_{\tilde{P}_h\in\mathbf{\Phi}(P^{\star}_h)}\mathbb{E}_{s'\sim \tilde{P}_h(\cdot|s_h,a_h)}[V_{h+1, P, \mathbf{\Phi}}^{\pi^{\star}}(s')]\right)^2\right]\notag\\ 
        &\qquad \leq \mathbb{E}_{(s_h,a_h)\sim d^{\pi^\mathrm{b}}_{P^{\star},h}}\left[\left(\sum_{j=1}^d\frac{H\exp(H/\underline{\lambda})}{\rho_j}\cdot\|P_{h,j}(\cdot|s_h[\mathrm{pa}_j],a_h) - P^{\star}_{h,j}(\cdot|s_h[\mathrm{pa}_j],a_h)\|_{\mathrm{TV}}\right)^2\right]\notag\\ 
        &\qquad \leq \frac{dH^2\exp(2H/\underline{\lambda})}{\rho_{\min}} \cdot  \sum_{j=1}^d\mathbb{E}_{(s_h[\mathrm{pa}_j],a_h)\sim d^{\pi^\mathrm{b}}_{P^{\star},h}}\left[\|P_{h,j}(\cdot|s_h[\mathrm{pa}_j],a_h) - P^{\star}_{h,j}(\cdot|s_h[\mathrm{pa}_j],a_h)\|_{\mathrm{TV}}^2\right],\label{eq: safrmdp estimation proof 8}
    \end{align}
    where the last inequality is from Cauchy-Schwarz inequality and $\rho_{\min} = \min_{i\in[d]}\rho_i$.
    Now invoking the second conclusion of Lemma \ref{lem: mle guarantee safrmdp}, we have that with probability at least $1-\delta$, 
    \begin{align}\label{eq: safrmdp estimation proof 9}
        \mathbb{E}_{(s_h[\mathrm{pa}_j],a_h)\sim d^{\pi^\mathrm{b}}_{P^{\star},h}}[\|P_{h,j}(\cdot|s_h[\mathrm{pa}_j],a_h) - P^{\star}_{h,j}(\cdot|s_h[\mathrm{pa}_j],a_h)\|_{\mathrm{TV}}^2] \leq \frac{C_1'|\mathcal{O}|^{1+|\mathrm{pa}_j|}|\mathcal{A}|\log(C_2'ndH/\delta)}{n},
    \end{align}
    for some absolute constant $C_1',C_2'>0$ and each $j\in[d]$.
    Combining \eqref{eq: safrmdp estimation proof 8} and \eqref{eq: safrmdp estimation proof 9}, we have that
    \begin{align*}
        \sqrt{\mathrm{Err}_{h,\mathrm{KL}}^{\mathbf{\Phi}}(n)} = \frac{H\exp(H/\underline{\lambda})}{\rho_{\min}}\cdot \sqrt{\frac{dC_1'\sum_{i=1}^d|\mathcal{O}|^{1+|\mathrm{pa}_i|}|\mathcal{A}|\log(C_2'ndH/\delta)}{n}}.
    \end{align*}
    This finishes the proof of Proposition \ref{prop: safrmdp estimation} under KL-divergence.
\end{proof}

\begin{proof}[Proof of Proposition \ref{prop: safrmdp estimation} with TV-distance]
    Firstly, by invoking the first conclusion of Lemma \ref{lem: mle guarantee safrmdp}, we know that the Condition \ref{cond: accuracy} holds.
    In the following, we prove the Condition \ref{cond: model estimation}.
    Using the same argument as in the proof of Proposition \ref{prop: safrmdp estimation} under KL-divergence, we can derive that 
    \begin{align}\label{eq: safrmdp estimation proof 10}
        \inf_{\tilde{P}_h\in\mathbf{\Phi}(P_h)}\mathbb{E}_{s'\sim \tilde{P}_h(\cdot|s_h,a_h)}[V_{h+1, P, \mathbf{\Phi}}^{\pi^{\star}}(s')] - \inf_{\tilde{P}_h\in\mathbf{\Phi}(P_h^{\star})}\mathbb{E}_{s'\sim \tilde{P}_h(\cdot|s_h,a_h)}[V_{h+1, P, \mathbf{\Phi}}^{\pi^{\star}}(s')]\leq \sum_{j=1}^d\Delta_h^j(s_h,a_h),
    \end{align}
    where $\Delta_h^j(s_h,a_h)$ is defined in \eqref{eq: safrmdp estimation proof delta}. 
    Now applying the same argument as in the proof of Proposition \ref{prop: sarmdp estimation} under TV-divergence, we can derive that 
    \begin{align}\label{eq: safrmdp estimation proof 11}
        \Delta_h^j(s_h,a_h) \leq H\cdot\|P_{h,j}(\cdot|s_h[\mathrm{pa}_j],a_h) - P_{h,j}^{\star}(\cdot|s_h[\mathrm{pa}_j],a_h)\|_{\mathrm{TV}},
    \end{align}
    where we have applied Lemma \ref{lem: bound lambda tv}.
    Therefore, by combining \eqref{eq: safrmdp estimation proof 10} and \eqref{eq: safrmdp estimation proof 11}, we can derive that
    \begin{align}\label{eq: safrmdp estimation proof 12}
        &\inf_{\tilde{P}_h\in\mathbf{\Phi}(P_h)}\mathbb{E}_{s'\sim \tilde{P}_h(\cdot|s_h,a_h)}[V_{h+1, P, \mathbf{\Phi}}^{\pi^{\star}}(s')] - \inf_{\tilde{P}_h\in\mathbf{\Phi}(P_h^{\star})}\mathbb{E}_{s'\sim \tilde{P}_h(\cdot|s_h,a_h)}[V_{h+1, P, \mathbf{\Phi}}^{\pi^{\star}}(s')]\notag\\
        &\qquad \leq H\cdot\sum_{j=1}^d\|P_{h,j}(\cdot|s_h[\mathrm{pa}_j],a_h) - P_{h,j}^{\star}(\cdot|s_h[\mathrm{pa}_j],a_h)\|_{\mathrm{TV}}.
    \end{align}
    By the same argument as in deriving \eqref{eq: safrmdp estimation proof 12}, we can also obtain that,
    \begin{align}\label{eq: safrmdp estimation proof 13}
        &\inf_{\tilde{P}_h\in\mathbf{\Phi}(P_h^{\star})}\mathbb{E}_{s'\sim \tilde{P}_h(\cdot|s_h,a_h)}[V_{h+1, P, \mathbf{\Phi}}^{\pi^{\star}}(s')]-\inf_{\tilde{P}_h\in\mathbf{\Phi}(P_h)}\mathbb{E}_{s'\sim \tilde{P}_h(\cdot|s_h,a_h)}[V_{h+1, P, \mathbf{\Phi}}^{\pi^{\star}}(s')] \notag\\
        &\qquad \leq H\cdot\sum_{j=1}^d\|P_{h,j}(\cdot|s_h[\mathrm{pa}_j],a_h) - P_{h,j}^{\star}(\cdot|s_h[\mathrm{pa}_j],a_h)\|_{\mathrm{TV}}.
    \end{align}
    Now by combining \eqref{eq: safrmdp estimation proof 12} and \eqref{eq: safrmdp estimation proof 13}, we can derive the following upper bound,
    \begin{align}
        &\mathbb{E}_{(s_h,a_h)\sim d^{\pi^\mathrm{b}}_{P^{\star},h}}\left[\left(\inf_{\tilde{P}_h\in\mathbf{\Phi}(P_h)}\mathbb{E}_{s'\sim \tilde{P}_h(\cdot|s_h,a_h)}[V_{h+1, P, \mathbf{\Phi}}^{\pi^{\star}}(s')] - \inf_{\tilde{P}_h\in\mathbf{\Phi}(P^{\star}_h)}\mathbb{E}_{s'\sim \tilde{P}_h(\cdot|s_h,a_h)}[V_{h+1, P, \mathbf{\Phi}}^{\pi^{\star}}(s')]\right)^2\right]\notag\\ 
        &\qquad \leq \mathbb{E}_{(s_h,a_h)\sim d^{\pi^\mathrm{b}}_{P^{\star},h}}\left[\left(H\cdot\sum_{j=1}^d\|P_{h,j}(\cdot|s_h[\mathrm{pa}_j],a_h) - P^{\star}_{h,j}(\cdot|s_h[\mathrm{pa}_j],a_h)\|_{\mathrm{TV}}\right)^2\right]\notag\\ 
        &\qquad \leq dH^2 \cdot  \sum_{j=1}^d\mathbb{E}_{(s_h[\mathrm{pa}_j],a_h)\sim d^{\pi^\mathrm{b}}_{P^{\star},h}}\left[\|P_{h,j}(\cdot|s_h[\mathrm{pa}_j],a_h) - P^{\star}_{h,j}(\cdot|s_h[\mathrm{pa}_j],a_h)\|_{\mathrm{TV}}^2\right],\label{eq: safrmdp estimation proof 14}
    \end{align}
    where the last inequality follows from Cauchy-Schwartz inequality.
    Now invoking the second conclusion of Lemma \ref{lem: mle guarantee safrmdp}, we have that with probability at least $1-\delta$, 
    \begin{align}\label{eq: safrmdp estimation proof 15}
        \mathbb{E}_{(s_h[\mathrm{pa}_j],a_h)\sim d^{\pi^\mathrm{b}}_{P^{\star},h}}[\|P_{h,j}(\cdot|s_h[\mathrm{pa}_j],a_h) - P^{\star}_{h,j}(\cdot|s_h[\mathrm{pa}_j],a_h)\|_{\mathrm{TV}}^2] \leq \frac{C_1'|\mathcal{O}|^{1+|\mathrm{pa}_j|}|\mathcal{A}|\log(C_2'ndH/\delta)}{n},
    \end{align}
    for some absolute constant $C_1',C_2'>0$ and each $j\in[d]$.
    Combining \eqref{eq: safrmdp estimation proof 14} and \eqref{eq: safrmdp estimation proof 15}, we have that
    \begin{align*}
        \sqrt{\mathrm{Err}_{h,\mathrm{KL}}^{\mathbf{\Phi}}(n)} = H \cdot \sqrt{\frac{dC_1'\sum_{i=1}^d|\mathcal{O}|^{1+|\mathrm{pa}_i|}|\mathcal{A}|\log(C_2'ndH/\delta)}{n}}.
    \end{align*}
    This finishes the proof of Proposition \ref{prop: safrmdp estimation} under TV-distance.
\end{proof}

%% file: tex/appendix/drlmdp.tex
\section{Proofs for $d$-rectangular Robust Linear MDP}\label{sec: proof drlmdp}

\begin{assumption}[Regularity of KL-divergence duality variable]\label{ass: kl regularity drlmdp}
    We assume that the optimal dual variable $\lambda^{\star}$ for the following optimization problem 
    \begin{align*}
        \sup_{\lambda\in\mathbb{R}_+}\left\{-\lambda\log\left(\mathbb{E}_{s'\sim \mu(\cdot)}\left[\exp\left\{-V_{h+1,Q,\mathbf{\Phi}}^{\pi^{\star}}(s')/\lambda\right\}\right]\right)-\lambda\rho\right\},
    \end{align*}
    is lower bounded by $\underline{\lambda}>0$ for any distribution $\mu\in\Delta(\mathcal{S})$, transition kernels $Q = \{Q_h\}_{h=1}^H\subseteq\mathcal{P}_{\mathrm{M}}$, and step $h\in[H]$.
\end{assumption}

\begin{proof}[Proof of Theorem \ref{cor: suboptimality drlmdp} with KL-divergence]
    Recall that we consider the following definition of $\mathcal{V}$,
    \begin{align}\label{eq: function class V kl}
        \mathcal{V} = \left\{v(s) = \exp\left(-\left\{\max_{a\in\mathcal{A}}\boldsymbol{\phi}(s,a)^\top\boldsymbol{w}/\lambda\right\}_+\right):\|\boldsymbol{w}\|_2\leq H\sqrt{d}, \lambda\in [\underline{\lambda},H/\rho]\right\}.
    \end{align}
    Following the Section 7 of \cite{uehara2021pessimistic} as well as the Section 8 of \cite{agarwal2019reinforcement}, we introduce the notion $\hat{P}_h$ that satisfies for any $v\in\mathcal{V}$ and $(s,a)\in\mathcal{S}\times\mathcal{A}$, 
    \begin{align}\label{eq: hat P}
        \int_{\mathcal{S}}\hat{P}_h(\mathrm{d}s'|s,a)v(s') = \boldsymbol{\phi}(s,a)^\top\hat{\boldsymbol{\theta}}_{h, v},
    \end{align}
    where $\hat{\boldsymbol{\theta}}_{h, v}$ is defined in \eqref{eq: hat theta v}. Actually $\hat{P}_h$ takes the following closed form, 
    \begin{align}\label{eq: hat P closed}
        \hat{P}_h(\mathrm{d}s'|s,a) = \boldsymbol{\phi}(s,a)^\top \frac{1}{n}\sum_{\tau=1}^n \boldsymbol{\Lambda}_{h,\alpha}^{-1}\boldsymbol{\phi}(s_h^{\tau},a_h^{\tau})\delta_{s_{h+1}^{\tau}}(\mathrm{d}s'),
    \end{align}
    where $\delta_{s}(\cdot)$ is the Dirac measure centering at $s$.
    Regarding the estimator $\hat{P}_h$, we have the following.
    \begin{lemma}\label{lem: drlmdp estimator}
        Setting $\alpha = 1$ and choosing the function class $\mathcal{V}$ as \eqref{eq: function class V kl}, then the estimator $\hat{P}_h$ defined in \eqref{eq: hat P closed} satisfies that, with probability at least $1-\delta$,
        \begin{align*}
            &\sup_{v\in\mathcal{V}}\left|\int_{\mathcal{S}}\big(P^{\star}_h(\mathrm{d}s'|s,a)-\hat{P}_h(\mathrm{d}s'|s,a)\big)v(s') \right|^2\\
            &\qquad \leq C_1\cdot \|\boldsymbol{\phi}(s,a)\|_{\boldsymbol{\Lambda}_{h,\alpha}^{-1}}^2\cdot\frac{d\big(\log(1+C_2nH/\delta) + \log(1+C_3ndH/(\rho\underline{\lambda}^2))\big)}{n},
        \end{align*}
        for any step $h\in[H]$, where $C_1,C_2,C_3>0$ are three constants.
    \end{lemma}
    \begin{proof}[Proof of Lemma \ref{lem: drlmdp estimator}]
        See Appendix \ref{subsec: proof lem drlmdp estimator} for a detailed proof.
    \end{proof}

    With Lemma \ref{lem: drlmdp estimator}, we can further derive that, with probability at least $1-\delta$, for any $h\in[H]$,
    \begin{align*}
        &\sup_{v\in\mathcal{V}}\frac{1}{n}\sum_{\tau=1}^n\left|\int_{\mathcal{S}}\big(P^{\star}_h(\mathrm{d}s'|s_h^\tau,a_h^\tau)-\hat{P}_h(\mathrm{d}s'|s_h^\tau,a_h^\tau)\big)v(s') \right|^2\\
        &\qquad \leq \frac{1}{n}\sum_{\tau=1}^n \|\boldsymbol{\phi}(s_h^\tau,a_h^\tau)\|_{\boldsymbol{\Lambda}_{h,\alpha}^{-1}}^2\cdot \frac{C_1d\big(\log(1+C_2nH/\delta) + \log(1+C_3ndH/(\rho\underline{\lambda}^2))\big)}{n}.
    \end{align*}
    In the right hand side of the above inequality, it holds that, 
    \begin{align}
        \frac{1}{n}\sum_{\tau=1}^n \|\boldsymbol{\phi}(s_h^\tau,a_h^\tau)\|_{\boldsymbol{\Lambda}_{h,\alpha}^{-1}}^2 &= \frac{1}{n}\sum_{i=1}^n\tr\left(\boldsymbol{\phi}(s_h^\tau,a_h^\tau)^\top\boldsymbol{\Lambda}_{h,\alpha}^{-1}\boldsymbol{\phi}(s_h^\tau,a_h^\tau)\right) \notag \\
        & =\tr\left(\frac{1}{n}\sum_{i=1}^n\boldsymbol{\phi}(s_h^\tau,a_h^\tau)\boldsymbol{\phi}(s_h^\tau,a_h^\tau)^\top\boldsymbol{\Lambda}_{h,\alpha}^{-1}\right) \notag \\ 
        & \leq \tr\left(\boldsymbol{\Lambda}_{h,\alpha} \boldsymbol{\Lambda}_{h,\alpha}^{-1}\right) = d.\label{eq: drmdp term 1-}
    \end{align}
    Thus, we have that with probability at least $1-\delta$, for each step $h\in[H]$,
    \begin{align*}
        &\sup_{v\in\mathcal{V}}\frac{1}{n}\sum_{\tau=1}^n\left|\int_{\mathcal{S}}\big(P^{\star}_h(\mathrm{d}s'|s_h^\tau,a_h^\tau)-\hat{P}_h(\mathrm{d}s'|s_h^\tau,a_h^\tau)\big)v(s') \right|^2\\
        &\qquad\leq \frac{C_1d^2\big(\log(1+C_2nH/\delta) + \log(1+C_3ndH/(\rho\underline{\lambda}^2))\big)}{n} =\xi.
    \end{align*}
    This proves Condition \ref{cond: accuracy} in Section \ref{subsec: theoretical analysis}.
    In the following, we prove Theorem \ref{cor: suboptimality drlmdp} given Condition \ref{cond: accuracy} holds.
    Using the definition of robust set $\boldsymbol{\Phi}(\cdot)$ in Example~\ref{exp: dlmdp}, we can derive that%linear structure for the transition kernels in the model space $\mathcal{P}_{\mathrm{M}}$, we have that 
    \begin{align}
        &\inf_{\tilde{P}_h\in\mathbf{\Phi}(P_h^{\star})}\mathbb{E}_{s'\sim \tilde{P}_h(\cdot|s_h,a_h)}[V_{h+1, P, \mathbf{\Phi}}^{\pi^{\star}}(s')] - \inf_{\tilde{P}_h\in\mathbf{\Phi}(P_h)}\mathbb{E}_{s'\sim \tilde{P}_h(\cdot|s_h,a_h)}[V_{h+1, P, \mathbf{\Phi}}^{\pi^{\star}}(s')]\notag\\
        &\qquad = \inf_{\tilde{P}_h\in\mathbf{\Phi}(P_h^{\star})}\sum_{i=1}^d\phi_i(s_h,a_h) \int_{\mathcal{S}}\tilde{\mu}_i(\mathrm{d}s')V_{h+1, P, \mathbf{\Phi}}^{\pi^{\star}}(s') - \inf_{\tilde{P}_h\in\mathbf{\Phi}(P_h)}\sum_{i=1}^d\phi_i(s,a) \int_{\mathcal{S}}\tilde{\mu}_i(\mathrm{d}s')V_{h+1, P, \mathbf{\Phi}}^{\pi^{\star}}(s')\notag\\
        &\qquad = \sum_{i=1}^d\phi_i(s_h,a_h) \inf_{\tilde{\mu}_{h,i}\in\Delta(\mathcal{S}):D(\tilde{\mu}_{h,i}(\cdot)\|\mu_{h,i}^{\star}(\cdot))\leq \rho} \int_{\mathcal{S}}\tilde{\mu}_{h,i}(\mathrm{d}s')V_{h+1, P, \mathbf{\Phi}}^{\pi^{\star}}(s') \notag \\
        &\qquad \qquad - \sum_{i=1}^d\phi_i(s_h,a_h)\inf_{\tilde{\mu}_{h,i}\in\Delta(\mathcal{S}):D(\tilde{\mu}_{h,i}(\cdot)\|\mu_{h,i}(\cdot))\leq \rho} \int_{\mathcal{S}}\tilde{\mu}_{h,i}(\mathrm{d}s')V_{h+1, P, \mathbf{\Phi}}^{\pi^{\star}}(s'),\label{eq: drmdp term 1}
    \end{align}
    where the last equality follows from $\phi(s,a)\geq 0$ for any $i\in[d]$.
    Now invoking the dual formulation of KL-divergence in Lemma \ref{lem: kl}, we can derive that 
    \begin{align}
        \eqref{eq: drmdp term 1} &= \sum_{i=1}^d\phi_i(s_h,a_h)\cdot\Bigg[
            \sup_{\lambda_i\geq 0} \left\{-\lambda_i\log\left(\mathbb{E}_{s'\sim \mu_{h,i}^{\star}(\cdot)}\left[\exp\left\{-V_{h+1,P,\mathbf{\Phi}}^{\pi^{\star}}(s')/\lambda_i\right\}\right]\right)-\lambda_i\rho\right\}\notag\\
        &\qquad - \sup_{\lambda_i\geq 0} \left\{-\lambda_i\log\left(\mathbb{E}_{s'\sim \mu_{h,i}(\cdot)}\left[\exp\left\{-V_{h+1,P,\mathbf{\Phi}}^{\pi^{\star}}(s')/\lambda_i\right\}\right]\right)-\lambda_i\rho\right\}\Bigg]\label{eq: drmdp term 2}
    \end{align}
    Following the same argument in the proof of Proposition \ref{prop: sarmdp estimation} (derivation of \eqref{eq: sketch term ii 3}), during which we invoke Assumption \ref{ass: kl regularity drlmdp} and Lemma \ref{lem: bound lambda kl} to bound the optimal dual variable $\lambda$, we can derive that 
    \begin{align}
        \eqref{eq: drmdp term 2} &\leq \sum_{i=1}^d\phi_i(s_h,a_h)\cdot\sup_{\underline{\lambda}\leq \lambda_i \leq H/\rho}\left\{g(\lambda_i,\mu_{h,i}^{\star}) \int_{\mathcal{S}}\left(\mu_{h,i}^{\star}(\mathrm{d}s') - \mu_{h,i}(\mathrm{d}s') \right) \exp\left\{-V_{h+1,P,\mathbf{\Phi}}^{\pi^{\star}}(s')/\lambda_i\right\}\right\},\notag\\
        &= \sum_{i=1}^d\sup_{\underline{\lambda}\leq \lambda_i \leq H/\rho}\left\{g(\lambda_i,\mu_{h,i}^{\star}) \phi_i(s_h,a_h)\int_{\mathcal{S}}\left(\mu_{h,i}^{\star}(\mathrm{d}s') - \mu_{h,i}(\mathrm{d}s') \right) \exp\left\{-V_{h+1,P,\mathbf{\Phi}}^{\pi^{\star}}(s')/\lambda_i\right\}\right\},\label{eq: drmdp term 3}
    \end{align}
    where we have defined $g(\lambda_i,\mu_{h,i}) = \lambda_i / (\int_{\mathcal{S}}\mu_{h,i}(\mathrm{d}s')\exp\{-V_{h+1,P,\mathbf{\Phi}}^{\pi^{\star}}(s')/\lambda_i\})$ for simplicity, and in the equality we have used the fact that $\phi_i(s,a)\geq 0$.
    To go ahead, we rewrite the summand in \eqref{eq: drmdp term 3} for each $i\in[d]$.
    To be specific, recall the regularized covariance matrix $\boldsymbol{\Lambda}_{h,\alpha}$ of the feature $\boldsymbol{\phi}$,
    % \begin{align*}
    %     \boldsymbol{\Lambda}_{h,\alpha} = \mathbb{E}_{(s_h,a_h)\sim d^{\pi^{\mathrm{b}}}_{P^{\star},h}}[\boldsymbol{\phi}(s_h,a_h)\boldsymbol{\phi}(s_h,a_h)^\top].
    % \end{align*}
    \begin{align*}
        \boldsymbol{\Lambda}_{h,\alpha} = \frac{1}{n}\sum_{\tau=1}^n\boldsymbol{\phi}(s_h^{\tau},a_h^{\tau})\boldsymbol{\phi}(s_h^{\tau},a_h^{\tau})^\top + \frac{\alpha}{n}\cdot\boldsymbol{I}_d.
    \end{align*}
    Then, by denoting $\mathbf{1}_i = (0,\cdots,0,1,0,\cdots,0)^\top$ where $1$ is at the $i$-th coordinate, we have the following,
    \begin{align}
        &\phi_i(s_h,a_h)\int_{\mathcal{S}}\left(\mu_{h,i}^{\star}(\mathrm{d}s') - \mu_{h,i}(\mathrm{d}s') \right) \exp\left\{-V_{h+1,P,\mathbf{\Phi}}^{\pi^{\star}}(s')/\lambda_i\right\}\notag\\
        &\qquad = \phi_i(s_h,a_h)\mathbf{1}_i^\top\boldsymbol{\Lambda}_{h,\alpha}^{-1/2}\boldsymbol{\Lambda}_{h,\alpha}^{1/2}\int_{\mathcal{S}}\left(\boldsymbol{\mu}_{h}^{\star}(\mathrm{d}s') - \boldsymbol{\mu}_{h}(\mathrm{d}s') \right) \exp\left\{-V_{h+1,P,\mathbf{\Phi}}^{\pi^{\star}}(s')/\lambda_i\right\}\notag\\
        &\qquad \leq \underbrace{\left\| \phi_i(s_h,a_h)\mathbf{1}_i \right\|_{\boldsymbol{\Lambda}_{h,\alpha}^{-1}}}_{\text{Term (i)}}\cdot \underbrace{\left\| \int_{\mathcal{S}}\left(\boldsymbol{\mu}_{h}^{\star}(\mathrm{d}s') - \boldsymbol{\mu}_{h}(\mathrm{d}s') \right) \exp\left\{-V_{h+1,P,\mathbf{\Phi}}^{\pi^{\star}}(s')/\lambda_i\right\} \right\|_{\boldsymbol{\Lambda}_{h,\alpha}}}_{\text{Term (ii)}}.\label{eq: drmdp term 4}
    \end{align}
    For the term (ii) in \eqref{eq: drmdp term 4}, by the definition of $\boldsymbol{\Lambda}_{h,\alpha}$, we have that,
    \begin{align}
        \text{Term (ii)}^2 &= \frac{1}{n}\sum_{\tau=1}^n\left|\boldsymbol{\phi}(s_h^{\tau},a_h^{\tau})^\top\int_{\mathcal{S}}\left(\boldsymbol{\mu}_{h}^{\star}(\mathrm{d}s') - \boldsymbol{\mu}_{h}(\mathrm{d}s') \right) \exp\left\{-V_{h+1,P,\mathbf{\Phi}}^{\pi^{\star}}(s')/\lambda_i\right\} \right|^2\notag\\
        &\qquad  + \frac{\alpha}{n}\cdot \left\| \int_{\mathcal{S}}\left(\boldsymbol{\mu}_{h}^{\star}(\mathrm{d}s') - \boldsymbol{\mu}_{h}(\mathrm{d}s') \right) \exp\left\{-V_{h+1,P,\mathbf{\Phi}}^{\pi^{\star}}(s')/\lambda_i\right\} \right\|_{2}^2 \notag\\
        &= \frac{1}{n}\sum_{\tau=1}^n\left|\int_{\mathcal{S}}\left(P_h^{\star}(\mathrm{d}s'|s_h^{\tau},a_h^{\tau}) - P_{h}(\mathrm{d}s'|s_h^{\tau},a_h^{\tau}) \right) \exp\left\{-V_{h+1,P,\mathbf{\Phi}}^{\pi^{\star}}(s')/\lambda_i\right\} \right|^2\notag\\
        &\qquad  + \frac{\alpha}{n}\cdot \left\| \int_{\mathcal{S}}\left(\boldsymbol{\mu}_{h}^{\star}(\mathrm{d}s') - \boldsymbol{\mu}_{h}(\mathrm{d}s') \right) \exp\left\{-V_{h+1,P,\mathbf{\Phi}}^{\pi^{\star}}(s')/\lambda_i\right\} \right\|_{2}^2.\label{eq: drmdp term 5+}
    \end{align}
    In the following, we upper bound the right hand side of \eqref{eq: drmdp term 5+}. On the one hand, we have that 
    \begin{align}
        &\frac{1}{n}\sum_{\tau=1}^n\left|\int_{\mathcal{S}}\left(P_h^{\star}(\mathrm{d}s'|s_h^{\tau},a_h^{\tau}) - P_{h}(\mathrm{d}s'|s_h^{\tau},a_h^{\tau}) \right) \exp\left\{-V_{h+1,P,\mathbf{\Phi}}^{\pi^{\star}}(s')/\lambda_i\right\} \right|^2\notag \\
        &\qquad \leq \sup_{v\in\mathcal{V}}\frac{1}{n}\sum_{\tau=1}^n\left|\int_{\mathcal{S}}\left(P_h^{\star}(\mathrm{d}s'|s_h^{\tau},a_h^{\tau}) - \hat{P}_{h}(\mathrm{d}s'|s_h^{\tau},a_h^{\tau}) \right)v(s') \right|^2\notag \\ 
        &\qquad \qquad + \sup_{v\in\mathcal{V}}\frac{1}{n}\sum_{\tau=1}^n\left|\int_{\mathcal{S}}\left(\hat{P}_h(\mathrm{d}s'|s_h^{\tau},a_h^{\tau}) - P_{h}(\mathrm{d}s'|s_h^{\tau},a_h^{\tau}) \right)v(s') \right|^2\notag \\ 
        &\qquad \leq 2\xi, \label{eq: drmdp term 5++}
    \end{align}
    with probability at least $1-\delta$, where the first inequality holds since $\exp\{-V_{h+1,P,\mathbf{\Phi}}^{\pi^{\star}}(s')/\lambda_i\}\in\mathcal{V}$\footnote{This is because the robust action value function $Q^{\pi^{\star}}_{h+1,P,\boldsymbol{\Phi}}(s,a)$ is linear in the feature $\boldsymbol{\phi}(s,a)$ (Lemma 4.2 of \cite{ma2022distributionally}), and it is direct to see that $V^{\pi^{\star}}_{h+1,P,\boldsymbol{\Phi}}(s) = \max_{a\in\mathcal{A}}Q^{\pi^{\star}}_{h+1,P,\boldsymbol{\Phi}}(s,a)$.}, and the last inequality follows from the fact that Condition \ref{cond: accuracy} holds and the fact that $P_h\in\hat{\mathcal{P}}_h$.
    On the other hand, by setting the regularization parameter $\alpha = 1$ we have that 
    \begin{align}
        &\frac{\alpha}{n}\cdot \left\| \int_{\mathcal{S}}\left(\boldsymbol{\mu}_{h}^{\star}(\mathrm{d}s') - \boldsymbol{\mu}_{h}(\mathrm{d}s') \right) \exp\left\{-V_{h+1,P,\mathbf{\Phi}}^{\pi^{\star}}(s')/\lambda_i\right\} \right\|_{2}^2 \notag \\
        &\qquad = \frac{1}{n}\cdot \sum_{i=1}^d\left| \int_{\mathcal{S}}\left(\mu_{h,i}^{\star}(\mathrm{d}s') - \mu_{h,i}(\mathrm{d}s') \right) \exp\left\{-V_{h+1,P,\mathbf{\Phi}}^{\pi^{\star}}(s')/\lambda_i\right\} \right|^2 \notag \\
        &\qquad \leq \frac{1}{n}\cdot\sum_{i=1}^d \|\mu_{h,i}^{\star}(\cdot) - \mu_{h,i}(\cdot)\|_{\mathrm{TV}}^2 \leq \frac{2d}{n}. \label{eq: drmdp term 5+++}
    \end{align}
    By combining \eqref{eq: drmdp term 5+}, \eqref{eq: drmdp term 5++} and \eqref{eq: drmdp term 5+++}, we can conclude that with probability at least $1-\delta$, 
    \begin{align}
        \text{Term (ii)}^2 \leq 2\xi + \frac{2d}{n} \leq 3\xi.\label{eq: drmdp term 6}
    \end{align}
    Now by combining \eqref{eq: drmdp term 3}, \eqref{eq: drmdp term 4}, \eqref{eq: drmdp term 6}, we can conclude that with probability at least $1-\delta$, 
    \begin{align}
        &\inf_{\tilde{P}_h\in\mathbf{\Phi}(P_h^{\star})}\mathbb{E}_{s'\sim \tilde{P}_h(\cdot|s_h,a_h)}[V_{h+1, P, \mathbf{\Phi}}^{\pi^{\star}}(s')] - \inf_{\tilde{P}_h\in\mathbf{\Phi}(P_h)}\mathbb{E}_{s'\sim \tilde{P}_h(\cdot|s_h,a_h)}[V_{h+1, P, \mathbf{\Phi}}^{\pi^{\star}}(s')] \notag\\
        &\qquad\leq   \sum_{i=1}^d  \sup_{\underline{\lambda}\leq \lambda_i\leq H/\rho}\left\{\left\| \phi_i(s_h,a_h)\mathbf{1}_i \right\|_{\boldsymbol{\Lambda}_{h,\alpha}^{-1}}\cdot g(\lambda_i,\mu_{h,i}^{\star})\cdot \sqrt{3\xi}\right\}\notag\\
        &\qquad \leq \frac{2\sqrt{\xi}\cdot H\exp(H/\underline{\lambda})}{\rho}\cdot \sum_{i=1}^d\left\| \phi_i(s_h,a_h)\mathbf{1}_i \right\|_{\boldsymbol{\Lambda}_{h,\alpha}^{-1}},\label{eq: drmdp term 7}
    \end{align}
    for any step $h\in[H]$, $(s_h,a_h)\in\mathcal{S}\times\mathcal{A}$, and $P_h\in\hat{\mathcal{P}}_h$, where we apply the definition of $g(\lambda_i,\mu_i)$.
    Now using the same argument as in the proof of Theorem \ref{thm: subopt general}, using Condition \ref{cond: accuracy}, we can derive that 
    \begin{align}
        \mathrm{SubOpt}(\hat{\pi};s_1)&\leq \sup_{P\in\hat{\mathcal{P}}}\sum_{h=1}^H\mathbb{E}_{(s_h,a_h)\sim d^{\pi^{\star}}_{P^{\pi^{\star},\dagger},h}}\left[\inf_{\tilde{P}_h\in\mathbf{\Phi}(P_h^{\star})}\mathbb{E}_{s'\sim \tilde{P}_h(\cdot|s_h,a_h)}[V_{h+1, P, \mathbf{\Phi}}^{\pi^{\star}}(s')]\right.\notag\\
        &\qquad\qquad\left.- \inf_{\tilde{P}_h\in\mathbf{\Phi}(P_h)}\mathbb{E}_{s'\sim \tilde{P}_h(\cdot|s_h,a_h)}[V_{h+1, P, \mathbf{\Phi}}^{\pi^{\star}}(s')]\right]\notag\\
        %&\leq \frac{d^2H^2\exp(H/\underline{\lambda})}{\rho\underline{c}^{1/2}}\cdot \sqrt{\frac{C_3\log(C_4 ndH/\delta)}{n}},
        &\leq \frac{2\sqrt{\xi}\cdot H\exp(H/\underline{\lambda})}{\rho}\cdot \sum_{h=1}^H\sum_{i=1}^d\mathbb{E}_{(s_h,a_h)\sim d^{\pi^{\star}}_{P^{\pi^{\star},\dagger},h}}\left[\left\| \phi_i(s_h,a_h)\mathbf{1}_i \right\|_{\boldsymbol{\Lambda}_{h,\alpha}^{-1}}\right],\label{eq: drmdp term 8}
    \end{align}
    where we have used \eqref{eq: drmdp term 7}. 
    Here $P^{\pi^{\star},\dagger}_h$ is some transition kernel chosen from $\boldsymbol{\Phi}(P_h^{\star})$.
    Now we upper bound the right hand side of \eqref{eq: drmdp term 8} using Assumption \ref{ass: partial coverage cov}.
    Consider that 
    \begin{align}
        &\sum_{i=1}^d\mathbb{E}_{(s_h,a_h)\sim d^{\pi^{\star}}_{P^{\pi^{\star},\dagger},h}}\left[\left\| \phi_i(s_h,a_h)\mathbf{1}_i \right\|_{\boldsymbol{\Lambda}_{h,\alpha}^{-1}}\right] \notag \\
        &\qquad = \sum_{i=1}^d\mathbb{E}_{(s_h,a_h)\sim d^{\pi^{\star}}_{P^{\pi^{\star},\dagger},h}}\left[\sqrt{\tr\left((\phi_i(s_h,a_h)\mathbf{1}_i)(\phi_i(s_h,a_h)\mathbf{1}_i)^\top\boldsymbol{\Lambda}_{h,\alpha}^{-1}\right)}\right] \notag \\
        &\qquad \leq \sum_{i=1}^d\sqrt{\tr\left(\mathbb{E}_{(s_h,a_h)\sim d^{\pi^{\star}}_{P^{\pi^{\star},\dagger},h}}\left[(\phi_i(s_h,a_h)\mathbf{1}_i)(\phi_i(s_h,a_h)\mathbf{1}_i)^\top\right]\boldsymbol{\Lambda}_{h,\alpha}^{-1}\right)}.\label{eq: drmdp term 9}
    \end{align}
    For notational simplicity, in the sequel, we denote by 
    \begin{align*}
        \boldsymbol{\Sigma}_{P,h,i} = \mathbb{E}_{(s_h,a_h)\sim d^{\pi^{\star}}_{P,h}}\left[(\phi_i(s_h,a_h)\mathbf{1}_i)(\phi_i(s_h,a_h)\mathbf{1}_i)^\top\right]
    \end{align*}
    Note that the matrix $\boldsymbol{\Sigma}_{P,h,i}$ has non-zero element only at $(\boldsymbol{\Sigma}_{P,h,i})_{(i,i)}$, which equals to $\phi_i(s,a)^2$.
    Under Assumption \ref{ass: partial coverage cov} and the fact that $P_h^{\pi^{\star},\dagger} \in \boldsymbol{\Phi}(P_h^{\star})$, we have that 
    \begin{align*}
        \boldsymbol{\Lambda}_{h,\alpha} \succeq \frac{\alpha}{n}\cdot\boldsymbol{I}_d + c^{\dagger}\cdot\boldsymbol{\Sigma}_{P^{\pi^{\star},\dagger},h,i}.
    \end{align*}
    Thus, using \eqref{eq: drmdp term 9} and under $\alpha = 1$, we have that,
    \begin{align}
        \sum_{i=1}^d\mathbb{E}_{(s_h,a_h)\sim d^{\pi^{\star}}_{P^{\pi^{\star},\dagger},h}}\left[\left\| \phi_i(s_h,a_h)\mathbf{1}_i \right\|_{\boldsymbol{\Lambda}_{h,\alpha}^{-1}}\right] &\leq\sum_{i=1}^d\sqrt{\tr\left(\boldsymbol{\Sigma}_{P^{\pi^{\star}},h,i} \left( \frac{\alpha}{n}\cdot\boldsymbol{I}_d + c^{\dagger}\cdot \boldsymbol{\Sigma}_{P^{\pi^{\star}},h,i}\right)^{-1} \right)} \notag \\
        & = \sum_{i=1}^d \sqrt{\frac{\phi_i(s,a)^2}{n^{-1}+c^{\dagger}\cdot\phi_i(s,a)^2}} \leq \frac{d}{c^{\dagger}}.\label{eq: drmdp term 10}
    \end{align}
    Therefore, by combining \eqref{eq: drmdp term 8} and \eqref{eq: drmdp term 10}, we have that with probability at least $1-\delta$,
    \begin{align*}
        \mathrm{SubOpt}(\hat{\pi};s_1)&\leq \frac{2\sqrt{\xi}\cdot H\exp(H/\underline{\lambda})}{\rho}\cdot \sum_{h=1}^H\frac{d}{c^{\dagger}} =\frac{2d\sqrt{\xi}\cdot H^2\exp(H/\underline{\lambda})}{c^{\dagger}\rho}.
    \end{align*}
    Using the definition of $\xi$, we can finally derive that with probability at least $1-\delta$,
    \begin{align*}
        \mathrm{SubOpt}(\hat{\pi};s_1) \leq \frac{d^2H^2\exp(H/\underline{\lambda})}{c^{\dagger}\rho}\cdot\sqrt{\frac{C_1'\big(\log(1+C_2'nH/\delta) + \log(1+C_3'ndH/(\rho\underline{\lambda}^2))\big)}{n}}.
    \end{align*}
    This finishes the proof of Theorem \ref{cor: suboptimality drlmdp} under KL-divergence.
\end{proof}

\begin{proof}[Proof of Theorem \ref{cor: suboptimality drlmdp} with TV-divergence]
    We use the same notation of $\hat{P}_h$ introduced in the proof of KL-divergence case, which satisfies \eqref{eq: hat P} with $\mathcal{V}$ defined as 
    \begin{align}\label{eq: function class V tv}
        \mathcal{V} = \left\{v(s) = \left(\lambda - \max_{a\in\mathcal{A}}\boldsymbol{\phi}(s,a)^\top\boldsymbol{w}\right)_+: \|\boldsymbol{w}\|_2\leq H\sqrt{d}, \lambda\in[0,H]\right\}.
    \end{align}
    Regarding the estimator $\hat{P}_h$ with $\mathcal{V}$ defined in \eqref{eq: function class V tv}, we have the following.
    \begin{lemma}\label{lem: drlmdp estimator tv}
        Setting $\alpha = 1$ and choosing the function class $\mathcal{V}$ as \eqref{eq: function class V tv}, then the estimator $\hat{P}_h$ defined in \eqref{eq: hat P closed} satisfies that, with probability at least $1-\delta$,
        \begin{align*}
            &\sup_{v\in\mathcal{V}}\left|\int_{\mathcal{S}}\big(P^{\star}_h(\mathrm{d}s'|s,a)-\hat{P}_h(\mathrm{d}s'|s,a)\big)v(s') \right|^2  \leq C_1\cdot \|\boldsymbol{\phi}(s,a)\|_{\boldsymbol{\Lambda}_{h,\alpha}^{-1}}^2 \cdot\frac{dH^2\log(C_2ndH/\delta)}{n},
        \end{align*}
        for any step $h\in[H]$, where $C_1,C_2>0$ are two constants.
    \end{lemma}
    \begin{proof}[Proof of Lemma \ref{lem: drlmdp estimator tv}]
        See Appendix \ref{subsec: proof lem drlmdp estimator} for a detailed proof.
    \end{proof}

    With Lemma \ref{lem: drlmdp estimator tv}, we can further derive that, with probability at least $1-\delta$, for any $h\in[H]$,
    \begin{align*}
        \sup_{v\in\mathcal{V}}\frac{1}{n}\sum_{\tau=1}^n\left|\int_{\mathcal{S}}\big(P^{\star}_h(\mathrm{d}s'|s_h^\tau,a_h^\tau)-\hat{P}_h(\mathrm{d}s'|s_h^\tau,a_h^\tau)\big)v(s') \right|^2\leq \frac{1}{n}\sum_{\tau=1}^n \|\boldsymbol{\phi}(s_h^\tau,a_h^\tau)\|_{\boldsymbol{\Lambda}_{h,\alpha}^{-1}}^2\cdot \frac{C_1dH^2\log(C_2ndH/\delta)}{n}.
    \end{align*}
    In the right hand side of the above inequality, it holds that, 
    \begin{align}
        \frac{1}{n}\sum_{\tau=1}^n \|\boldsymbol{\phi}(s_h^\tau,a_h^\tau)\|_{\boldsymbol{\Lambda}_{h,\alpha}^{-1}}^2 = \frac{1}{n}\sum_{i=1}^n\tr\left(\boldsymbol{\phi}(s_h^\tau,a_h^\tau)^\top\boldsymbol{\Lambda}_{h,\alpha}^{-1}\boldsymbol{\phi}(s_h^\tau,a_h^\tau)\right)  \leq \tr\left(\boldsymbol{\Lambda}_{h,\alpha} \boldsymbol{\Lambda}_{h,\alpha}^{-1}\right) = d.\label{eq: drmdp tv term 1-}
    \end{align}
    Thus, we have that with probability at least $1-\delta$, for each step $h\in[H]$,
    \begin{align*}
        \sup_{v\in\mathcal{V}}\frac{1}{n}\sum_{\tau=1}^n\left|\int_{\mathcal{S}}\big(P^{\star}_h(\mathrm{d}s'|s_h^\tau,a_h^\tau)-\hat{P}_h(\mathrm{d}s'|s_h^\tau,a_h^\tau)\big)v(s') \right|^2\leq \frac{C_1d^2H^2\log(C_2ndH/\delta)}{n} =\xi.
    \end{align*}
    This proves Condition \ref{cond: accuracy} in Section \ref{subsec: theoretical analysis}.
    In the following, we prove Theorem \ref{cor: suboptimality drlmdp} given Condition \ref{cond: accuracy} holds.
    Using the definition of robust set $\boldsymbol{\Phi}(\cdot)$ in Example~\ref{exp: dlmdp}, following the same argument as \eqref{eq: drmdp term 1}, we have that,
    \begin{align}
        &\inf_{\tilde{P}_h\in\mathbf{\Phi}(P_h^{\star})}\mathbb{E}_{s'\sim \tilde{P}_h(\cdot|s_h,a_h)}[V_{h+1, P, \mathbf{\Phi}}^{\pi^{\star}}(s')] - \inf_{\tilde{P}_h\in\mathbf{\Phi}(P_h)}\mathbb{E}_{s'\sim \tilde{P}_h(\cdot|s_h,a_h)}[V_{h+1, P, \mathbf{\Phi}}^{\pi^{\star}}(s')]\notag\\
        &\qquad = \sum_{i=1}^d\phi_i(s_h,a_h) \inf_{\tilde{\mu}_{h,i}\in\Delta(\mathcal{S}):D(\tilde{\mu}_{h,i}(\cdot)\|\mu_{h,i}^{\star}(\cdot))\leq \rho} \int_{\mathcal{S}}\tilde{\mu}_{h,i}(\mathrm{d}s')V_{h+1, P, \mathbf{\Phi}}^{\pi^{\star}}(s') \notag \\
        &\qquad \qquad - \sum_{i=1}^d\phi_i(s_h,a_h)\inf_{\tilde{\mu}_{h,i}\in\Delta(\mathcal{S}):D(\tilde{\mu}_{h,i}(\cdot)\|\mu_{h,i}(\cdot))\leq \rho} \int_{\mathcal{S}}\tilde{\mu}_{h,i}(\mathrm{d}s')V_{h+1, P, \mathbf{\Phi}}^{\pi^{\star}}(s').\label{eq: drmdp tv term 1}
    \end{align}
    Now invoking the dual formulation of TV-distance in Lemma \ref{lem: tv}, we can further derive that 
    \begin{align}
        \eqref{eq: drmdp tv term 1} &= \sum_{i=1}^d\phi_i(s_h,a_h) \cdot \left[\sup_{\lambda\in\mathbb{R}}\left\{-\mathbb{E}_{s'\sim \mu_{h,i}^{\star}(\cdot)}\left[\left(\lambda - V_{h+1,P,\mathbf{\Phi}}^{\pi^{\star}}(s')\right)_+\right] - \frac{\rho}{2}\left(\lambda - \inf_{s''\in\mathcal{S}}V_{h+1,P,\mathbf{\Phi}}^{\pi}(s'')\right)+\lambda\right\}\right. \notag\\
        &\qquad \left. - \sup_{\lambda\in\mathbb{R}}\left\{-\mathbb{E}_{s'\sim \mu_{h,i}(\cdot)}\left[\left(\lambda - V_{h+1,P,\mathbf{\Phi}}^{\pi^{\star}}(s')\right)_+\right] - \frac{\rho}{2}\left(\lambda - \inf_{s''\in\mathcal{S}}V_{h+1,P,\mathbf{\Phi}}^{\pi}(s'')\right)+\lambda\right\}\right] \notag \\
        &\leq \sum_{i=1}^d\phi_i(s_h,a_h) \cdot \sup_{\lambda\in[0,H]} \left\{ \left(\mathbb{E}_{s'\sim \mu_{h,i}^{\star}(\cdot)} - \mathbb{E}_{s'\sim \mu_{h,i}(\cdot)}\right)\left[\left(\lambda - V_{h+1,P,\mathbf{\Phi}}^{\pi^{\star}}(s')\right)_+ \right]\right\}\notag \\
        & =  \sum_{i=1}^d \sup_{\lambda\in[0,H]} \left\{ \phi_i(s_h,a_h)\int_{\mathcal{S}}\left(\mu_{h,i}^{\star}(\mathrm{d}s') - \mu_{h,i}(\mathrm{d}s') \right) \left(\lambda - V_{h+1,P,\mathbf{\Phi}}^{\pi^{\star}}(s')\right)_+\right\}.\label{eq: drmdp tv term 2}
    \end{align}
    where in the first inequality we use Lemma \ref{lem: bound lambda tv} to bound $\lambda\in[0,H]$.
    Now we consider each summand $i\in[d]$ in the right hand side of \eqref{eq: drmdp tv term 2}. 
    We rewrite it as 
    \begin{align}
        &\phi_i(s_h,a_h)\int_{\mathcal{S}}\left(\mu_{h,i}^{\star}(\mathrm{d}s') - \mu_{h,i}(\mathrm{d}s') \right) \left(\lambda - V_{h+1,P,\mathbf{\Phi}}^{\pi^{\star}}(s')\right)_+ \notag\\
        &\qquad = \phi_i(s_h,a_h)\mathbf{1}_i^\top\boldsymbol{\Lambda}_{h,\alpha}^{-1/2}\boldsymbol{\Lambda}_{h,\alpha}^{1/2}\int_{\mathcal{S}}\left(\boldsymbol{\mu}_{h}^{\star}(\mathrm{d}s') - \boldsymbol{\mu}_{h}(\mathrm{d}s') \right)\left(\lambda - V_{h+1,P,\mathbf{\Phi}}^{\pi^{\star}}(s')\right)_+\notag\\
        &\qquad \leq \underbrace{\left\| \phi_i(s_h,a_h)\mathbf{1}_i \right\|_{\boldsymbol{\Lambda}_{h,\alpha}^{-1}}}_{\text{Term (i)}}\cdot \underbrace{\left\| \int_{\mathcal{S}}\left(\boldsymbol{\mu}_{h}^{\star}(\mathrm{d}s') - \boldsymbol{\mu}_{h}(\mathrm{d}s') \right) \left(\lambda - V_{h+1,P,\mathbf{\Phi}}^{\pi^{\star}}(s')\right)_+ \right\|_{\boldsymbol{\Lambda}_{h,\alpha}}}_{\text{Term (ii)}}.\label{eq: drmdp tv term 3}
    \end{align}
    Following the same argument as \eqref{eq: drmdp term 5+}, \eqref{eq: drmdp term 5++}, and \eqref{eq: drmdp term 5+++}, using the fact that $(\lambda - V_{h+1,P,\mathbf{\Phi}}^{\pi^{\star}}(s'))_+\in\mathcal{V}$ with $\mathcal{V}$ in \eqref{eq: function class V tv}, we can derive that with probability at least $1-\delta$, 
    \begin{align}\label{eq: drmdp tv term 4}
        \mathrm{Term (ii)}^2 \leq  3\xi
    \end{align}
    Now by combining \eqref{eq: drmdp tv term 1}, \eqref{eq: drmdp tv term 3}, \eqref{eq: drmdp tv term 4}, we can conclude that with probability at least $1-\delta$, 
    \begin{align}
        &\inf_{\tilde{P}_h\in\mathbf{\Phi}(P_h^{\star})}\mathbb{E}_{s'\sim \tilde{P}_h(\cdot|s_h,a_h)}[V_{h+1, P, \mathbf{\Phi}}^{\pi^{\star}}(s')] - \inf_{\tilde{P}_h\in\mathbf{\Phi}(P_h)}\mathbb{E}_{s'\sim \tilde{P}_h(\cdot|s_h,a_h)}[V_{h+1, P, \mathbf{\Phi}}^{\pi^{\star}}(s')] \notag\\
        &\qquad\leq   \sum_{i=1}^d  \sup_{0\leq \lambda_i \leq H}\left\{\left\| \phi_i(s_h,a_h)\mathbf{1}_i \right\|_{\boldsymbol{\Lambda}_{h,\alpha}^{-1}}\cdot \sqrt{3\xi}\right\}\leq 2\sqrt{\xi} \cdot \sum_{i=1}^d\left\| \phi_i(s_h,a_h)\mathbf{1}_i \right\|_{\boldsymbol{\Lambda}_{h,\alpha}^{-1}},\label{eq: drmdp tv term 5}
    \end{align}
    for any step $h\in[H]$, $(s_h,a_h)\in\mathcal{S}\times\mathcal{A}$, and $P_h\in\hat{\mathcal{P}}_h$.
    Now using the same argument as in the proof of Theorem \ref{thm: subopt general}, using Condition \ref{cond: accuracy}, we can derive that with probability at least $1-\delta$,
    \begin{align}
        \mathrm{SubOpt}(\hat{\pi};s_1)&\leq \sup_{P\in\hat{\mathcal{P}}}\sum_{h=1}^H\mathbb{E}_{(s_h,a_h)\sim d^{\pi^{\star}}_{P^{\pi^{\star},\dagger},h}}\left[\inf_{\tilde{P}_h\in\mathbf{\Phi}(P_h^{\star})}\mathbb{E}_{s'\sim \tilde{P}_h(\cdot|s_h,a_h)}[V_{h+1, P, \mathbf{\Phi}}^{\pi^{\star}}(s')]\right.\notag\\
        &\qquad\qquad\left.- \inf_{\tilde{P}_h\in\mathbf{\Phi}(P_h)}\mathbb{E}_{s'\sim \tilde{P}_h(\cdot|s_h,a_h)}[V_{h+1, P, \mathbf{\Phi}}^{\pi^{\star}}(s')]\right]\notag\\
        %&\leq \frac{d^2H^2\exp(H/\underline{\lambda})}{\rho\underline{c}^{1/2}}\cdot \sqrt{\frac{C_3\log(C_4 ndH/\delta)}{n}},
        &\leq 2\sqrt{\xi}\cdot \sum_{h=1}^H\sum_{i=1}^d\mathbb{E}_{(s_h,a_h)\sim d^{\pi^{\star}}_{P^{\pi^{\star},\dagger},h}}\left[\left\| \phi_i(s_h,a_h)\mathbf{1}_i \right\|_{\boldsymbol{\Lambda}_{h,\alpha}^{-1}}\right],\label{eq: drmdp tv term 6}
    \end{align}
    where in the last inequality we apply \eqref{eq: drmdp tv term 5}. 
    Here $P^{\pi^{\star},\dagger}_h$ is some transition kernel chosen from $\boldsymbol{\Phi}(P_h^{\star})$.
    Now we use the same argument as \eqref{eq: drmdp term 9} and \eqref{eq: drmdp term 10} to upper bound the right hand side of \eqref{eq: drmdp tv term 6} using Assumption \ref{ass: partial coverage cov}, which gives that,
    \begin{align}
        \sum_{i=1}^d\mathbb{E}_{(s_h,a_h)\sim d^{\pi^{\star}}_{P^{\pi^{\star},\dagger},h}}\left[\left\| \phi_i(s_h,a_h)\mathbf{1}_i \right\|_{\boldsymbol{\Lambda}_{h,\alpha}^{-1}}\right] &\leq \frac{d}{c^{\dagger}}.\label{eq: drmdp tv term 7}
    \end{align}
    Therefore, by combining \eqref{eq: drmdp tv term 6} and \eqref{eq: drmdp tv term 7}, we have that with probability at least $1-\delta$,
    \begin{align*}
        \mathrm{SubOpt}(\hat{\pi};s_1)&\leq 2\sqrt{\xi}\cdot \sum_{h=1}^H\frac{d}{c^{\dagger}} = \frac{2d\sqrt{\xi}\cdot H}{c^{\dagger}}.
    \end{align*}
    Using the definition of $\xi$, we can finally derive that with probability at least $1-\delta$,
    \begin{align*}
        \mathrm{SubOpt}(\hat{\pi};s_1) \leq \frac{d^2H^2}{c^{\dagger}}\cdot\sqrt{\frac{C_1'\log(C_2'ndH/\delta)}{n}}.
    \end{align*}
    This finishes the proof of Theorem \ref{cor: suboptimality drlmdp} under TV-distance.
\end{proof}

\subsection{Proof of Lemma \ref{lem: drlmdp estimator} and Lemma \ref{lem: drlmdp estimator tv}} \label{subsec: proof lem drlmdp estimator}

\begin{proof}[Proof of Lemma \ref{lem: drlmdp estimator}]
    The proof of Lemma \ref{lem: drlmdp estimator} follows from the main proofs in Section 8 of \cite{agarwal2019reinforcement} and the covering number of the function class $\mathcal{V}$ (Lemma \ref{lem: covering number V}).
    Denote $\mathcal{C}_{\mathcal{V},\epsilon}$ as an $\epsilon$-cover of the function class $\mathcal{V}$ under $\|\cdot\|_\infty$. 
    Following the exact same argument of Lemma 8.7 in \cite{agarwal2019reinforcement}, we can derive that with probability at least $1-\delta$, for any $h$ and $v\in\mathcal{C}_{\mathcal{V},\epsilon}$.
    \begin{align}
        &\left\|\sum_{\tau=1}^n\boldsymbol{\phi}(s_h^{\tau},a_h^{\tau})\left(\int_{\mathcal{S}}P_h^{\star}(\mathrm{d}s'|s_h^\tau,a_h^\tau)v(s') - v(s_{h+1}^{\tau})\right)\right\|_{\boldsymbol{\Lambda}_{h,\alpha}^{-1}}^2 \notag \\
        &\qquad \leq 9n\cdot\left(\log(H/\delta) + \log(|\mathcal{C}_{\mathcal{V},\epsilon}|) + d\log(1+N)\right),\label{eq: proof drlmdp estimator 0}
    \end{align}
    where we have taken $\alpha = 1$, which we will keep in the following.
    For any function $v\in\mathcal{V}$, take $\hat{v}\in\mathcal{C}_{\mathcal{V},\epsilon}$ such that $\|v - \hat{v}\|_{\infty}\leq \epsilon$. 
    Then we have that 
    \begin{align}
        &\left\|\sum_{\tau=1}^n\boldsymbol{\phi}(s_h^{\tau},a_h^{\tau})\left(\int_{\mathcal{S}}P_h^{\star}(\mathrm{d}s'|s_h^\tau,a_h^\tau)v(s') - v(s_{h+1}^{\tau})\right)\right\|_{\boldsymbol{\Lambda}_{h,\alpha}^{-1}}^2 \notag \\
        &\qquad \leq 2\left\|\sum_{\tau=1}^n\boldsymbol{\phi}(s_h^{\tau},a_h^{\tau})\left(\int_{\mathcal{S}}P_h^{\star}(\mathrm{d}s'|s_h^\tau,a_h^\tau)\hat{v}(s') - \hat{v}(s_{h+1}^{\tau})\right)\right\|_{\boldsymbol{\Lambda}_{h,\alpha}^{-1}}^2 \notag \\
        &\qquad \qquad + 2\left\|\sum_{\tau=1}^n\boldsymbol{\phi}(s_h^{\tau},a_h^{\tau})\left(\int_{\mathcal{S}}P_h^{\star}(\mathrm{d}s'|s_h^\tau,a_h^\tau)(\hat{v} - v)(s') - (\hat{v} - v)(s_{h+1}^{\tau})\right)\right\|_{\boldsymbol{\Lambda}_{h,\alpha}^{-1}}^2 \notag \\
        &\qquad \leq 18n\cdot\left(\log(H/\delta) + \log(|\mathcal{C}_{\mathcal{V},\epsilon}|) + d\log(1+n)\right) + 8\epsilon^2n^2.\label{eq: proof drlmdp estimator 1}
    \end{align}
    Now we apply the definition of $\hat{P}_h$ in \eqref{eq: hat P closed} and we can then derive that 
    \begin{align}
        &\left|\int_{\mathcal{S}}\big(P_h^{\star}(\mathrm{d}s'|s,a) - \hat{P}_h(\mathrm{d}s'|s,a)v(s')\big)\right|^2 \notag \\
        &\qquad = \left|\boldsymbol{\phi}(s,a)^\top\left(\int_{\mathcal{S}}\boldsymbol{\mu}^{\star}(\mathrm{d}s')v(s') - \frac{1}{n}\sum_{\tau=1}^n\boldsymbol{\Lambda}_{h,\alpha}^{-1}\boldsymbol{\phi}(s_h^\tau,a_h^\tau)v(s_{h+1}^\tau)\right)\right|^2\notag \\
        &\qquad = \left|\boldsymbol{\phi}(s,a)^\top\boldsymbol{\Lambda}_{h,\alpha}^{-1}\left(\boldsymbol{\Lambda}_{h,\alpha}\int_{\mathcal{S}}\boldsymbol{\mu}^{\star}(\mathrm{d}s')v(s') - \frac{1}{n}\sum_{\tau=1}^n\boldsymbol{\phi}(s_h^\tau,a_h^\tau)v(s_{h+1}^\tau)\right)\right|^2\notag \\
        &\qquad = \left|\boldsymbol{\phi}(s,a)^\top\boldsymbol{\Lambda}_{h,\alpha}^{-1}\left(\frac{1}{n}\int_{\mathcal{S}}\boldsymbol{\mu}_h^{\star}(\mathrm{d}s')v(s') + \frac{1}{n}\sum_{\tau=1}^n\boldsymbol{\phi}(s,a)\int_{\mathcal{S}}P_h^{\star}(\mathrm{d}s'|s_h^{\tau},a_h^{\tau})v(s')  - \frac{1}{n}\sum_{\tau=1}^n\boldsymbol{\phi}(s_h^\tau,a_h^\tau)v(s_{h+1}^\tau)\right)\right|^2 \notag \\
        &\qquad \leq \frac{2}{n^2}\cdot\|\boldsymbol{\phi}(s,a)\|_{\boldsymbol{\Lambda}_{h,\alpha}^{-1}}^2 \cdot \left\|\int_{\mathcal{S}}\boldsymbol{\mu}^{\star}(\mathrm{d}s')v(s')\right\|_{\boldsymbol{\Lambda}_{h,\alpha}^{-1}}^2 \notag \\
        &\qquad \qquad + \frac{2}{n^2}\cdot \|\boldsymbol{\phi}(s,a)\|_{\boldsymbol{\Lambda}_{h,\alpha}^{-1}}^2\cdot\left\|\sum_{\tau=1}^n\boldsymbol{\phi}(s_h^{\tau},a_h^{\tau})\left(\int_{\mathcal{S}}P_h^{\star}(\mathrm{d}s'|s_h^\tau,a_h^\tau)v(s') - v(s_{h+1}^{\tau})\right)\right\|_{\boldsymbol{\Lambda}_{h,\alpha}^{-1}}^2.\label{eq: proof drlmdp estimator 2}
    \end{align}
    On the one hand, the first term in the right hand side of \eqref{eq: proof drlmdp estimator 2} is bounded by 
    \begin{align}
        \!\!\!\frac{2}{n^2}\cdot\|\boldsymbol{\phi}(s,a)\|_{\boldsymbol{\Lambda}_{h,\alpha}^{-1}}^2 \cdot \left\|\int_{\mathcal{S}}\boldsymbol{\mu}^{\star}(\mathrm{d}s')v(s')\right\|_{\boldsymbol{\Lambda}_{h,\alpha}^{-1}}^2
        &\leq \frac{2}{n}\cdot \|\boldsymbol{\phi}(s,a)\|_{\boldsymbol{\Lambda}_{h,\alpha}^{-1}}^2 \cdot \left\|\int_{\mathcal{S}}\boldsymbol{\mu}^{\star}(\mathrm{d}s')v(s')\right\|_{2}^2\leq \frac{2d}{n} \cdot\|\boldsymbol{\phi}(s,a)\|_{\boldsymbol{\Lambda}_{h,\alpha}^{-1}}^2,\label{eq: proof drlmdp estimator 3}
    \end{align}
    where we use the fact that $\boldsymbol{\Lambda}_{h,\alpha}\succeq (1/n)\cdot\boldsymbol{I}_d$ and $\|v(\cdot)\|_{\infty}\leq 1$ for any $v\in\mathcal{V}$.
    On the other hand, the second term in the right hand side of \eqref{eq: proof drlmdp estimator 2} is bounded by 
    \begin{align*}
        &\frac{2}{n^2}\cdot \|\boldsymbol{\phi}(s,a)\|_{\boldsymbol{\Lambda}_{h,\alpha}^{-1}}^2\cdot\left\|\sum_{\tau=1}^n\boldsymbol{\phi}(s_h^{\tau},a_h^{\tau})\left(\int_{\mathcal{S}}P_h^{\star}(\mathrm{d}s'|s_h^\tau,a_h^\tau)v(s') - v(s_{h+1}^{\tau})\right)\right\|_{\boldsymbol{\Lambda}_{h,\alpha}^{-1}}^2 \\
        &\qquad \leq \left(\frac{36}{n}\cdot\left(\log(H/\delta) + \log(|\mathcal{C}_{\mathcal{V},\epsilon}|) + d\log(1+n)\right) + 16\epsilon^2\right) \cdot \|\boldsymbol{\phi}(s,a)\|_{\boldsymbol{\Lambda}_{h,\alpha}^{-1}}^2,
    \end{align*}
    where we have applied \eqref{eq: proof drlmdp estimator 1}. 
    Now taking $\epsilon = 1/\sqrt{n}$, applying Lemma \ref{lem: covering number V} to bound the covering number of $\mathcal{V}$, we can further derive that,
    \begin{align}
        &\frac{2}{n^2}\cdot \|\boldsymbol{\phi}(s,a)\|_{\boldsymbol{\Lambda}_{h,\alpha}^{-1}}^2\cdot\left\|\sum_{\tau=1}^n\boldsymbol{\phi}(s_h^{\tau},a_h^{\tau})\left(\int_{\mathcal{S}}P_h^{\star}(\mathrm{d}s'|s_h^\tau,a_h^\tau)v(s') - v(s_{h+1}^{\tau})\right)\right\|_{\boldsymbol{\Lambda}_{h,\alpha}^{-1}}^2 \notag \\
        &\qquad \leq \frac{36}{n}\cdot\left(\log(H/\delta) + d\log(1+4\sqrt{n}Hd/(\underline{\lambda})) + \log(1+4\sqrt{n}Hd/(\underline{\lambda}^2\rho)) + d\log(1+n)\right)\cdot \|\boldsymbol{\phi}(s,a)\|_{\boldsymbol{\Lambda}_{h,\alpha}^{-1}}^2 \notag \\
        &\qquad \qquad  + \frac{16}{n} \cdot \|\boldsymbol{\phi}(s,a)\|_{\boldsymbol{\Lambda}_{h,\alpha}^{-1}}^2,\notag \\
        &\qquad \leq \frac{C_1d\big(\log(1+C_2nH/\delta) + \log(1+C_3ndH/(\rho\underline{\lambda}^2))\big)}{n}\cdot\|\boldsymbol{\phi}(s,a)\|_{\boldsymbol{\Lambda}_{h,\alpha}^{-1}}^2,\label{eq: proof drlmdp estimator 4}
    \end{align}
    where $C_1,C_2,C_3>0$ are three constants.
    Finally, by combining \eqref{eq: proof drlmdp estimator 2}, \eqref{eq: proof drlmdp estimator 3}, and \eqref{eq: proof drlmdp estimator 4}, we can conclude that with probability at least $1-\delta$, for each step $h\in[H]$, 
    \begin{align*}
        \sup_{v\in\mathcal{V}}\left|\int_{\mathcal{S}}\big(P^{\star}_h(\mathrm{d}s'|s,a)-\hat{P}_h(\mathrm{d}s'|s,a)\big)v(s') \right|^2 \leq C_1'\cdot \|\boldsymbol{\phi}(s,a)\|_{\boldsymbol{\Lambda}_{h,\alpha}^{-1}}^2\cdot\frac{d\big(\log(1+C_2nH/\delta) + \log(1+C_3ndH/(\rho\underline{\lambda}^2))\big)}{n}.
    \end{align*}
    where $C_1'$ is another constant.
    This finishes the proof of Lemma \ref{lem: drlmdp estimator}.
\end{proof}

\begin{proof}[Proof of Lemma \ref{lem: drlmdp estimator tv}]
    The proof of Lemma \ref{lem: drlmdp estimator tv} follows the same argument as proof of Lemma \ref{lem: drlmdp estimator}, except a different covering number of the function class $\mathcal{V}$ which we show in the following.
    Using the same argument as the proof of Lemma \ref{lem: drlmdp estimator} (except that now $\|v(\cdot)\|_{\infty}\leq H$), with probability at least $1-\delta$, for any $v\in\mathcal{V}$,
    \begin{align}
        &\left|\int_{\mathcal{S}}\big(P_h^{\star}(\mathrm{d}s'|s,a) - \hat{P}_h(\mathrm{d}s'|s,a)v(s')\big)\right|^2 \notag \\
        &\qquad \leq \frac{2H^2}{n^2}\cdot\|\boldsymbol{\phi}(s,a)\|_{\boldsymbol{\Lambda}_{h,\alpha}^{-1}}^2 \cdot \left\|\int_{\mathcal{S}}\boldsymbol{\mu}^{\star}(\mathrm{d}s')v(s')\right\|_{\boldsymbol{\Lambda}_{h,\alpha}^{-1}}^2 \notag \\
        &\qquad \qquad + \frac{2H^2}{n^2}\cdot \|\boldsymbol{\phi}(s,a)\|_{\boldsymbol{\Lambda}_{h,\alpha}^{-1}}^2\cdot\left\|\sum_{\tau=1}^n\boldsymbol{\phi}(s_h^{\tau},a_h^{\tau})\left(\int_{\mathcal{S}}P_h^{\star}(\mathrm{d}s'|s_h^\tau,a_h^\tau)v(s') - v(s_{h+1}^{\tau})\right)\right\|_{\boldsymbol{\Lambda}_{h,\alpha}^{-1}}^2\notag \\
        &\qquad \leq  H^2\cdot\left(\frac{36}{n}\cdot\left(\log(H/\delta) + \log(|\mathcal{C}_{\mathcal{V},\epsilon}|) + d\log(1+n)\right) + 16\epsilon^2 + \frac{2d}{n}\right) \cdot \|\boldsymbol{\phi}(s,a)\|_{\boldsymbol{\Lambda}_{h,\alpha}^{-1}}^2, \label{eq: proof drlmdp estimator tv 1}
    \end{align}
    where $\mathcal{C}_{\mathcal{V},\epsilon}$ is an $\epsilon$-covering of the function class $\mathcal{V}$ defined in \eqref{eq: function class V tv}.
    Now taking $\epsilon = 1/\sqrt{n}$, applying Lemma~\ref{lem: covering number V tv} to bound the covering number of $\mathcal{V}$, we can further derive that,
    \begin{align}
        &\sup_{v\in\mathcal{V}}\left|\int_{\mathcal{S}}\big(P_h^{\star}(\mathrm{d}s'|s,a) - \hat{P}_h(\mathrm{d}s'|s,a)v(s')\big)\right|^2   \notag \\
        &\qquad  \leq H^2\cdot \|\boldsymbol{\phi}(s,a)\|_{\boldsymbol{\Lambda}_{h,\alpha}^{-1}}^2 \cdot \left(\frac{36}{n}\cdot\left(\log(H/\delta) + d\log(1+4\sqrt{n}Hd) + \log(1+4\sqrt{n}H) + d\log(1+n)\right) + \frac{16+2d}{n}\right) \notag \\
        &\qquad \leq C_1\cdot \|\boldsymbol{\phi}(s,a)\|_{\boldsymbol{\Lambda}_{h,\alpha}^{-1}}^2 \cdot\frac{dH^2\log(C_2ndH/\delta)}{n}. \label{eq: proof drlmdp estimator tv 2}
    \end{align}
    This finishes the proof of Lemma \ref{lem: drlmdp estimator tv}.
\end{proof}

\subsection{Other Lemmas}

\begin{lemma}[Covering number of $\mathcal{V}$: KL-divergence case]\label{lem: covering number V}
    The $\epsilon$-covering number of function class $\mathcal{V}$ defined in \eqref{eq: function class V kl} under $\|\cdot\|_{\infty}$-norm is bounded by 
    \begin{align*}
        \log(\mathcal{N}(\epsilon, \mathcal{V}, \|\cdot\|_{\infty})) \leq d\log(1+4Hd/(\underline{\lambda}\epsilon)) + \log(1+4H^2d/(\underline{\lambda}^2\rho\epsilon)).
    \end{align*}
\end{lemma}

\begin{proof}[Proof of Lemma \ref{lem: covering number V}]
    Consider any two pairs of parameters $(\boldsymbol{w}, \lambda)$ and $(\hat{\boldsymbol{w}}, \hat{\lambda})$, and denote the functions they induce as $v$ and $\hat{v}$.
    Then we have that 
    \begin{align*}
        |v(s) - \hat{v}(s)| = \left|\exp\left\{-\left\{\max_{a\in\mathcal{A}}\boldsymbol{\phi}(s,a)^\top\boldsymbol{w}/\lambda\right\}_+\right\} - \exp\left\{-\left\{\max_{a\in\mathcal{A}}\boldsymbol{\phi}(s,a)^\top\hat{\boldsymbol{w}}/\hat{\lambda}\right\}_+\right\}\right|
    \end{align*} 
    Using the fact that, for any $x,y>0$, $\exp(-x) - \exp(-y) = \exp(-\zeta(x,y))\cdot(y-x)$ for some $\zeta(x,y)$ between $x$ and $y$, we know that 
    \begin{align*}
        &|v(s) - \hat{v}(s)| \\
        &\qquad \leq \exp\left\{-\zeta\left(\left\{\max_{a\in\mathcal{A}}\boldsymbol{\phi}(s,a)^\top\boldsymbol{w}/\lambda\right\}_+, \left\{\max_{a\in\mathcal{A}}\boldsymbol{\phi}(s,a)^\top\hat{\boldsymbol{w}}/\hat{\lambda}\right\}_+\right)\right\}\cdot\left|\max_{a\in\mathcal{A}}\boldsymbol{\phi}(s,a)^\top\boldsymbol{w}/\lambda-\max_{a\in\mathcal{A}}\boldsymbol{\phi}(s,a)^\top\hat{\boldsymbol{w}}/\hat{\lambda}\right|\\
        &\qquad \leq \left|\max_{a\in\mathcal{A}}\left\{\boldsymbol{\phi}(s,a)^\top\boldsymbol{w}/\lambda-\boldsymbol{\phi}(s,a)^\top\hat{\boldsymbol{w}}/\hat{\lambda}\right\}\right|\\
        &\qquad = \left|\max_{a\in\mathcal{A}}\left\{\boldsymbol{\phi}(s,a)^\top\boldsymbol{w}/\lambda- \boldsymbol{\phi}(s,a)^\top\hat{\boldsymbol{w}}/\lambda + \boldsymbol{\phi}(s,a)^\top\hat{\boldsymbol{w}}/\lambda - \boldsymbol{\phi}(s,a)^\top\hat{\boldsymbol{w}}/\hat{\lambda}\right\}\right|.
    \end{align*}
    Notice that $\|\boldsymbol{\phi}(s,a)\|_2\leq \sqrt{d}$ (because $\sum_{i=1}^d\phi_i(s,a)=1$), $\|\hat{\boldsymbol{w}}\|_2\leq H\sqrt{d}$, and $\lambda,\hat{\lambda}\geq \underline{\lambda}$, we have,
    \begin{align*}
        &\left|\boldsymbol{\phi}(s,a)^\top\boldsymbol{w}/\lambda- \boldsymbol{\phi}(s,a)^\top\hat{\boldsymbol{w}}/\lambda + \boldsymbol{\phi}(s,a)^\top\hat{\boldsymbol{w}}/\lambda - \boldsymbol{\phi}(s,a)^\top\hat{\boldsymbol{w}}/\hat{\lambda}\right| \\
        &\qquad \leq \left|\lambda^{-1}\boldsymbol{\phi}(s,a)^\top(\boldsymbol{w} - \hat{\boldsymbol{w}})\right| + \left|\lambda^{-1}\hat{\lambda}^{-1}\boldsymbol{\phi}(s,a)^\top\hat{\boldsymbol{w}}(\lambda - \hat{\lambda})\right| \\
        &\qquad \leq \underline{\lambda}^{-1}\sqrt{d}\cdot\|\boldsymbol{w} - \hat{\boldsymbol{w}}\|_2 + \underline{\lambda}^{-2}Hd\cdot|\lambda - \hat{\lambda}|.
    \end{align*}
    Thus we conclude that to form an $\epsilon$-cover of $\mathcal{V}$ under $\|\cdot\|_\infty$-norm, it suffices to consider the product of an $\underline{\lambda}\epsilon/(2\sqrt{d})$-cover of $\{\boldsymbol{w}:\|\boldsymbol{w}\|_2\leq H\sqrt{d}\}$ under $\|\cdot\|_2$-norm and an $\underline{\lambda}^2\epsilon/(2Hd)$-cover of the interval $[\underline{\lambda},H/\rho]$.
    Therefore, we can derive that 
    \begin{align*}
        \log(\mathcal{N}(\epsilon, \mathcal{V}, \|\cdot\|_{\infty})) \leq d\log(1+4Hd/(\underline{\lambda}\epsilon)) + \log(1+4H^2d/(\underline{\lambda}^2\rho\epsilon)).
    \end{align*}
    This finishes the proof of Lemma \ref{lem: covering number V}.
\end{proof}

\begin{lemma}[Covering number of $\mathcal{V}$: TV-distance case]\label{lem: covering number V tv}
    The $\epsilon$-covering number of function class $\mathcal{V}$  defined in \eqref{eq: function class V tv} under $\|\cdot\|_{\infty}$-norm is bounded by 
    \begin{align*}
        \log(\mathcal{N}(\epsilon, \mathcal{V}, \|\cdot\|_{\infty})) \leq d\log(1+4Hd/\epsilon) + \log(1+4H/\epsilon).
    \end{align*}
\end{lemma}

\begin{proof}[Proof of Lemma \ref{lem: covering number V tv}]
    Consider any two pairs of parameters $(\boldsymbol{w}, \lambda)$ and $(\hat{\boldsymbol{w}}, \hat{\lambda})$, and denote the functions they induce as $v$ and $\hat{v}$.
    Then we have that,
    \begin{align*}
        |v(s) - \hat{v}(s)| &= \left|\left(\lambda - \max_{a\in\mathcal{A}}\boldsymbol{\phi}(s,a)^\top\boldsymbol{w}\right)_+ - \left(\hat{\lambda} - \max_{a\in\mathcal{A}}\boldsymbol{\phi}(s,a)^\top\hat{\boldsymbol{w}}\right)_+\right| \\
        & \leq |\lambda - \hat{\lambda}| + \left|\max_{a\in\mathcal{A}}\boldsymbol{\phi}(s,a)^\top\boldsymbol{w} - \max_{a\in\mathcal{A}}\boldsymbol{\phi}(s,a)^\top\hat{\boldsymbol{w}}\right| \\
        &\leq |\lambda - \hat{\lambda}| + \sup_{(s,a)\in\mathcal{S}\times\mathcal{A}}\|\boldsymbol{\phi}(s,a)\|_2\cdot\|\boldsymbol{w} - \hat{\boldsymbol{w}}\|_2 \\
        &\leq |\lambda - \hat{\lambda}| + \sqrt{d}\cdot\|\boldsymbol{w} - \hat{\boldsymbol{w}}\|_2
    \end{align*}
    Thus we conclude that to form an $\epsilon$-cover of $\mathcal{V}$ under $\|\cdot\|_\infty$-norm, it suffices to consider the product of an $\epsilon/(2\sqrt{d})$-cover of $\{\boldsymbol{w}:\|\boldsymbol{w}\|_2\leq H\sqrt{d}\}$ under $\|\cdot\|_2$-norm and an $\epsilon/2$-cover of the interval $[0,H]$.
    Therefore, we can derive that 
    \begin{align*}
        \log(\mathcal{N}(\epsilon, \mathcal{V}, \|\cdot\|_{\infty})) \leq d\log(1+4Hd/\epsilon) + \log(1+4H/\epsilon).
    \end{align*}
    This finishes the proof of Lemma \ref{lem: covering number V tv}.
\end{proof}

%% file: tex/appendix/mle.tex
\section{Analysis of Maximum Likelihood Estimator}

\begin{lemma}[MLE estimator guarantee: infinite model space]\label{lem: mle guarantee sarmdp}
    The maximum likelihood estimator procedure given by \eqref{eq: mle sarmdp} and \eqref{eq: confidence sarmdp} for $\mathcal{S}\times\mathcal{A}$-rectangular robust MDP with tuning parameter $\xi$ given by Proposition \ref{prop: sarmdp estimation} satisfies that w.p. at least $1-\delta$,
    \begin{enumerate}
        \item $P_h^{\star}\in\hat{\mathcal{P}}_h$ for any step $h\in[H]$.
        \item for any step $h\in[H]$ and $P_h\in\hat{\mathcal{P}}_h$, it holds that 
        \begin{align*}
            \mathbb{E}_{(s_h,a_h)\sim d^{\mathrm{b}}_{P^{\star},h}}[\|P_h(\cdot|s_h,a_h) - P_h^{\star}(\cdot|s_h,a_h)\|_{\mathrm{TV}}^2]\leq \frac{C_1\log(C_2H\mathcal{N}_{[]}(1/n^2,\mathcal{P}_{\mathrm{M}},\|\cdot\|_{1,\infty})/\delta)}{n}.
        \end{align*}
        for some absolute constant $C_1,C_2>0$. Here $d^{\mathrm{b}}_{P^{\star},h}$ is the state-action visitation measure induced by the behavior policy $\pi^{\mathrm{b}}$ and transition kernel $P^{\star}$.
    \end{enumerate}
\end{lemma}

\begin{proof}[Proof of Lemma \ref{lem: mle guarantee sarmdp}]
    See Appendix \ref{subsec: proof mle guarantee sarmdp} for a detailed proof.
\end{proof}

\begin{lemma}[MLE estimator guarantee: factored model space]\label{lem: mle guarantee safrmdp}
    The maximum likelihood estimator procedure given by \eqref{eq: mle safrmdp} and \eqref{eq: confidence safrmdp} for $\mathcal{S}\times\mathcal{A}$-rectangular robust factored MDP with tuning parameter $\xi_i$ given by Proposition \ref{prop: safrmdp estimation} satisfies that w.p. at least $1-\delta$,
    \begin{enumerate}
        \item $P_h^{\star}\in\hat{\mathcal{P}}_h$ for any step $h\in[H]$.
        \item for any step $h\in[H]$, $P_h\in\hat{\mathcal{P}}_h$, and any factor $i\in[d]$ it holds that 
        \begin{align*}
            \mathbb{E}_{(s_h[\mathrm{pa}_i],a_h)\sim d^{\mathrm{b}}_{P^{\star},h}}[\|P_{h,i}(\cdot|s_h[\mathrm{pa}_i],a_h) - P_{h,i}^{\star}(\cdot|s_h[\mathrm{pa}_i],a_h)\|_{\mathrm{TV}}^2]\leq \frac{C_1|\mathcal{O}|^{1+|\mathrm{pa}_i|}|\mathcal{A}|\log(C_2ndH/\delta)}{n}.
        \end{align*}
        for some absolute constant $C_1,C_2>0$. Here $d^{\mathrm{b}}_{P^{\star},h}$ is the state-action visitation measure induced by the behavior policy $\pi^{\mathrm{b}}$ and transition kernel $P^{\star}$.
    \end{enumerate}
\end{lemma}

\begin{proof}[Proof of Lemma \ref{lem: mle guarantee safrmdp}]
    See Appendix \ref{subsec: proof mle guarantee safrmdp} for a detailed proof.
\end{proof}

\subsection{Proof of Lemma \ref{lem: mle guarantee sarmdp}}\label{subsec: proof mle guarantee sarmdp}

In this section, we establish the proof of Lemma \ref{lem: mle guarantee sarmdp}. 
We firstly introduce several notations.
For any function $f:\mathcal{S}\times\mathcal{A}\mapsto\mathbb{R}$, we denote 
\begin{align*}
    \mathbb{E}_{\mathbb{D}_h}[f] = \frac{1}{n}\sum_{\tau=1}^nf(s_h^\tau,a_h^\tau).
\end{align*}

\begin{proof}[Proof of Lemma \ref{lem: mle guarantee sarmdp}]
    We follow the proof of similar MLE guarantees in \cite{uehara2021pessimistic} and \cite{liu2022welfare}.
    We begin with proving the first conclusion of Lemma \ref{lem: mle guarantee sarmdp}, i.e., $P_h^{\star}\in\hat{\mathcal{P}}_h$ for each step $h\in[H]$.
    For notational simplicity, we define 
    \begin{align}
        g_h(P)(s,a) = \|P(\cdot|s,a) - P_h^{\star}(\cdot|s,a)\|_{1}^2,\quad \forall P\in\mathcal{P}_{\mathrm{M}}.
    \end{align}
    To prove the first conclusion, it suffices to show that 
    \begin{align}\label{eq: proof mle 0}
        \mathbb{E}_{\mathbb{D}_h}[g_h(\hat{P}_h)]\leq \xi,\quad \forall h\in[H].
    \end{align}
    where $\hat{P}_h$ is the MLE estimator given in \eqref{eq: mle sarmdp} and the parameter $\xi$ is given by Proposition \ref{prop: sarmdp estimation}.
    To this end, we first invoke Lemma \ref{lem: mle}, which gives that with probability at least $1-\delta$,
    \begin{align}\label{eq: proof mle 1}
        \mathbb{E}_{d_{P^{\star},h}^{\mathrm{b}}}[g_h(\hat{P}_h)]\leq c_1\big(\zeta_h+\sqrt{\log(c_2 / \delta) / n}\big)^2,
    \end{align}
    for some absolute constants $c_1,c_2>0$. 
    Here $\zeta_h$ is a solution to the inequality $\sqrt{n} \epsilon^2 \geq c_0 G_h(\epsilon)$ w.r.t $\epsilon$, with some carefully chosen function $G_h$ which is specified in Lemma \ref{lem: mle}.
    As proved in Lemma \ref{lem: choice of G and zeta}, choosing $G_h(\epsilon) = (\epsilon - \epsilon^2/2)\sqrt{\log(\mathcal{N}_{[]}(\epsilon^4/2,\mathcal{P}_{\mathrm{M}},\|\cdot\|_{1,\infty}))}$ and $\zeta_h = c_3\sqrt{\log(\mathcal{N}_{[]}(1/n^2,\mathcal{P}_{\mathrm{M}},\|\cdot\|_{1,\infty}))/n}$
    for some absolute constant $c_3>0$ can satisfy the inequality and the requirements on $G_h$.
    Thus we can obtain from \eqref{eq: proof mle 1} that, with probability at least $1-\delta$,
    \begin{align}
        \mathbb{E}_{d_{P^{\star},h}^{\mathrm{b}}}[g_h(\hat{P}_h)]&\leq c_1\left(c_3\sqrt{\frac{\log(\mathcal{N}_{[]}(1/n^2,\mathcal{P}_{\mathrm{M}},\|\cdot\|_{1,\infty}))}{n}} +\sqrt{\frac{\log(c_2/\delta)}{n}}\right)^2 \notag\\
        &\leq \frac{c_1'\log(c_2'\mathcal{N}_{[]}(1/n^2,\mathcal{P}_{\mathrm{M}},\|\cdot\|_{1,\infty})/\delta)}{n},\label{eq: proof mle 2}
    \end{align}
    for some absolute constants $c_1',c_2'>0$. 
    Now to prove \eqref{eq: proof mle 0}, it suffices to relate the expectation w.r.t. dataset $\mathbb{D}_h$ and the expectation w.r.t. visitation measure $d_{P^{\star},h}^{\mathrm{b}}$.
    To bridge this gap, we invoke Lemma \ref{lem: bernstein 1}, which is a Bernstein style concentration inequality and gives that with probability at least $1-\delta$,
    \begin{align}\label{eq: proof mle 3}
        |\mathbb{E}_{\mathbb{D}_h}[g_h(\hat{P}_h)] - \mathbb{E}_{d_{P^{\star},h}^{\mathrm{b}}}[g_h(\hat{P}_h)]| \leq \frac{c_4\log(c_5\mathcal{N}_{[]}(1/n^2,\mathcal{P}_{\mathrm{M}},\|\cdot\|_{1,\infty})/\delta)}{n},
    \end{align}
    for some absolute constant $c_4>0$.
    Now combining \eqref{eq: proof mle 2} and \eqref{eq: proof mle 3}, we can obtain that, 
    \begin{align*}
        \mathbb{E}_{\mathbb{D}_h}[g_h(\hat{P}_h)] = \mathbb{E}_{\mathbb{D}_h}[g_h(\hat{P}_h)] - \mathbb{E}_{d_{P^{\star},h}^{\mathrm{b}}}[g_h(\hat{P}_h)] + \mathbb{E}_{d_{P^{\star},h}^{\mathrm{b}}}[g_h(\hat{P}_h)]
        \leq \frac{c_1''\log(c_2''\mathcal{N}_{[]}(1/n^2,\mathcal{P}_{\mathrm{M}},\|\cdot\|_{1,\infty})/\delta)}{n},
    \end{align*}
    for some absolute constants $c_1'',c_2''>0$. 
    Finally, taking a union bound over step $h\in[H]$ and rescaling $\delta$, we obtain that, with probability at least $1-\delta/2$, 
    \begin{align}\label{eq: proof mle 3+}
        \mathbb{E}_{\mathbb{D}_h}[g_h(\hat{P}_h)] \leq \frac{\tilde{C}_1\log(\tilde{C}_2H\mathcal{N}_{[]}(1/n^2,\mathcal{P}_{\mathrm{M}},\|\cdot\|_{1,\infty})/\delta)}{n} = \xi,\quad\forall h\in[H],
    \end{align}
    for some absolute constants $\tilde{C}_1,\tilde{C}_2>0$. This finishes the proof of the first conclusion of Lemma \ref{lem: mle guarantee sarmdp}.
    
    The following of the proof is to prove the second conclusion of Lemma \ref{lem: mle guarantee sarmdp}.
    With the notation of $g_h$, it suffices to prove that with probability at least $1-\delta/2$,
    \begin{align*}
        \sup_{h\in[H],P_h\in\hat{\mathcal{P}}_h}\mathbb{E}_{d_{P^{\star},h}^{\mathrm{b}}}[g_h(P_h)] \leq \frac{C_1\log(C_2H\mathcal{N}_{[]}(1/n^2,\mathcal{P}_{\mathrm{M}},\|\cdot\|_{1,\infty})/\delta)}{n},
    \end{align*}
    for some absolute constants $C_1,C_2>0$.
    To this end, for any step $h\in[H]$ and $P_h\in\hat{\mathcal{P}}_h$, consider the following decomposition of $\mathbb{E}_{d_{P^{\star},h}^{\mathrm{b}}}[g_h(P_h)]$,
    \begin{align}\label{eq: proof mle 4}
        \mathbb{E}_{d_{P^{\star},h}^{\mathrm{b}}}[g_h(P_h)] &= \mathbb{E}_{d_{P^{\star},h}^{\mathrm{b}}}[g_h(P_h)] - \mathbb{E}_{\mathbb{D}_h}[g_h(P_h)] + \mathbb{E}_{\mathbb{D}_h}[g_h(P_h)].
    \end{align}
    Note that the term $\mathbb{E}_{\mathbb{D}_h}[g_h(P_h)]$ in \eqref{eq: proof mle 4} satisfies, with probability at least $1-\delta/2$, 
    \begin{align}
        \mathbb{E}_{\mathbb{D}_h}[g_h(P_h)] &= \mathbb{E}_{\mathbb{D}_h}[\|P_h(\cdot|s,a) - P_h^{\star}(\cdot|s,a)\|_{1}^2] \notag\\
        &= \mathbb{E}_{\mathbb{D}_h}[\|P_h(\cdot|s,a) - \hat{P}_h(\cdot|s,a) + \hat{P}_h(\cdot|s,a) - P_h^{\star}(\cdot|s,a)\|_{1}^2] \notag\\
        &\leq 2\mathbb{E}_{\mathbb{D}_h}[\|P_h(\cdot|s,a) - \hat{P}_h(\cdot|s,a)\|_1^2] + 2 \mathbb{E}_{\mathbb{D}_h}[\|\hat{P}_h(\cdot|s,a) - P_h^{\star}(\cdot|s,a)\|_{1}^2]\notag \\
        &\leq 4\xi, \label{eq: proof mle 5}
    \end{align}
    where the last inequality follows from the definition of confidence region $\hat{\mathcal{P}}_h$ and the first conclusion of Lemma \ref{lem: mle guarantee sarmdp}, i.e., \eqref{eq: proof mle 3+}.
    Thus by taking \eqref{eq: proof mle 5} back into \eqref{eq: proof mle 4}, we obtain that,
    \begin{align}\label{eq: proof mle 6}
        \mathbb{E}_{d_{P^{\star},h}^{\mathrm{b}}}[g_h(P_h)] \leq 4\xi + \mathbb{E}_{d_{P^{\star},h}^{\mathrm{b}}}[g_h(P_h)] - \mathbb{E}_{\mathbb{D}_h}[g_h(P_h)].
    \end{align}
    Finally, invoking another Bernstein style concentration inequality (Lemma \ref{lem: bernstein 2}), we have that with probability at least $1-\delta$, 
    \begin{align}\label{eq: proof mle 7}
        \sup_{P_h\in\hat{\mathcal{P}}_h}|\mathbb{E}_{\mathbb{D}_h}[g_h(P_h)] - \mathbb{E}_{d_{P^{\star},h}^{\mathrm{b}}}[g_h(P_h)]| \leq \frac{c_6\log(c_7\mathcal{N}_{[]}(1/n^2,\mathcal{P}_{\mathrm{M}},\|\cdot\|_{1,\infty})/\delta)}{n}
    \end{align}
    Thus by combining \eqref{eq: proof mle 6} and \eqref{eq: proof mle 7}, taking a union bound over step $h\in[H]$, rescaling $\delta$, and using the definition of $\xi$, we can conclude that with probability at least $1-\delta/2$, 
    \begin{align*}
        \sup_{h\in[H],P_h\in\hat{\mathcal{P}}_h}\mathbb{E}_{d_{P^{\star},h}^{\mathrm{b}}}[g_h(P_h)] \leq \frac{C_1\log(C_2H\mathcal{N}_{[]}(1/n^2,\mathcal{P}_{\mathrm{M}},\|\cdot\|_{1,\infty})/\delta)}{n},
    \end{align*}
    for some absolute constants $C_1,C_2>0$.
    This finishes the proof of Lemma \ref{lem: mle guarantee sarmdp}.
\end{proof}

\subsection{Proof of Lemma \ref{lem: mle guarantee safrmdp}}\label{subsec: proof mle guarantee safrmdp}
\begin{proof}[Proof of Lemma \ref{lem: mle guarantee safrmdp}]
    This is a direct corollary of Lemma \ref{lem: mle guarantee sarmdp} in the finite state space case: for each factor $i\in[d]$, consider $\mathcal{O}$ as the finite state space and apply the upper bound of bracket number \eqref{eq: bracket number tabular} for finite state space case proved in Appendix \ref{subsec: proof sartmdp}.
    This proves Lemma \ref{lem: mle guarantee safrmdp}.
\end{proof}

%% file: tex/appendix/tech.tex
\section{Technical Lemmas}

\subsection{Lemmas for Maximum Likelihood Estimator}

In this section, we give technical lemmas for the maximum likelihood estimator.
We firstly introduce several notations which are also considered by \cite{uehara2021pessimistic} and \cite{liu2022welfare},
We define a localized model space $\overline{\mathcal{P}}_h(\epsilon)$ as 
\begin{align*}
    \overline{\mathcal{P}}_h(\epsilon) = \left\{P\in\overline{\mathcal{P}}_{\mathrm{M},h}:\mathbb{E}_{d_{P^{\star},h}^{\mathrm{b}}}[D_{\mathrm{Hellinger}}^2(P(\cdot|s,a)\|P_h^{\star}(\cdot|s,a))]\leq \epsilon^2\right\},
\end{align*}
where $D_{\mathrm{Hellinger}}(\cdot\|\cdot)$ is the Hellinger distance between two probability measures, and $\overline{\mathcal{P}}_{\mathrm{M},h}$ is called a modified space $\mathcal{P}_{\mathrm{M}}$, defined as 
$\overline{\mathcal{P}}_{\mathrm{M},h} = \{(P+P_h^{\star})/2:P\in\mathcal{P}_{\mathrm{M}}\}$.
Also, we define the entropy integral of $\overline{\mathcal{P}}_h(\epsilon)$ under the $\|\cdot\|_{2,d_{P^{\star},h}^{\mathrm{b}}}$-norm as 
\begin{align*}
    J_{\mathrm{B}}(\epsilon,\overline{\mathcal{P}}_h(\epsilon),\|\cdot\|_{2,d_{P^{\star},h}^{\mathrm{b}}}) = \max\left\{\epsilon, \int_{\epsilon^2/2}^\epsilon \sqrt{\log(\mathcal{N}_{[]}(u,\overline{\mathcal{P}}_{h}(\epsilon),\|\cdot\|_{2, d_{P^{\star},h}^{\mathrm{b}}}))}\mathrm{d}u\right\}.
\end{align*}

\begin{lemma}[MLE Gaurantee, \cite{van2000empirical}]\label{lem: mle}
    Take a function $G_h(\epsilon):[0,1] \rightarrow \mathbb{R}$ s.t. $G_h(\epsilon) \geq J_B(\epsilon, \overline{\mathcal{P}}_h(\epsilon),\|\cdot\|_{2, d_{P^{\star},h}^{\mathrm{b}}})$ and $G_h(\epsilon) / \epsilon^2$ non-increasing w.r.t $\epsilon$. 
    Then, letting $\zeta_h$ be a solution to $\sqrt{n} \epsilon^2 \geq c_0 G_h(\epsilon)$ w.r.t $\epsilon$, where $c_0$ is an absolute constant. With probability at least $1-\delta$, we have that 
    \begin{align*}
        \mathbb{E}_{d_{P^{\star},h}^{\mathrm{b}}}[\|\widehat{P}_h(\cdot|s, a)-P_h^{\star}(\cdot | s, a)\|_1^2] \leq c_1\big(\zeta_h+\sqrt{\log(c_2 / \delta) / n}\big)^2.
    \end{align*}
\end{lemma}

\begin{proof}[Proof of Lemma \ref{lem: mle}]
    We refer to Theorem 7.4 in \cite{van2000empirical} for a detailed proof.
\end{proof}

\begin{lemma}[Choice of $G_h(\epsilon)$ and $\zeta_h$ in Lemma \ref{lem: mle}]\label{lem: choice of G and zeta}
    In Lemma \ref{lem: mle}, we can choose $G_h(\epsilon)$ as 
    \begin{align*}
        G_h(\epsilon) = (\epsilon - \epsilon^2/2)\sqrt{\log(\mathcal{N}_{[]}(\epsilon^4/2,\mathcal{P}_{\mathrm{M}},\|\cdot\|_{1,\infty}))},
    \end{align*}
    In this case, $\zeta_h = c_0\sqrt{\log(\mathcal{N}_{[]}(1/n^2,\mathcal{P}_{\mathrm{M}},\|\cdot\|_{1,\infty}))/n}$ solves the inequality $\sqrt{n} \epsilon^2 \geq c_0 G_h(\epsilon)$ w.r.t $\epsilon$.
\end{lemma}

\begin{proof}[Proof of Lemma \ref{lem: choice of G and zeta}]
    We first check the conditions that $G_h$ should satisfy.
    By the choice of $G_h$,
    \begin{align*}
        G_h(\epsilon) &= (\epsilon - \epsilon^2/2)\sqrt{\log(\mathcal{N}_{[]}(\epsilon^4/2,\mathcal{P}_{\mathrm{M}},\|\cdot\|_{1,\infty}))}\\
        & \geq (\epsilon - \epsilon^2/2)\sqrt{\log(\mathcal{N}_{[]}(\epsilon^2/2,\overline{\mathcal{P}}_{h}(\epsilon),\|\cdot\|_{2, d_{P^{\star},h}^{\mathrm{b}}}))}\\
        & \geq \max\left\{\epsilon, \int_{\epsilon^2/2}^\epsilon \sqrt{\log(\mathcal{N}_{[]}(u,\overline{\mathcal{P}}_{h}(\epsilon),\|\cdot\|_{2, d_{P^{\star},h}^{\mathrm{b}}}))}\mathrm{d}u\right\}\\
        & = J_B(\epsilon, \overline{\mathcal{P}}_h(\epsilon),\|\cdot\|_{2, d_{P^{\star},h}^{\mathrm{b}}}),
    \end{align*}
    where the first inequality follows from Lemma \ref{lem: bracket 1}, 
    the second inequality follows from the fact that $\mathcal{N}_{[]}(u_1,\overline{\mathcal{P}}_{h}(\epsilon),\|\cdot\|_{2, d_{P^{\star},h}^{\mathrm{b}}}) \geq \mathcal{N}_{[]}(u_2,\overline{\mathcal{P}}_{h}(\epsilon),\|\cdot\|_{2, d_{P^{\star},h}^{\mathrm{b}}})$ for $u_1\leq u_2$.
    In the second inequality we assume without loss of generality that $\log(\mathcal{N}_{[]}(\epsilon^2/2,\overline{\mathcal{P}}_{h}(\epsilon),\|\cdot\|_{2, d_{P^{\star},h}^{\mathrm{b}}}))\geq 4$.
    Besides, since 
    \begin{align*}
        G_h(\epsilon) / \epsilon^2 = (1/\epsilon - 1/2)\sqrt{\log(\mathcal{N}_{[]}(\epsilon^4/2,\mathcal{P}_{\mathrm{M}},\|\cdot\|_{1,\infty}))}
    \end{align*}
    is non-increasing w.r.t $\epsilon$ for $\epsilon\in[0,1]$, we can confirm that $G_h$ satisfy the conditions in Lemma \ref{lem: mle}.
    With this choice of $G_h$, the inequality $\sqrt{n} \epsilon^2 \geq c_0 G_h(\epsilon)$ reduces to 
    \begin{align*}
        \sqrt{n} &\geq c_0(1/\epsilon - 1/2)\sqrt{\log(\mathcal{N}_{[]}(\epsilon^4/2,\mathcal{P}_{\mathrm{M}},\|\cdot\|_{1,\infty}))},
    \end{align*}
    which equivalents to 
    \begin{align}
        \epsilon&\geq \frac{c_0\sqrt{\log(\mathcal{N}_{[]}(\epsilon^4/2,\mathcal{P}_{\mathrm{M}},\|\cdot\|_{1,\infty}))}}{\sqrt{n}+\frac{c_0}{2}\sqrt{\log(\mathcal{N}_{[]}(\epsilon^4/2,\mathcal{P}_{\mathrm{M}},\|\cdot\|_{1,\infty}))}}.\label{eq: proof choice G zeta 1}
    \end{align}
    Taking $\zeta_h = c_0\sqrt{\log(\mathcal{N}_{[]}(1/n^2,\mathcal{P}_{\mathrm{M}},\|\cdot\|_{1,\infty}))/n}$, when $c_0\sqrt{\log(\mathcal{N}_{[]}(1/n^2,\mathcal{P}_{\mathrm{M}},\|\cdot\|_{1,\infty}))}\geq 2^{1/4}$, we can check that $\zeta_h$ satisfies the inequality  \eqref{eq: proof choice G zeta 1} by,
    \begin{align*}
        \zeta_h = \frac{c_0\sqrt{\log(\mathcal{N}_{[]}(1/n^2,\mathcal{P}_{\mathrm{M}},\|\cdot\|_{1,\infty}))}}{\sqrt{n}}
        \geq \frac{c_0\sqrt{\log(\mathcal{N}_{[]}(\zeta_h^2/2,\mathcal{P}_{\mathrm{M}},\|\cdot\|_{1,\infty}))}}{\sqrt{n}+\frac{c_0}{2}\sqrt{\log(\mathcal{N}_{[]}(\zeta_h^2/2,\mathcal{P}_{\mathrm{M}},\|\cdot\|_{1,\infty}))}}.
    \end{align*}
    This finishes the proof of Lemma \ref{lem: choice of G and zeta}.
\end{proof}

\subsection{Lemmas for Concentration Inequalities and Bracket Numbers}

\begin{lemma}[Bernstein inequality \uppercase\expandafter{\romannumeral1}]\label{lem: bernstein 1}
    For any step $h\in[H]$, with probability at least $1-\delta$,
    \begin{align*}
        |\mathbb{E}_{\mathbb{D}_h}[g_h(\hat{P}_h)] - \mathbb{E}_{d_{P^{\star},h}^{\mathrm{b}}}[g_h(\hat{P}_h)]| \leq \frac{c_1\log(c_2\mathcal{N}_{[]}(1/n^2,\mathcal{P}_{\mathrm{M}},\|\cdot\|_{1,\infty})/\delta)}{n}.
    \end{align*}
\end{lemma}

\begin{proof}[Proof of Lemma \ref{lem: bernstein 1}]
    Motivated by \cite{uehara2021pessimistic} and \cite{liu2022welfare}, to obtain a fast rate of convergence, we will utilize the localization technique in proving concentration.
    To this end, we first define the following localized realizable model space,
    \begin{align*}
        \mathcal{P}^{\mathrm{Loc}}_{\mathrm{M},h} = \left\{P\in\mathcal{P}_{\mathrm{M}}:\mathbb{E}_{d_{P^{\star},h}^{\mathrm{b}}}[g_h(P)]\leq \frac{c_1'\log(c_2'\mathcal{N}_{[]}(1/n^2,\mathcal{P}_{\mathrm{M}},\|\cdot\|_{1,\infty})/\delta)}{n}\right\},
    \end{align*}
    where absolute constants $c_1'$ and $c_2'$ are specified in \eqref{eq: proof mle 2}.
    According to the proof of \eqref{eq: proof mle 2}, we know that with probability at least $1-\delta$, the event $E_1 = \{\hat{P}_h\in\mathcal{P}^{\mathrm{Loc}}_{\mathrm{M},h}\}$ holds.
    In the sequel, we will always condition on the event $E_1$.
    Now we define another function class as 
    \begin{align*}
        \mathcal{F}_h = \left\{g_h(P):P\in\mathcal{P}^{\mathrm{Loc}}_{\mathrm{M},h}\right\}.
    \end{align*}
    Then applying Bernstein inequality with union bound (Lemma \ref{lem: bernstein union}) on the function class $\mathcal{F}_h$, we can obtain that with probability at least $1-\delta$, for any $P\in\mathcal{P}^{\mathrm{Loc}}_{\mathrm{M},h}$, (denote $\mathcal{M}(\epsilon) = \mathcal{N}(\epsilon,\mathcal{F}_h,\|\cdot\|_{\infty})$)
    \begin{align}
        &|\mathbb{E}_{\mathbb{D}_h}[g_h(P)] - \mathbb{E}_{d_{P^{\star},h}^{\mathrm{b}}}[g_h(P)]| \label{eq: proof bernstein 1 1}\\
        &\qquad \leq \sqrt{\frac{2\textcolor{blue}{\mathbb{V}}_{d_{P^{\star},h}^{\mathrm{b}}}[g_h(P)]\log(\mathcal{M}(\epsilon)/\delta)}{n}}+8\sqrt{\frac{\epsilon\log(\mathcal{M}(\epsilon)/\delta)}{n}} +\frac{8\log(\mathcal{M}(\epsilon)/\delta)}{3n}+2\epsilon\notag\\
        &\qquad \leq \sqrt{\frac{8\textcolor{blue}{\mathbb{E}}_{d_{P^{\star},h}^{\mathrm{b}}}[g_h(P)]\log(\mathcal{M}(\epsilon)/\delta)}{n}}+8\sqrt{\frac{\epsilon\log(\mathcal{M}(\epsilon)/\delta)}{n}} +\frac{8\log(\mathcal{M}(\epsilon)/\delta)}{3n}+2\epsilon\notag\\
        &\qquad \leq \frac{\sqrt{8c_1'\log(c_2'\mathcal{N}_{[]}(1/n^2,\mathcal{P}_{\mathrm{M}},\|\cdot\|_{1,\infty})/\delta)\cdot\log(\mathcal{M}(\epsilon)/\delta)}}{n}+8\sqrt{\frac{\epsilon\log(\mathcal{M}(\epsilon)/\delta)}{n}} +\frac{8\log(\mathcal{M}(\epsilon)/\delta)}{3n}+2\epsilon,\notag
    \end{align}
    where the first inequality follows from Lemma \ref{lem: bernstein union}, both the first and the second inequality use the fact that $\sup_{P\in\mathcal{P}_{\mathrm{M},h}^{\mathrm{Loc}}}|g_h(P)|\leq 4$, and the last inequality uses the definition of $\mathcal{P}_{\mathrm{M},h}^{\mathrm{Loc}}$.
    If we denote 
    \begin{align}\label{eq: F prime}
        \mathcal{F}_h^{\prime} = \left\{g_h(P):P\in\mathcal{P}_{\mathrm{M}}\right\},
    \end{align}
    we can upper bound the covering number $\mathcal{M}(\epsilon)$ via the following sequence of inequalities,
    \begin{align}\label{eq: proof bernstein 1 2}
        \mathcal{M}(\epsilon) = \mathcal{N}(\epsilon,\mathcal{F}_h,\|\cdot\|_{\infty}) \leq \mathcal{N}(\epsilon,\mathcal{F}_h^{\prime},\|\cdot\|_{\infty}) \leq \mathcal{N}(\epsilon,\mathcal{P}_{\mathrm{M}},\|\cdot\|_{1,\infty})\leq \mathcal{N}_{[]}(\epsilon,\mathcal{P}_{\mathrm{M}},\|\cdot\|_{1,\infty}),
    \end{align}
    where the first inequality follows from $\mathcal{F}_h\subseteq\mathcal{F}_h^{\prime}$, the second inequality can be easily derived from the relationship between $\mathcal{F}_h^{\prime}$ and $\mathcal{P}_{\mathrm{M}}$, and the last inequality follows from the fact that covering number can be bounded by bracket number.
    Therefore, by combining \eqref{eq: proof bernstein 1 1} and \eqref{eq: proof bernstein 1 2}, letting $\epsilon = 1/n^2$, we can derive that, conditioning on $E_1=\{\hat{P}_h\in\mathcal{P}_{\mathrm{M},h}^{\mathrm{Loc}}\}$, with probability at least $1-\delta$, 
    \begin{align*}
        |\mathbb{E}_{\mathbb{D}_h}[g_h(\hat{P}_h)] - \mathbb{E}_{d_{P^{\star},h}^{\mathrm{b}}}[g_h(\hat{P}_h)]| \leq \frac{c_1\log(c_2\mathcal{N}_{[]}(1/n^2,\mathcal{P}_{\mathrm{M}},\|\cdot\|_{1,\infty})/\delta)}{n},
    \end{align*}
    for some absolute constant $c_1,c_2>0$.
    Finally, since the event $E_1$ holds with probability at least $1-\delta$, by rescaling $\delta$, we can finish the proof.
\end{proof}

\begin{lemma}[Bernstein inequality \uppercase\expandafter{\romannumeral2}]\label{lem: bernstein 2}
    For any step $h\in[H]$, with probability at least $1-\delta$,
    \begin{align*}
        |\mathbb{E}_{\mathbb{D}_h}[g_h(P_h)] - \mathbb{E}_{d_{P^{\star},h}^{\mathrm{b}}}[g_h(P_h)]| \leq \frac{c_1\log(c_2\mathcal{N}_{[]}(1/n^2,\mathcal{P}_{\mathrm{M}},\|\cdot\|_{1,\infty})/\delta)}{n},\quad \forall P_h\in\hat{\mathcal{P}}_h.
    \end{align*}
\end{lemma}

\begin{proof}[Proof of Lemma \ref{lem: bernstein 2}]
    According to the proof of \eqref{eq: proof mle 5}, we know that the event $E_2$ defined as 
    \begin{align*}
        E_2 = \left\{\mathbb{E}_{\mathbb{D}_h}[g_h(P_h)]\leq 4\xi,\,\,\forall P_h\in\hat{\mathcal{P}}_h\right\}
    \end{align*}
    holds with probability at least $1-\delta/2$.
    In the sequel, we always condition on the event $E_2$.
    Now we define a function class $\mathcal{G}_h$ as following,
    \begin{align*}
        \mathcal{G}_h = \left\{g_h(P_h):P_h\in\hat{\mathcal{P}}_h\right\}.
    \end{align*}
    Applying Bernstein inequality with union bound (Lemma \ref{lem: bernstein union}) on the function class $\mathcal{G}_h$, we can obtain that with probability at least $1-\delta$, for any $P_h\in\hat{\mathcal{P}}_h$, (denote $\mathcal{M}'(\epsilon) = \mathcal{N}(\epsilon,\mathcal{G}_h,\|\cdot\|_{\infty})$)
    \begin{align}
        &|\mathbb{E}_{\mathbb{D}_h}[g_h(P_h)] - \mathbb{E}_{d_{P^{\star},h}^{\mathrm{b}}}[g_h(P_h)]|\notag\\
        &\qquad \leq \sqrt{\frac{2\textcolor{blue}{\mathbb{V}}_{d_{P^{\star},h}^{\mathrm{b}}}[g_h(P_h)]\log(\mathcal{M}'(\epsilon)/\delta)}{n}}+8\sqrt{\frac{\epsilon\log(\mathcal{M}'(\epsilon)/\delta)}{n}} +\frac{8\log(\mathcal{M}'(\epsilon)/\delta)}{3n}+2\epsilon\notag\\
        &\qquad \leq \sqrt{\frac{8\textcolor{blue}{\mathbb{E}}_{d_{P^{\star},h}^{\mathrm{b}}}[g_h(P_h)]\log(\mathcal{M}'(\epsilon)/\delta)}{n}}+8\sqrt{\frac{\epsilon\log(\mathcal{M}'(\epsilon)/\delta)}{n}} +\frac{8\log(\mathcal{M}'(\epsilon)/\delta)}{3n}+2\epsilon\notag\\
        &\qquad \leq \sqrt{\frac{8(|\textcolor{blue}{\mathbb{E}}_{d_{P^{\star},h}^{\mathrm{b}}}[g_h(P_h)]-\mathbb{E}_{\mathbb{D}_h}[g_h(P_h)]| + 4\xi)\log(\mathcal{M}'(\epsilon)/\delta)}{n}}\notag\\
        &\qquad\qquad+8\sqrt{\frac{\epsilon\log(\mathcal{M}'(\epsilon)/\delta)}{n}}+\frac{8\log(\mathcal{M}'(\epsilon)/\delta)}{3n}+2\epsilon,\label{eq: proof bernstein 2 1}
    \end{align}
    where the first inequality follows from Lemma \ref{lem: bernstein union}, both the first and the second inequality use the fact that $\sup_{P_h\in\hat{\mathcal{P}}_h}|g_h(P_h)|\leq 4$, and the last inequality uses the definition of event $E_2$.
    By using the fact that the function class $\mathcal{G}_h\subseteq\mathcal{F}_h^{\prime}$ where $\mathcal{F}_h^{\prime}$ is defined in \eqref{eq: F prime} in the proof of Lemma \ref{lem: bernstein 1}, we can apply the same argument as \eqref{eq: proof bernstein 1 2} to derive that $\mathcal{M}'(\epsilon)\leq \mathcal{N}_{[]}(\epsilon,\mathcal{P}_{\mathrm{M}},\|\cdot\|_{1,\infty})$.
    Thus taking $\epsilon = 1/n^2$, denoting $\Delta_h(P_h) = |\mathbb{E}_{\mathbb{D}_h}[g_h(P_h)] - \mathbb{E}_{d_{P^{\star},h}^{\mathrm{b}}}[g_h(P_h)]|$, we can derive from \eqref{eq: proof bernstein 2 1} that,
    \begin{align}
        \Delta_h(P_h) &\leq \sqrt{\frac{8(\Delta_h(P_h) + 4\xi)\log(\mathcal{N}_{[]}(\epsilon,\mathcal{P}_{\mathrm{M}},\|\cdot\|_{1,\infty})/\delta)}{n}} \notag\\
        &\qquad + 8\sqrt{\frac{\log(\mathcal{N}_{[]}(\epsilon,\mathcal{P}_{\mathrm{M}},\|\cdot\|_{1,\infty})/\delta)}{n^3}} + \frac{8\log(\mathcal{N}_{[]}(\epsilon,\mathcal{P}_{\mathrm{M}},\|\cdot\|_{1,\infty})/\delta)}{3n}+\frac{2}{n^2}\notag\\
        &\leq \sqrt{\frac{8(\Delta_h(P_h) + 4\xi)\log(\mathcal{N}_{[]}(\epsilon,\mathcal{P}_{\mathrm{M}},\|\cdot\|_{1,\infty})/\delta)}{n}} + \frac{c_1'\log(\mathcal{N}_{[]}(\epsilon,\mathcal{P}_{\mathrm{M}},\|\cdot\|_{1,\infty})/\delta)}{n}\notag\\
        &\leq \sqrt{\frac{8\Delta_h(P_h)\log(\mathcal{N}_{[]}(\epsilon,\mathcal{P}_{\mathrm{M}},\|\cdot\|_{1,\infty})/\delta)}{n}} + \frac{c_1''\log(c_2''\mathcal{N}_{[]}(\epsilon,\mathcal{P}_{\mathrm{M}},\|\cdot\|_{1,\infty})/\delta)}{n},\label{eq: proof bernstein 2 2}
    \end{align}
    for some absolute constants $c_1',c_1'',c_2''>0$, where in the last inequality we have applied the definition of $\xi$.
    Now solving this quadratic inequality \eqref{eq: proof bernstein 2 2} w.r.t $\Delta_h(P_h)$, we can obtain that, 
    \begin{align*}
        \Delta_h(P_h) \leq \frac{c_1\log(c_2\mathcal{N}_{[]}(\epsilon,\mathcal{P}_{\mathrm{M}},\|\cdot\|_{1,\infty})/\delta)}{n},
    \end{align*}
    for some absolute constants $c_1,c_2>0$.
    Thus we obtain that when conditioning on the event $E_2$, with probability at least $1-\delta$, for any $P_h\in\hat{\mathcal{P}}_h$, the desired concentration inequality holds.
    Finally, since $E_2$ holds with probability at least $1-\delta/2$, by rescaling $\delta$, we can finish the proof of Lemma \ref{lem: bernstein 2}.
\end{proof}

\begin{lemma}[Bernstein inequality with union bound]\label{lem: bernstein union}
    Consider a function class $\mathcal{F} \subset\{f:$ $\mathcal{X} \mapsto \mathbb{R}\}$, where $\mathcal{X}$ is a probability space. 
    If we assume that the $\epsilon$-covering number of $\mathcal{F}$ under infinity-norm is finite, that is, $M=\mathcal{N}\left(\epsilon, \mathcal{F},\|\cdot\|_{\infty}\right)<\infty$, and we also assume that there exists an absolute constant $R$ such that $|f(X)| \leq R$, then with probability at least $1-\delta$ the following inequality holds for all $f \in \mathcal{F}$,
    \begin{align*}
        \left|\frac{1}{n} \sum_{\tau=1}^n f\left(X_\tau\right)-\mathbb{E}[f(X)]\right| \leq 2 \epsilon+\sqrt{\frac{2 \mathbb{V}[f(X)] \log (M / \delta)}{n}}+4 \sqrt{\frac{R \epsilon \log (M / \delta)}{n}}+\frac{2 R \log (M / \delta)}{3 n},
    \end{align*}
    where $X, X_1, \ldots, X_n$ are i.i.d. samples on the probability space $\mathcal{X}$.
\end{lemma}
\begin{proof}[Proof of Lemma \ref{lem: bernstein union}]
    We refer to Lemma F.1 in \cite{liu2022welfare} for a detailed proof.
\end{proof}

\begin{lemma}[Bracket number \uppercase\expandafter{\romannumeral1}]\label{lem: bracket 1}
    It holds for any $\epsilon\geq 0$ that
    \begin{align*}
        \mathcal{N}_{[]}(\epsilon,\overline{P}_h(\epsilon),\|\cdot\|_{2,d_{P^{\star},h}^{\mathrm{b}}})\leq \mathcal{N}_{[]}(2\epsilon^2,\mathcal{P}_{\mathrm{M}},\|\cdot\|_{1,\infty}).
    \end{align*}
\end{lemma}

\begin{proof}[Proof of Lemma \ref{lem: bracket 1}]
    We refer to Lemma G.2 in \cite{liu2022welfare} for a detailed proof.
\end{proof}

\subsection{Lemmas for Dual Variables}

\begin{lemma}[Dual variable for KL-divergence]\label{lem: bound lambda kl}
    The optimal solution to the following optimization problem 
    \begin{align*}
        \lambda^{\star} = \argsup_{\lambda\in\mathbb{R}_+} \left\{ - \lambda\log\left(\int \exp\left\{-f(x)/\lambda\right\}P(\mathrm{d}x)\right) -\lambda\sigma\right\},
    \end{align*}
    with $\|f\|_{\infty}\leq H$ and some probability measure $P$ satisfies that $\lambda^{\star} \leq H/\sigma$.
\end{lemma}

\begin{proof}[Proof of Lemma \ref{lem: bound lambda kl}]
    For simplicity, denote by $g(\lambda) = - \lambda\log\left(\int \exp\left\{-f(x)/\lambda\right\}P(\mathrm{d}x)\right) -\lambda\sigma$.
    Notice that $g(0) = 0$, and for $\lambda > H/\sigma$, due to $\|f\|_{\infty}\leq H$, we have that 
    \begin{align*}
        g(\lambda) < -\lambda \log(\exp\{-H/(H/\sigma)\}) - \lambda\sigma  = \lambda\sigma - \lambda\sigma = 0.
    \end{align*}
    Thus we can conclude that $\lambda^{\star} \leq H/\sigma$.
\end{proof}

\begin{lemma}[Dual variable for TV-distance]\label{lem: bound lambda tv}
    The optimal solution to the following optimization problem 
    \begin{align*}
        \lambda^{\star}
        = \argsup_{\lambda\in\mathbb{R}} \left\{ - \int (\lambda-f(x))_+ P(\mathrm{d}x)- \frac{\sigma}{2}(\lambda-\inf_x f(x))_+ + \lambda \right\}.
    \end{align*}
    with $\|f\|_{\infty}\leq H$ and some probability measure $P$ satisfies that $0\leq \lambda^{\star} \leq H$.
\end{lemma}
\begin{proof}[Proof of Lemma \ref{lem: bound lambda tv}]
     For simplicity, denote $g(\lambda) = - \int (\lambda-f(x))_+ P(\mathrm{d}x)- \frac{\sigma}{2}(\lambda-\inf_x f(x))_+ + \lambda$.
     We can observe that $g(0) = 0$, and $g(\lambda)\leq 0$ for $\lambda\leq 0$. 
     Thus we have shown that $\lambda^{\star}\geq 0$.
     Also, for $\lambda\geq H$, due to $\|f\|_{\infty}\leq H$, we can write $g(\lambda)$ as
     \begin{align*}
         g(\lambda) &= - \int \lambda-f(x) P(\mathrm{d}x)- \frac{\sigma}{2}(\lambda-\inf_x f(x)) + \lambda\\
         &= \int f(x) P(\mathrm{d}x) +\frac{\sigma}{2}\inf_x f(x) -  \frac{\sigma}{2}\lambda,
     \end{align*}
     which is a monotonically decreasing function with respect to $\lambda$. Thus we prove that $\lambda^{\star}\leq H$.
\end{proof}